\documentclass[11pt]{article}
\usepackage[T1]{fontenc}         
\usepackage{amsfonts}      
\usepackage{nicefrac}      
\usepackage{microtype}   
\usepackage[colorlinks,
            linkcolor=red,
            anchorcolor=blue,
            citecolor=blue, 
            pagebackref=true
           ]{hyperref}

\usepackage{fullpage}
\usepackage{tabularx}
\usepackage{float}
\usepackage{enumitem}
\usepackage{wrapfig}
\usepackage{diagbox}

\usepackage{mmll}
\usepackage{tikz}
\usetikzlibrary{arrows,automata,positioning}
\usepackage[colorinlistoftodos, textsize=tiny, textwidth=2.2cm, backgroundcolor=blue!10, linecolor=magenta, bordercolor=magenta]{todonotes}

\def \algname {We-DRIVE-U}

\allowdisplaybreaks
\title{Upper and Lower Bounds for Distributionally Robust Off-Dynamics Reinforcement Learning}

\author{
    Zhishuai Liu\thanks{ 
    Duke University; email: {\tt
    zhishuai.liu@duke.edu}}~\footnotemark[3] 
    ~~~~~~
    Weixin Wang\thanks{ 
    Duke University; email: {\tt
    weixin.wang@duke.edu}}~\thanks{Equal contribution} 
    ~~~~~~
    Pan Xu\thanks{
    Duke University; email: {\tt
    pan.xu@duke.edu}}\\
}

\begin{document}

\date{}
\maketitle

\begin{abstract}
We study off-dynamics Reinforcement Learning (RL), where the policy training and deployment environments are different. To deal with this environmental perturbation, we focus on learning policies robust to uncertainties in transition dynamics under the framework of distributionally robust Markov decision processes (DRMDPs), where the nominal and perturbed dynamics are linear Markov Decision Processes. We propose a novel algorithm We-DRIVE-U that enjoys an average suboptimality $\widetilde{\mathcal{O}}\big({d H \cdot \min \{1/{\rho}, H\}/\sqrt{K} }\big)$, where $K$ is the number of episodes, $H$ is the horizon length, $d$ is the feature dimension and $\rho$ is the uncertainty level. This result improves the state-of-the-art by $\mathcal{O}(dH/\min\{1/\rho,H\})$. We also construct a novel hard instance and derive the first information-theoretic lower bound in this setting, which indicates our algorithm is near-optimal up to $\mathcal{O}(\sqrt{H})$ for any uncertainty level $\rho\in(0,1]$. Our algorithm also enjoys a `rare-switching' design, and thus only requires  $\mathcal{O}(dH\log(1+H^2K))$ policy switches and $\mathcal{O}(d^2H\log(1+H^2K))$ calls for oracle to solve dual optimization problems, which significantly improves the computational efficiency of existing algorithms for DRMDPs, whose policy switch and oracle complexities are both $\mathcal{O}(K)$. 
\end{abstract}

\section{Introduction}
\label{sec:introduction}
In dynamic decision-making and reinforcement learning (RL), Markov decision processes (MDPs) offer a well-established framework for understanding complex systems and guiding agent behavior \citep{sutton2018reinforcement}. However, MDPs encounter significant challenges in practical applications due to incomplete knowledge of model parameters, especially transition probabilities. This sim-to-real gap, representing the difference between training and testing environments, can lead to failures in fields like infectious disease control and robotics \citep{farebrother2018generalization, zhao2020sim, laber2018optimal, liu2023deep, peng2018sim}. To address these challenges, off-dynamics RL provides a framework where policies are trained on a source domain and deployed to a distinct target domain, promoting robust performance across varying environments \citep{eysenbach2020off, jiang2021simgan}. Within this framework, distributionally robust Markov decision processes (DRMDPs) have emerged as a promising way to model transition uncertainty. DRMDPs focus on learning robust policies that perform well under worst-case scenarios \citep{nilim2005robust, iyengar2005robust}. Prior works \citep{zhang2021robust, yang2022toward, panaganti2022robust, shi2022distributionally, yang2023distributionally, shen2024wasserstein} have proposed algorithms mainly for tabular DRMDP settings, where the number of states and actions is finite, which are infeasible in large state and action spaces.

In environments characterized by large state and action spaces, function approximation techniques are crucial to overcome the computational burden posed by high dimensionality. Linear function approximation methods, based on relatively simple function classes, have shown significant theoretical and practical successes in standard MDP environments \citep{jin2020provably, he2021logarithmic, he2023nearly, yang2020reinforcement, hsu2024randomized}. However, their application in DRMDPs introduces additional complexities. These complexities arise from the nonlinearity caused by the dual formulation in the worst-case analysis, even when the transition dynamics in the source domain are modeled as linear. 

Recently, \citet{liu2024distributionally} provided the first theoretical results in the online setting of $d$-rectangular linear DRMDPs, a specific type of DRMDPs where the nominal model is a linear MDP \citep{jin2020provably} and the uncertainty set is defined based on the linear structure of the nominal transition kernel. Apart from this,  online DRMDP with linear function approximation is largely underexplored and it is not clear how far existing algorithms are from optimal. Consequently, two natural questions arise:
\begin{center}
{\it Can we improve the current results for online DRMDPs with linear function approximation? \\ What is the fundamental limit in this setting?}  
\end{center} 
In this paper, we provide an affirmative answer to the first question and answer the second question by providing an information theoretic lower bound for the online setting of $d$-rectangular linear DRMDPs. In particular,
motivated by the adoption of variance-weighted ridge regression to achieve nearly optimal result in standard linear MDPs \citep{zhou2021nearly,zhou2022computationally,zhang2021improved,kim2022improved,zhao2023variance,he2023nearly,hu2023nearly}, we propose a variance-aware distributionally robust algorithm to solve the off-dynamics RL problem. Due to the nonlinearity caused by the dual optimization of DRMDPs, the adoption of variance information in linear DRMDPs is highly nontrivial. The only existing algorithm that incorporates variance information in learning linear DRMDPs requires coverage assumptions on the offline dataset \citep{liu2024minimax}, which is infeasible in our setting where the algorithm needs to interact with the environment in an online fashion. Therefore, our work poses a distinct algorithm design and calls for different theoretical analysis techniques. %
Specifically, our \textbf{main contributions} are summarized as follows:
\begin{itemize}[leftmargin=*,nosep]
    \item We propose a novel algorithm, \algname, for $d$-rectangular linear DRMDPs with total-variation (TV) divergence uncertainty sets. \algname\ is designed based on the optimistic principle \citep{jin2018q, jin2020provably, he2023nearly} to trade off the exploration and exploitation during  interacting with the source environment to learn a robust policy. The key idea of \algname\ lies in incorporating the variance information into the policy learning. In particular, a carefully designed optimistic estimator of the variance of the optimal robust value function is established at each episode, which will be used in variance-weighted regressions under a novel `rare-switching' regime to update the robust policy estimation.

    \item We prove that  \algname\ achieves  $\widetilde{\cO}({d H \cdot \min \{1/{\rho}, H\}/\sqrt{K} })$ average suboptimality when the number of episode $K$ is large, which improves the state-of-the-art result \citep{liu2024distributionally} by $\widetilde{\cO}(dH/\min\{1/\rho,H\})$, %
    We highlight that the average suboptimality of \algname\ demonstrates the `Range Shrinkage' property (refer to \Cref{lem:range_shrinkage}) through the term $\min\{1/\rho,H\}$.
    We further established an information-theoretic lower bound $\Omega(dH^{1/2}\cdot \min\{1/\rho,H\}/\sqrt{K})$, which shows that \algname\ is near-optimal up to $\cO(\sqrt{H})$ for any uncertainty level $\rho\in(0,1]$.

     \item \algname\ is favorable in applications where policy switching is risky or costly, 
     since \algname\ achieves  $\cO(dH\log(1+H^2K))$ global policy switch (refer to \Cref{def:Global Switching Cost}). Moreover, we note that calls for oracle to solve dual
     optimizations \eqref{eq:dual optimization} are one of the main sources of computation complexity in DRMDP with linear function approximation. Thanks to the specifically designed `rare-switching' regime, \algname\ achieves $\cO(d^2H\log(1+H^2K))$ oracle complexity (refer to \Cref{def:dual oracle}). Both results improve exiting online DRMDP algorithms by a factor of $K$. Thus, \algname\ enjoys low switching cost and low computation cost.
\end{itemize}

\paragraph{Novelty in Algorithm and Hard Instance Design} 
The variance estimator and the variance-weighted ridge regression in 
\algname\ lead to two major improvements on the average suboptimality compared to the previous result: 1) the incorporation of variance information enables us to leverage the recently discovered  `Range Shrinkage' property for linear DRMDPs, which is crucial in achieving the tighter dependence on the horizon length $H$; 2) inspired by previous works on standard MDPs \citep{azar2017minimax, he2023nearly}, we design a new `rare-switching' regime (refer to \Cref{remark: low switching}) and monotonic robust value function estimation (refer to \Cref{remark:monotonic update}). Together with the optimistic variance estimator, we achieve the tight dependence on $d$. As for the lower bound, we construct a novel family of hard-to-learn linear DRMDPs, showing the `Range Shrinkage' property on robust value functions for any policy $\pi$. 

\paragraph{Technical Challenges}
The incorporation of variance information 
poses unique challenges to our theoretical analysis. In particular, in order to get the near-optimal upper bound on average suboptimality, we need to bound the variance-weighted version of the $d$-rectangular estimation error (see \eqref{eq:AveSubopt_thm_origin} for more details), instead of the vanilla one in (5.1) of \citet{liu2024distributionally}. However, this term is in general intractable through direct matrix analysis. To solve this challenge, we seek to convert the variance-weighted $d$-rectangular estimation error to the vanilla version, which requires a precise upper bound on the variance estimator. Intuitively, the variance estimator should be close to the true variance when the `sample size' $k$ is large. While different from the recent study \citep{liu2024minimax} on the offline linear DRMDP, our variance estimator is an optimistic one and thus cannot be trivially upper bounded by the true variance.
To this end,  we carefully analyze the error of the variance estimator in the large $k$ regime, and meticulous calculation shows that the optimistic variance estimator can be upper bounded by a clipped version of the true variance (refer to \Cref{lem:est_var_upper_bound}). 

\paragraph{Notations} 
For any positive integer $H\in\ZZ_{+}$, we denote $[H]=\{1,2,\cdots, H\}$. For any set $\cS$, define $\Delta(\cS)$ as the set of probability distributions over $\cS$. For any function $V:\cS\rightarrow \RR$, define $[\PP_hV](s,a) = \EE_{s'\sim P_h(\cdot|s,a)}[V(s')]$, and $[V(s)]_{\alpha}=\min\{V(s),\alpha\}$, where  $\alpha>0$ is a constant. For a vector $\bx$, define $x_j$ as its $j$-th entry. Moreover, denote $[x_i]_{i\in [d]}$ as a vector with the $i$-th entry being $x_i$.
For a matrix $A$, denote $\lambda_i(A)$ as the $i$-th eigenvalue of $A$. For two matrices $A$ and $B$, denote $A\preceq B$ as the fact that $B-A$ is a positive semi-definite matrix. For any  $P,Q\in\Delta(\cS)$, the total variation divergence of $P$ and $Q$ is defined as $D(P||Q)=1/2\int_{\cS}|P(s)-Q(s)|ds$.

\section{Related Work}

\paragraph{Distributionally Robust MDPs} 
There has been a large body of works studying DRMDPs under various settings, for instance, the setting of planning and control \citep{xu2006robustness, wiesemann2013robust, yu2015distributionally, mannor2016robust, goyal2023robust} where the exact transition model is known, the setting with a generative model \citep{zhou2021finite, yang2022toward, panaganti2022sample, xu2023improved, shi2023curious, yang2023avoiding}, the offline setting \citep{panaganti2022robust, shi2022distributionally, blanchet2023double}
and the online setting \citep{dong2022online, liu2024distributionally, lu2024distributionally}.
Among tabular DRMDPs, the most relevant studies to ours are \citet{shi2023curious, lu2024distributionally}. In particular, \citet{shi2023curious} studies tabular DRMDPs with TV uncertainty sets. %
They provide an information-theoretic lower bound, as well as a matching upper bound on the sample complexity. The key message is that the sample complexity bounds depend on the uncertainty level, and when the uncertainty level is of constant order, policy learning in a DRMDP requires less samples than in a standard MDP. Further, \citet{lu2024distributionally} studies the online tabular DRMDPs with TV uncertainty sets, they provide an  algorithm  that achieves the near-optimal sample complexity under a vanishing minimal value assumption to circumvent the curse of support shift. 

\paragraph{Online Linear MDPs and Linear DRMDPs} 
The nominal model studied in our paper is assumed to be a linear MDP with a simplex feature space. There is a line of works studying online linear MDPs \citep{yang2020reinforcement, jin2020provably, modi2020sample, zanette2020frequentist, wang2020reward, he2021logarithmic, wagenmaker2022reward, ishfaq2023provable}, and the minimax optimality of this setting is studied in the recent work of \citet{he2023nearly}. In particular, they adopt the variance-weighted ridge regression scheme and the `rare-switching' policy update strategy in their algorithm design. 
The setting of online linear DRMDP is relatively understudied, with both the lower bound and the near-optimal upper bound remain elusive. Specifically, the only work studies the online linear DRMDP setting is \citet{liu2024distributionally}. 
Under the TV uncertainty set, their algorithm, DR-LSVI-UCB, achieves an average suboptimality of the order $\tilde{O}(d^2H^2/\sqrt{K})$. However, recent evidence from studies \citep{liu2024minimax, wang2024sample} on offline linear DRMDPs suggests that this rate is far from optimality. In particular, \citet{liu2024minimax} proves that their algorithm, VA-DRPVI, achieves
an upper bound on the suboptimality in the order of $\tilde{O}({dH\min\{1/\rho, H\}}/{\sqrt{K}})$. Nonetheless, their algorithm and analysis are based on a pre-collected offline dataset which satisfies some coverage assumption, and thus cannot be utilized in the online setting, where a strategy on data collection is required to deal with the challenge of exploration and exploitation trade-off.

\section{Preliminary}
\label{sec:preliminary}

We use a tuple $\text{DRMDP}(\mathcal{S}, \mathcal{A}, H, \mathcal{U}^{\rho}(P^0), r)$ to denote a finite horizon distributionally robust Markov decision process (DRMDP), where $\mathcal{S}$ and $\mathcal{A}$ are the state and action spaces, 
$H\in \ZZ_+$ is the horizon length, 
$P^0=\{P^0_h\}_{h=1}^{H}$ is the nominal transition kernel, $\cU^{\rho}(P^0) =  \bigotimes_{h\in[H]}\cU^{\rho}_h(P_h^0)$ denotes an uncertainty set centered around the nominal transition kernel with an uncertainty level $\rho\geq 0$, $r=\{r_h\}_{h=1}^H$ is the reward function. A policy $\pi=\{\pi_h\}_{h=1}^H$ is a sequence of decision rules.
For any policy $\pi$, we define the robust value function $V_h^{\pi, \rho}(s) =\inf_{P \in \mathcal{U}^{\rho}(P^0)}\mathbb{E}^P[\sum_{t=h}^H r_t(s_t,a_t)|s_h=s, \pi ]$ and the robust Q-function $Q_h^{\pi, \rho}(s,a) =\inf_{P \in \mathcal{U}^{\rho}(P^0)}\mathbb{E}^P[\sum_{t=h}^H r_t(s_t,a_t)|s_h=s,a_h=a, \pi ]$ for any $(h,s,a)\in[H]\times\cS\times\cA$. 
Moreover, we define the optimal robust value function and optimal robust state-action value function: for any $(h,s,a)\in[H]\times\cS\times\cA$,  $V_h^{\star, \rho}(s) = \sup_{\pi \in \Pi} V_h^{\pi, \rho}(s)$, $Q_h^{\star, \rho}(s,a) = \sup_{\pi \in \Pi} Q_h^{\pi, \rho}(s,a)$, where $\Pi$ is the set of all policies.
Correspondingly, the optimal robust policy is the policy that achieves the optimal robust value function $\pi^{\star} = \text{argsup}_{\pi \in \Pi}V_h^{\pi, \rho}(s)$.

We study the $d$-rectangular linear DRMDP \citep{ma2022distributionally,blanchet2023double,liu2024distributionally,liu2024minimax}, which is a special DRMDP where the nominal environment is a linear MDP \citep{jin2020provably} with a simplex state space, defined as follows.
\begin{assumption}
\label{assumption:linear_mdp}
Given a known feature mapping $\bphi: \cS \times \cA \rightarrow \RR^d$ satisfying $ \sum_{i=1}^d\phi_i(s,a)=1$, $\phi_i(s,a) \geq 0$,
for any $(i, s,a)\in [d] \times \cS \times \cA$,
we assume the reward functions $\{r_h\}_{h=1}^H$ and nominal transition kernels $\{P_h^0\}_{h=1}^H$ are linearly parameterized. Specifically, for any $(h,s,a) \in [H]\times \mathcal{S} \times \mathcal{A}$, $r_h(s,a)=\langle \bphi(s,a), \btheta_h\rangle, P_h^0(\cdot|s,a)=\langle \bphi(s,a), \bmu_h^0(\cdot)\rangle,$ where $\{\btheta_h\}_{h=1}^H$ are known vectors with bounded norm $\Vert \btheta_h \Vert_2 \leq \sqrt{d}$ and $\{\bmu_h\}_{h=1}^H$ are unknown probability measures over $\cS$. 
\end{assumption}
In $d$-rectangular linear DRMDPs, the uncertainty set $\cU_h^\rho(P_h^0)$ is defined based on the linear structure of $P_h^0$ satisfying \Cref{assumption:linear_mdp}. In particular, 
we first define the factor uncertainty sets as $\cU_{h, i}^{\rho}(\mu^0_{h,i}) = \{\mu: \mu\in \Delta(\cS), D(\mu||\mu_{h,i}^0)\leq \rho\},\forall (h,i)\in[H]\times[d]$. In this work we choose $D(\cdot||\cdot)$ as the total variation (TV) divergence. Then we define the uncertainty set as $\cU^{\rho}_h(P_h^0) = \bigotimes_{(s,a)\in \cS\times\cA}\cU_h^{\rho}(s,a; \bmu^0_h)$, where $\cU_h^{\rho}(s,a; \bmu^0_h)=\{\sum_{i=1}^d \phi_i(s,a)\mu_{h,i}(\cdot): \mu_{h,i}(\cdot) \in \cU_{h, i}^{\rho}(\mu^0_{h,i}), \forall i \in [d]\}$. 
 \citet{liu2024distributionally} show that the following robust Bellman equations hold, that is for any policy $\pi$, 
\begin{subequations}
\label{eq:robust bellman equation}
\begin{align}
 Q_h^{\pi, \rho}(s,a)&=\textstyle r_h(s,a)+\inf_{P_h(\cdot|s,a)\in\cU_h^{\rho}(s,a;\bmu_h^0)}[\PP_h V_{h+1}^{\pi,\rho}](s,a), \\
 V_h^{\pi, \rho}(s) &=\textstyle \EE_{a\sim\pi_h(\cdot|s)}\big[Q_h^{\pi,\rho}(s,a)\big],
\end{align}
\end{subequations}
as well as the robust Bellman optimality equations
\begin{subequations}
\label{eq:optimal robust bellman equation}
\begin{align}
     Q_h^{\star, \rho}(s,a) &= \textstyle r_h(s,a) +\inf_{P_h(\cdot|s,a) \in \cU_h^{\rho}(s,a;\bmu_h^0)}[\PP_h V_{h+1}^{\star, \rho}](s,a),\\
    V_h^{\star, \rho}(s)&= \textstyle \max_{a\in\cA}Q_h^{\star}(s,a).
\end{align}
\end{subequations}
In the context of online DRMDPs, an agent actively interacts with the nominal environment within $K$ episodes to learn the optimal robust policy. Specifically, at the start of episode $k$, an agent chooses a policy $\pi^k$ based on the history information and receives the initial state $s_1^k$. Then the agent interacts with the nominal environment by executing $\pi^k$ until the end of episode $k$, and collects a new trajectory. The goal of the agent is to minimize the average suboptimality after $K$ episodes, which is defined as $\text{AveSubopt}(K) = 1/K\sum_{k=1}^K\big[V_1^{\star,\rho}(s_1^k) - V_1^{\pi^k, \rho}(s_1^k)\big]$.

\citet{lu2024distributionally} recently show that in general sample efficient learning in online DRMDPs is impossible due to the curse of support shift, i.e., the nominal kernel and target kernel do not share the same support. By designing proper feature mappings, we show that their hard example implies the same hardness result for the online linear DRMDP setting.
\begin{proposition}(Hardness result)
\label{prop:hardness result}
    There exists two $d$-rectangular linear DRMDPs $\{\cM_0, \cM_1\}$, such that $\inf_{\cA\cL\cG}\sup_{\theta\in\{0,1\}}\EE[\text{AveSubopt}^{\cM_{\theta},\cA\cL\cG}(K)] \geq \Omega(\rho\cdot H)$, where $\text{AveSubopt}^{\cM_{\theta},\cA\cL\cG}(K)$ is the average suboptimality of algorithm $\cA\cL\cG$ under the $d$-rectangular linear DRMDP $\cM_{\theta}$.
\end{proposition}
Note that the lower bound in \Cref{prop:hardness result} does not converge to zero as $K$ increases, which means that in general no algorithm can guarantee to learn the optimal robust policy approximately. To circumvent this problem, in the rest of paper we focus on a tractable subclass of $d$-rectangular linear DRMDP following \citet{liu2024distributionally, lu2024distributionally},  which is formally defined in the following assumption.
\begin{assumption}[Fail-state]
\label{assumption:fail_state}
Assume there exists a `fail state' $s_f$ in the $d$-rectangular linear DRMDP, such that for all $ (h,a) \in [H]\times \cA$, $r_h(s_f, a)=0$, $\PP_h^0(s_f|s_f,a)=1$.
\end{assumption}
With \Cref{assumption:fail_state}, we can follow the framework in \citet{liu2024distributionally}, where we have the following results on robust value functions and dual formulation that are helpful in solving the optimization in \eqref{eq:optimal robust bellman equation}.
\begin{proposition}[Remark 4.2 of \citet{liu2024distributionally}]
\label{prop:dual with fail state}
    Under \Cref{assumption:fail_state}, for any $(\pi, h,a) \in \Pi\times [H]\times \cA$, we have $Q^{\pi, \rho}_h(s_f,a)=0$, and $V^{\pi,\rho}_h(s_f)=0$.
    Moreover, for any function $V:\cS \rightarrow [0,H]$ with $\min_{s\in\cS}V(s)=V(s_f)=0$, we have $\inf_{\mu\in\cU^{\rho}(\mu^0)}\EE_{s\sim\mu}V(s) = \max_{\alpha \in [0,H]}\{\EE_{s\sim \mu^0}[V(s)]_\alpha - \rho\alpha \}$.
\end{proposition}

\section{Algorithm Design}
\label{sec:algorithm_design}
One prominent property of the $d$-rectangular DRMDP is that the robust Q-functions possess linear representations with respect to the feature mapping $\bphi$.
In particular, under \Cref{assumption:linear_mdp,assumption:fail_state}, \citet{liu2024distributionally} show that for any $(\pi,s,a,h)\in \Pi\times\cS \times \cA \times [H]$, the robust Q-function $Q_h^{\pi,\rho}(s,a)$ has a linear form as follows $Q_h^{\pi,\rho}(s,a) = \big( r_h(s,a) + \bphi(s,a)^\top \bnu_{h}^{\pi,\rho} \big)\ind\{s\neq s_f\}$, where $\bnu_h^{\pi,\rho}=\big(\nu_{h,1}^{\pi, \rho},\ldots,\nu_{h,d}^{\pi, \rho}\big)^{\top}$, $\nu_{h,i}^{\pi, \rho}=\max_{\alpha\in[0,H]}\big\{ \allowbreak z_{h,i}^{\pi}(\alpha)-\rho\alpha\big\}$, $z_{h,i}^{\pi}(\alpha)=\EE^{\mu_{h,i}^0}\big[V_{h+1}^{\pi, \rho}(s')\big]_{\alpha}$ and $\alpha\in[0,H]$ is the dual variable derived from the dual formulation (see \Cref{prop:strong duality for TV} for more details). Moreover, the robust Bellman optimality equation \eqref{eq:optimal robust bellman equation} shows that the greedy policy with respect to the optimal robust Q-function is exactly the optimal robust policy $\pi^\star$. Therefore, the core idea behind the algorithm design is to estimate the optimal robust Q-function using linear function approximation, and then find $\pi^\star$ by the greedy policy derived from the estimated optimal robust Q-function. We present our algorithm in \Cref{alg:DR-LSVI-UCB+}. In the sequel, we provide detailed discussion about the components in our algorithm design. 

\subsection{Variance-Weighted Ridge Regression for Online DRMDPs}
From Line \ref{line:start backward induction} to \ref{algline:end of the first loop} of \Cref{alg:DR-LSVI-UCB+},
we adopt the backward induction procedure to update the robust Q-function estimation. In particular, for any $(k,h)\in [K]\times[H]$, suppose we have an estimated robust value function {\small$\hat{V}_{k,h+1}^\rho$}. By the robust Bellman optimality equation \eqref{eq:optimal robust bellman equation} and \Cref{prop:dual with fail state}, conducting one step backward induction on {\small$\hat{V}_{k,h+1}^\rho$} leads to the following linear form \citep{liu2024distributionally}:
\begin{align}
\label{equ:backward_induction_linear_form}
    \textstyle r_h(s,a) + \inf_{P_h\in\cU_h^\rho(s,a;\bmu^0)}\PP_h[\hat{V}_{k,h+1}](s,a) =  \bphi(s,a)^\top (\btheta_h+\bnu_h^{\rho,k} )\ind\{s\neq s_f\},
\end{align}
where $\nu_{h,i}^{\rho,k} := \max_{\alpha \in [0,H]}\{z_{h,i}^k(\alpha)-\rho\alpha\}$ and $z^k_{h,i}(\alpha) := \EE^{\mu_{h,i}^0}[\hat{V}_{k,h+1}(s') ]_{\alpha}$, for any $i\in[d]$. 
Note that under \Cref{assumption:linear_mdp}, for any $\alpha\in[0,H]$, $z^k_{h,i}(\alpha)$ is the $i$-th element of the parameter of the following linear formulation, $[\PP_h^0[\hat{V}_{k,h+1}]_{\alpha}](s,a) =\la \bphi(s,a),\bz^k_h(\alpha) \ra$. Thus, we can estimate $\bz^k_h(\alpha)$ from data to get estimations of $z^k_{h,i}(\alpha), \forall i\in[d]$. To this end, we introduce the variance-weighted ridge regression regime to estimate $\bz^k_h(\alpha)$ as follows
\begin{align}
    \hat{\bz}_h^k(\alpha)& \textstyle=\argmin_{\bz \in \RR^d}\sum_{\tau=1}^{k-1} \bar{\sigma}^{-2}_{\tau,h} \Big(\bz^\top \bphi\big(s_h^\tau, a_h^\tau\big)- \big[\hat{V}_{k, h+1}^\rho\big(s_{h+1}^\tau\big)\big]_{\alpha}\Big)^2 + \lambda\Vert \bz \Vert_2^2\notag\\
    & \textstyle= \bSigma_{k,h}^{-1} \sum_{\tau=1}^{k-1} \bar{\sigma}^{-2}_{\tau,h} \bphi\big(s_h^\tau, a_h^\tau\big) \big[\hat{V}_{k, h+1}^\rho\big(s_{h+1}^\tau\big)\big]_{\alpha},
    \label{eq:optimistic_parameter_closed_form}
\end{align}
where {\small$\bSigma_{k,h} = \lambda\Ib + \sum_{\tau=1}^{k-1} \bar{\sigma}^{-2}_{\tau,h} \bphi\big(s_h^\tau, a_h^\tau\big)\bphi\big(s_h^\tau, a_h^\tau\big)^\top$}, $\bar{\sigma}_{\tau,h}$ are regression weights that will be formally introduced later. We then approximate $\bnu_h^{\rho,k}$ by solving the optimization problem element-wisely
\begin{align}
\label{eq:dual optimization}
    {\textstyle\hat{\nu}_{h,i}^{\rho, k} = \max_{\alpha \in [0,H]}\big\{\hat{z}_{h,i}^k (\alpha)-\rho\alpha\big\},\quad i\in[d]}.
\end{align} 
Further, we incorporate a bonus term $\hat{\Gamma}_{k,h}(s,a) = \beta \sum_{i=1}^d\phi_i(s,a)\sqrt{\mathbf{1}_i^\top \bSigma_{k,h}^{-1}\mathbf{1}_i}$, where $\beta = \widetilde{\cO}\big(H\sqrt{d\lambda} + \sqrt{d}\big)$, into the robust Q-function estimation.  We will prove in our analysis that the estimated Q-function $\hat{Q}^{\rho}_{k,h}$ in Line \ref{line:Q_hat} of \Cref{alg:DR-LSVI-UCB+} is an optimistic estimator for the optimal robust Q-function.  
Inspired by \citet{he2023nearly}, we also establish pessimistic estimated robust Q-functions by the same backward induction procedure, which will be helpful in constructing the variance estimator $\bar{\sigma}_{\tau,h}$ as shown in the next section. 
In particular, given $\check{V}^\rho_{k,h}$, we estimate 
\begin{align*}
    \check{\bz}_h^k(\alpha) & \textstyle= \argmin_{\bz \in \RR^d}\sum_{\tau=1}^{k-1} \bar{\sigma}^{-2}_{\tau,h} \big(\bz^\top \bphi\big(s_h^\tau, a_h^\tau\big)- \big[\check{V}_{k, h+1}^\rho\big(s_{h+1}^\tau\big)\big]_{\alpha}\big)^2 + \lambda\Vert \bz \Vert_2^2 %
    \\
    & \textstyle= \bSigma_{k,h}^{-1} \sum_{\tau=1}^{k-1} \bar{\sigma}^{-2}_{\tau,h} \bphi\big(s_h^\tau, a_h^\tau\big) \big[\check{V}_{k, h+1}^\rho\big(s_{h+1}^\tau\big)\big]_{\alpha}, %
\end{align*}
and then get the estimation 
\begin{align}
\label{eq:pessimistic_parameter}
    \textstyle \check{\nu}_{h,i}^{\rho, k} = \max_{\alpha \in [0,H]}\big\{\check{z}_{h,i}^k (\alpha)-\rho\alpha\big\},\quad i\in[d].
\end{align} 
Next, by incorporating a penalty term $\check{\Gamma}_{k,h}(s,a) = \bar{\beta} \sum_{i=1}^d\phi_i(s,a)\sqrt{\mathbf{1}_i^\top \bSigma_{k,h}^{-1}\mathbf{1}_i}$, 
where $\bar{\beta} = \widetilde{\cO}\big(H\sqrt{d\lambda} + \sqrt{d^3 H^3}\big)$, we get the pessimistic estimated Q-function $\check{Q}_{k,h}^{\rho}$ as Line \ref{line:Q_check} in \Cref{alg:DR-LSVI-UCB+}.
We note that \citet{liu2024minimax} also construct pessimistic robust Q-function estimations, but 1) they do not construct the estimation episodically, 2) their pessimistic estimators are used to get the optimal robust policy estimation. While ours are used to construct the variance estimator, as is shown in the next section. 

\begin{algorithm}[t]
    \caption{{\small Weighted Distributionally Robust Iterative Value Estimation with UCB (\algname)}\label{alg:DR-LSVI-UCB+}}
    \begin{algorithmic}[1]

        \STATE {\bf Initialization:} hyperparameters $\beta, \bar{\beta}, \widetilde{\beta} >0$ and $\lambda>0$. Set $k_{\text {last }}=0$; for each stage $h \in[H]$, set $\bSigma_{0, h}, \bSigma_{1, h}, \bLambda_{1, h} \leftarrow \lambda \Ib$ and set $\hat{Q}_{0, h}^\rho(\cdot, \cdot) \leftarrow H, \check{Q}_{0, h}^\rho(\cdot, \cdot) \leftarrow 0$

        \FOR {episode $k=1, \cdots, K$}
        
            \STATE Receive the initial state $s_1^k$

            \STATE Set $\hat{V}_{k,H+1}^\rho(\cdot) \leftarrow 0, \check{V}_{k,H+1}^\rho(\cdot) \leftarrow 0$ \label{line:value_function_initialization}
            
            \IF{there exists a stage $h^{\prime} \in[H]$ such that $\text{det}(\bSigma_{k, h^{\prime}}) \geq 2 \text{det}(\bSigma_{k_{\text {last}}, h^{\prime}})$} \label{line:update_rule}
            \FOR {stage $h=H, \cdots, 1$}
            \label{line:start backward induction}
                \IF{$h=H$} \label{line:backward_bellman_update_begin}
                
                    \STATE $\hat{\bnu}_h^{\rho, k} \leftarrow 0$, $\check{\bnu}_h^{\rho, k} \leftarrow 0$ \label{line:nu_initialization}

                \ELSE

                    \STATE \label{line:compute parameter}
                    Compute $\hat{\nu}_{h,i}^{\rho, k},\forall i\in[d]$ according to \eqref{eq:dual optimization} and $\check{\nu}_{h,i}^{\rho, k},\forall i\in[d]$ according to \eqref{eq:pessimistic_parameter}.

                \ENDIF

                    \STATE
                    \label{line:Q_hat}
                    {\small $\hat{Q}_{k,h}^\rho(s,a) \leftarrow \min \big\{ r_h(s,a) + \bphi(s,a)^\top \hat{\bnu}_h^{\rho,k} + \hat{\Gamma}_{k,h}(s,a) , \hat{Q}_{k-1, h}^\rho(s,a), H-h+1 \big\}\ind\{s\neq s_f\}$ }

                    \STATE 
                    \label{line:Q_check}
                    $\check{Q}_{k,h}^\rho(s,a) \leftarrow \max \big\{ r_h(s,a) + \bphi(s,a)^\top \check{\bnu}_h^{\rho,k} - \check{\Gamma}_{k,h}(s,a) , \check{Q}_{k-1, h}^\rho(s,a), 0 \big\}\ind\{s\neq s_f\}$
                    
                    \STATE Set the last updating episode $k_{\text {last }}\leftarrow k$
                \STATE $\hat{V}_{k,h}^\rho(s) \leftarrow \max_a \hat{Q}_{k,h}^\rho(s, a), \quad \check{V}_{k,h}^\rho(s) \leftarrow \max_a \check{Q}_{k,h}^\rho(s,a)$

                \STATE $\pi_{h}^{k} (s)\leftarrow \argmax_{a\in \mathcal{A}}\hat{Q}_{k,h}^\rho(s, a)$
              
            \ENDFOR \label{algline:end of the first loop}
            
            \ELSE 
                \STATE $\hat{V}_{k,h}^\rho(s)  \leftarrow  \hat{V}_{k-1,h}^\rho(s),~\check{V}_{k,h}^\rho(s)  \leftarrow  \check{V}_{k-1,h}^\rho(s)$,~$\pi_{h}^{k}(s)\leftarrow\pi_{h}^{k-1} (s)$   for all $h\in[H]$
            \ENDIF \label{line:first_stage_end}
            \FOR{stage $h=1, \cdots, H$}\label{algline:begin of the second loop}

                \STATE Take the action $a_h^k \leftarrow \pi_{h}^{k}(s_h^k)$
    
                \STATE Calculate the estimated variance $\sigma_{k, h}$ according to \eqref{eq:estimated_variance_optimal_v_function} and $\bar{\sigma}_{k, h}$ according to \eqref{eq:weight}
    
                \STATE $\bSigma_{k+1, h} \leftarrow \bSigma_{k, h}+\bar{\sigma}_{k, h}^{-2} \bphi\big(s_h^k, a_h^k\big) \bphi\big(s_h^k, a_h^k\big)^{\top}$, $\bLambda_{k+1, h} \leftarrow \bLambda_{k, h}+ \bphi\big(s_h^k, a_h^k\big) \bphi\big(s_h^k, a_h^k\big)^{\top}$
    
                \STATE Receive next state $s_{h+1}^k$
            
            \ENDFOR \label{algline:end of the second loop}

        \ENDFOR
    \end{algorithmic}
\end{algorithm}

\subsection{Variance Estimator with Refined Dependence on Problem Parameters}
\label{sec:variance_estimator}

In this section, we construct the weights used in \eqref{eq:optimistic_parameter_closed_form} and aim to get an optimistic estimator for the variance of the optimal robust value function, $\mathbb{V}_h V_{h+1}^{*,\rho}$. Inspired by \citet{he2023nearly}, the variance estimator at episode $k$ should be a uniform variance upper bound for all subsequent episodes. To obtain the optimistic estimator for $\mathbb{V}_h V_{h+1}^{*,\rho}$, we first solve regression problems to obtain the estimator for $\mathbb{V}_h \hat{V}_{k,h+1}^{\rho}$, which is denoted as $\bar{\mathbb{V}}_h \hat{V}_{k,h+1}^{\rho}$. Then we analyze the error between $\mathbb{V}_h V_{h+1}^{*,\rho}$ and $\bar{\mathbb{V}}_h \hat{V}_{k,h+1}^{\rho}$ to finish the construction. Different from \citet[Equation (5.2)]{liu2024minimax}, the variance estimator here is not trivially constructed from subtracting a specific penalty term because we should guarantee the monotonicity of estimated variance for the online exploration.

The variance of estimated optimistic value function $\hat{V}_{k,h+1}^\rho$ can be denoted by
\begin{align}
\label{eq:true_variance_optimistic_v_function}
    {\textstyle\big[\mathbb{V}_h \hat{V}_{k,h+1}^\rho\big](s,a)=\big[\mathbb{P}_h^0 \big(\hat{V}_{k,h+1}^\rho\big)^2\big](s,a)-\big(\big[\mathbb{P}_h^0 \hat{V}_{k,h+1}^\rho\big](s,a)\big)^2}.
\end{align}
Under \Cref{assumption:linear_mdp}, {\small$\mathbb{P}_h^0 \big(\hat{V}_{k,h+1}^\rho\big)^2$} and {\small$\mathbb{P}_h^0 \hat{V}_{k,h+1}^\rho$} on the RHS of \eqref{eq:true_variance_optimistic_v_function} are linear in $\bphi(s,a)$ based on \citet[Proposition 2.3]{jin2020provably}. Thus we can approximate them as follows
\begin{align*}
    {\textstyle\big[\mathbb{V}_h \hat{V}_{k,h+1}^\rho\big](s, a) \approx \big[\bar{\mathbb{V}}_h \hat{V}_{k,h+1}^\rho\big](s, a) = \big[\bphi(s,a)^\top \widetilde{\bz}^k_{h,2} \big]_{[0,H^2]} -\big[\bphi(s,a)^\top \hat{\bz}^k_{h,1}\big]_{[0,H]}^2},
\end{align*}
where $\hat{\bz}^k_{h,1}$ and $\widetilde{\bz}^k_{h,2}$ are solutions to the following ridge regression problems
\begin{align*}
    \widetilde{\bz}^k_{h,2} &\textstyle= \argmin_{\bz \in \RR^d}\sum_{\tau=1}^{k-1}  \big(\bz^\top \bphi\big(s_h^\tau, a_h^\tau\big) - \big(\hat{V}_{k, h+1}^\rho\big(s_{h+1}^\tau\big)\big)^2\big)^2 + \lambda\Vert \bz \Vert_2^2, \\
    \hat{\bz}^k_{h,1} &\textstyle= \argmin_{\bz \in \RR^d}\sum_{\tau=1}^{k-1} \big(\bz^\top \bphi\big(s_h^\tau, a_h^\tau\big) - \hat{V}_{k, h+1}^\rho\big(s_{h+1}^\tau\big)\big)^2 + \lambda\Vert \bz \Vert_2^2.
\end{align*}
Different from the variance estimation in \citet{he2023nearly}, we construct both $\tilde{\bz}_{h,2}^k$ and $\hat{\bz}_{h,1}^k$  by solving vanilla ridge regressions, instead of variance-weighted ridge regressions.
This specific choice of parameter estimation will simplify our analysis of the variance estimation error, %
while fully capture the variance information.
Now we can construct $\sigma_{k,h}$, which is the estimated variance of the optimal robust value function $V_h^{*,\rho}$ in episode $k$, as follows
\begin{align} 
\label{eq:estimated_variance_optimal_v_function}
    \sigma_{k,h} \textstyle= \sqrt{\big[\bar{\mathbb{V}}_h \hat{V}_{k, h+1}^\rho\big](s_h^k, a_h^k) + E_{k,h} + d^3 H \cdot D_{k,h} + 1/2},
\end{align}
where $E_{k,h}$ represents the error between the estimated variance and the true variance of $\hat{V}_{k,h+1}^\rho$, and $D_{k,h}$ represents the error between the true variance of $\hat{V}_{k,h+1}^\rho$ and the true variance of $V_{h+1}^{*,\rho}$. We define $E_{k,h}, D_{k,h}$ as follows
\begin{align*}
    E_{k,h} &\textstyle = \min \big\{\widetilde{\beta} \big\|\bphi\big(s_h^k, a_h^k\big)\big\|_{\bLambda_{k, h}^{-1}}, H^2\big\} + \min \big\{2 H \bar{\beta}\big\|\bphi\big(s_h^k, a_h^k\big)\big\|_{\bLambda_{k, h}^{-1}}, H^2\big\}, \\
    D_{k,h} &\textstyle = \min \big\{4 H \big(\bphi\big(s_h^k, a_h^k\big)^\top \hat{\bz}^k_{h,1} - \bphi\big(s_h^k, a_h^k\big)^\top \check{\bz}^k_{h,1}
    + 2 \bar{\beta}\big\|\bphi\big(s_h^k, a_h^k\big)\big\|_{\bLambda_{k, h}^{-1}}\big), H^2\big\},
\end{align*}
where $\bLambda_{k,h} = \lambda\Ib + \sum_{\tau=1}^{k-1} \bphi\big(s_h^\tau, a_h^\tau\big)\bphi\big(s_h^\tau, a_h^\tau\big)^\top$, $\bar{\beta} = \widetilde{\cO}\big(H\sqrt{d\lambda} + \sqrt{d^3 H^3}\big)$ and $\widetilde{\beta} = \widetilde{\cO}\big(H^2\sqrt{d\lambda} + \sqrt{d^3 H^6}\big)$, and $\check{\bz}^k_{h,1}$ is the solution of the following regression problems
\begin{align*}
    \check{\bz}^k_{h,1} \textstyle= \argmin_{\bz \in \RR^d}\sum_{\tau=1}^{k-1}  \big(\bz^\top \bphi\big(s_h^\tau, a_h^\tau\big) - \check{V}_{k, h+1}^\rho\big(s_{h+1}^\tau\big)\big)^2 + \lambda\Vert \bz \Vert_2^2.
\end{align*}
Finally, we construct weights for the variance-weighted ridge regression problem \eqref{eq:optimistic_parameter_closed_form} as follows
\begin{align}
\label{eq:weight}
   \forall (k,h)\in[K]\times[H], \quad \bar{\sigma}_{k, h} \textstyle = \max \big\{\sigma_{k, h}, 1, \sqrt{2d^3H^2} \big\|\bphi\big(s_h^k, a_h^k\big)\big\|_{\bSigma_{k, h}^{-1}}^{1/2} \big\}.
\end{align}
\citet{he2023nearly} construct a similar weight in the form 
{\small$\bar{\sigma}_{k,h} = \max \big\{\sigma_{k, h}, H, 2d^3H^2 \big\|\bphi\big(s_h^k, a_h^k\big)\big\|_{\bSigma_{k, h}^{-1}}^{1/2} \big\}$}. 
Differently, the second term of our constructed weight in \eqref{eq:weight} is $1$, instead of $H$. This is important in achieving a tighter dependence on $H$. The intuition is that, when $k$ is large, $\bar{\sigma}_{k,h}$ should be close to the  variance of the optimal robust value function. According to the `Range Shrinkage' phenomenon unique to DRMDPs, which will be introduced in the next section, the true variance is in the order of $\cO(1)$ when $\rho=\cO(1)$. To get a precise variance estimation, $\bar{\sigma}_{k,h}$ should be in the same order of the true variance. Moreover, a constant order lower bound on $\bar{\sigma}_{k,h}$ will also ensure the weight  will not cause any inflation in the weighted regression \eqref{eq:optimistic_parameter_closed_form}. As for the third term, we choose it to be tight without infecting our analysis. We refer to more details of the analysis in the proof of \Cref{lem:est_var_upper_bound}.

\subsection{Algorithm Interpretation}
In this section, we provide several remarks to fully interpret \Cref{alg:DR-LSVI-UCB+}.
\begin{remark}
\label{remark:rare-switch}
    We highlight that \Cref{alg:DR-LSVI-UCB+} is the first algorithm adopting `rare-switching' update strategy for distributionally robust RL. Different from \citet{he2023nearly}, the `rare-switching' condition on Line \eqref{line:update_rule} is set at the beginning of each episode. This is achieved by our variance estimator design, which is independent of the parameter $\bz_h^k(\alpha)$ update. The update rule on Line \eqref{line:update_rule}
    determines whether to update robust Q-function estimations and switch to a new policy for the current episode, and leads to two advantages, 1) the number of times solving the ridge regression \eqref{eq:optimistic_parameter_closed_form} and dual optimization \eqref{eq:dual optimization} significantly decreases, which constitute the main computation cost of \Cref{alg:DR-LSVI-UCB+}, and 2) in real application scenarios where policy switching is costly or risky, \Cref{alg:DR-LSVI-UCB+} possesses low policy switching property. We refer the readers to \Cref{prop: low switching} and \Cref{remark: low switching} for more details. 
\end{remark}

\begin{remark}
\label{remark:bonus term}
    In Line \ref{line:compute parameter}, we estimate $\bnu_{h}^{\rho, k}$ element-wisely, and thus the estimator $\hat{\bnu}_{h}^{\rho, k}$ is derived from $d$ separate variance-weighted ridge regressions \eqref{eq:optimistic_parameter_closed_form} and dual optimizations \eqref{eq:dual optimization}. This leads to the specific form of bonus term $\hat{\Gamma}_{k,h}(s,a) = \beta \sum_{i=1}^d\phi_i(s,a)\sqrt{\mathbf{1}_i^\top \bSigma_{k,h}^{-1}\mathbf{1}_i}$, which is actually an upper bound of the robust estimation error (see \Cref{lem:error_bound} and its proof) at episode $k$. Though the bonus term resembles that in \citet{liu2024distributionally}, we highlight that the sampling covariance matrix $\bSigma_{k,h}$ in $\hat{\Gamma}_{k,h}(s,a)$ is indeed a variance-weighted one. The specific form of the bonus term leads to the new variance-weighted $d$-rectangular robust estimation error defined in \eqref{eq:AveSubopt_thm_origin}.
\end{remark}

\begin{remark}%
\label{remark:monotonic update}
On Line \ref{line:Q_hat} and \ref{line:Q_check}, we adopt a monotonic Q-function update strategy, such that the estimated optimistic (pessimistic) robust value function is monotonically decreasing (increasing) to the optimal robust value function. 
This strategy is to make sure that the variance estimator $\sigma_{k,h}$ at any episode $k\in[H]$ is a uniform upper bound for those in the subsequent episodes, which would be helpful in bounding the estimation error arising from the variance-weighted ridge regression \eqref{eq:optimistic_parameter_closed_form}.
This idea is first introduced by \citet{azar2017minimax} for standard tabular MDPs and then utilized by \citet{he2023nearly} for standard linear MDPs. 
This is the first time it is utilized in the online linear DRMDP setting, where the episodic estimation regime proposes additional requirement on the variance estimator construction compared to the offline setting studied in \citet{liu2024minimax}.

\end{remark}

\section{Theoretical Analysis}
\label{sec:theoretical_analysis}
We now provide theoretical results on the upper and lower bounds on the suboptimality of \Cref{alg:DR-LSVI-UCB+}.

\begin{theorem}
\label{thm:AveSubopt_origin}
Under \Cref{assumption:linear_mdp,assumption:fail_state}, set $\lambda = 1/H^2$, then for any fixed $\delta \in (0,1)$ and $\rho \in (0,1]$, with probability at least $1-\delta$, the average suboptimality of \algname\ satisfies
\begin{align}
\label{eq:AveSubopt_thm_origin}
   \text{AveSubopt}(K) \textstyle \leq 2 \sqrt{2 H^3 \log (6 / \delta)/K} + \frac{4\beta}{K}\underbrace{\textstyle \sum_{k=1}^K\sum_{h=1}^H\sum_{i=1}^d\phi_{h,i}^k\sqrt{\mathbf{1}_i^\top \bSigma_{k,h}^{-1}\mathbf{1}_i}}_{\text{variance-weighted}~d\text{-rectangular estimation error}},
\end{align}
where $\beta=\widetilde{\cO}(\sqrt{d})$,  $\phi_{h,i}^k$ is the $i$-th element of $\bphi_{h}^k = \bphi(s_h^k, a_h^k)$ and $\mathbf{1}_i$ is the one-hot vector with its $i$-th entry being 1.
\end{theorem}

Recall from \Cref{remark:bonus term}, the quantity {\small$\sum_{i=1}^d\phi_{h,i}^k\sqrt{\mathbf{1}_i^\top\bSigma_{k,h}^{-1}\mathbf{1}_i}$} in \eqref{thm:AveSubopt_origin} originates from solving $d$ separate variance-weighted ridge regressions at step $h$ in episode $k$. 
A similar term also appears in the Theorem 5.1 of \citet{liu2024distributionally}. Differently, the the quantity {\small$\sum_{i=1}^d\phi_{h,i}^k\sqrt{\mathbf{1}_i^\top\bSigma_{k,h}^{-1}\mathbf{1}_i}$} is based on the variance-weighted sampling covariance matrix $\bSigma_{k,h}$, rather than the vanilla sampling covariance matrix $\bLambda_{k,h}$ as in \citet{liu2024distributionally}. 
In order to further bound \eqref{thm:AveSubopt_origin}, we need to take a closer examination of the variance estimator.
Intuitively, when episode $k$ is large, the variance estimator should be close to the variance of the optimal robust value function. Recent study \citep{liu2024minimax} shows a  `Range shrinkage' phenomenon in the $d$-rectangular linear DRMDP (refer to \Cref{lem:range_shrinkage}), stating that the range of any robust value function satisfies
$\max_{s\in\cS}V_{h}^{\pi, \rho}(s) - \min_{s\in\cS}V_{h}^{\pi, \rho}(s)\leq \min\{1/\rho,H\}, \forall (\pi, h, \rho)\in\Pi\times[H]\times(0,1]$. This implies that the variance of the optimal robust value function is upper bounded by $\min\{1/\rho,H\}$.
Thus, when $k$ is large, we can expect $\bar{\sigma}_{k,h}\lesssim \widetilde{\cO}(\min\{1/\rho,H\})$ and hence $\bSigma_{k,h}^{-1}\preceq \widetilde{\cO}(\min\{1/\rho^2,H^2\})\bLambda_{k,h}^{-1}$. 
Based on this idea, next we rigorously bound \eqref{thm:AveSubopt_origin} under the same setting as the Corrollary 5.3 of \citet{liu2024distributionally}, and formally show in the following theorem and remark that the variance information leads to a tighter dependence on $H$ compared to \citet{liu2024distributionally}.

\begin{theorem}
\label{thm:DRLSVIUCB}
Assume that there exists an absolute constant $c>0$, such that for all $(\pi, h) \in \Pi \times [H]$
\begin{align}
\label{assumption:lambda_lower_bound}
    \EE_{\pi}^{P^0}\big[\bphi(s_h,a_h)\bphi(s_h, a_h)^\top\big] \geq c/d \cdot \Ib.
\end{align}
Then under the same setting in \Cref{thm:AveSubopt_origin} and the additional assumption in \eqref{assumption:lambda_lower_bound}, 
for any fixed $\delta \in (0,1)$, with probability at least $1-\delta$, the average suboptimality of \algname\ satisfies
\begin{align}
\label{eq:AveSubopt_thm}
{\textstyle \text{AveSubopt}(K) \leq \widetilde{\cO}\big( \big({d H \cdot \min \big\{1/{\rho}, H\big\} + H^{3/2} }\big)/{\sqrt{K}} + {d^{15} H^{13}}/{K}\big).
}   
\end{align}
\end{theorem}

\begin{remark}
\label{remark:simplified upper bound}
    When $d\geq H$ and the total number of episodes $K$ is sufficiently large, the average suboptimality can be simplified as $\widetilde{\cO}\{dH\min\{1/\rho, H\}/{\sqrt{K}}\}$. %
    Note that under the same assumption in \eqref{assumption:lambda_lower_bound}, \citet{liu2024distributionally} prove that the average suboptimality of their algorithm DR-LSVI-UCB is of the order $\widetilde{\cO}(d^2H^2/\sqrt{K})$. Thus, \algname\ improves the state-of-the-art result by $\cO(dH/\min\{1/\rho,H\})$.
    Moreover, we highlight that the upper bound \eqref{eq:AveSubopt_thm} depends on the uncertainty level $\rho$, which arises from the `Range Shrinkage' phenomenon. When $\rho$ increases from 0 to 1, the suboptimality decreases up to a factor of $\cO(H)$.
\end{remark}

\begin{remark}
The assumption \eqref{assumption:lambda_lower_bound} is actually imposed on the DRMDP, requiring that the environment we encounter is exploratory enough. We would like to note that this assumption is necessary in deriving our upper bound, since the elliptical potential lemma \citep[Lemma 11]{abbasi2011improved}, which is critical in deriving upper bounds in linear bandits and linear MDPs, does not apply in the analysis of linear DRMDPs. We note that the previous work \citep{liu2024distributionally} also used this assumption to get the final upper bound for their algorithm. Moreover, the assumption \eqref{assumption:lambda_lower_bound} can be deemed as an online version of the well-known full-type coverage assumption on the offline dataset in offline (non-) robust RL. Specifically, in the context of standard offline RL, \citet{chen2019information, wang2020statistical, xie2021bellman} assume the offline dataset should cover the distribution measure induced by any policy under the nominal environment. In the context of offline robust RL, \citet{panaganti2022robust, panaganti2024model, zhangsoft} assume that the offline dataset should cover the distribution measure induced by any policy under any transition kernel in the uncertainty set. It would be an interesting future research direction to study if assumption \eqref{assumption:lambda_lower_bound} can be relaxed.
\end{remark}

Notably, when $\rho = \cO(1)$, the suboptimality of \algname\ is of order $\cO(dH/\sqrt{K})$. After multiplying $K$ to recover the cumulative suboptimality, it is smaller than the minimax lower bound for standard linear MDP, $\Omega(d\sqrt{H^3K})$ \citep{zhou2021nearly}.
To assess the optimality of \algname, we show an information-theoretic lower bound for the online linear DRMDP setting in the following theorem. 
\begin{theorem}\label{thm:lower_bound}
    Let uncertainty level $\rho\in(0,3/4]$, $H\geq 6$, and $K\geq 9d^2H/32$. Then for any algorithm, there exists a $d$-rectangular linear DRMDP parameterized by $\bxi=(\bxi_1,\cdots,\bxi_{H-1})$ such that the expected average suboptimality is lower bounded as follows:
    \begin{align}
    \label{eq:lower bound}
       {\textstyle \EE_{\bxi} \text{AveSubopt}(M_{\bxi},K) \geq \Omega\big(\big({dH^{1/2}\cdot \min\{1/{\rho},H\}}\big)/{\sqrt{K}}\big),
       } 
    \end{align}
    where $\EE_{\bxi}$ denotes the expectation over the probability distribution generated by the algorithm and the nominal environment.
\end{theorem}

\begin{remark}
    \Cref{thm:lower_bound} shows that \algname\ is near-optimal up to a factor of $\cO(\sqrt{H})$ among the full range of uncertainty level. Moreover, when $\rho \rightarrow 0$, the linear DRMDP degrades to the standard linear MDP, and \eqref{eq:lower bound} matches the information-theoretic lower bound, $\Omega(d\sqrt{H^3K})$, for standard linear MDPs \citep{zhou2021nearly} after multiplying $K$ to recover the cumulative regret. When $\rho = \cO(1)$, \eqref{eq:lower bound} is realized to $\Omega(dH^{1/2}/\sqrt{K})$, which has a factor of $\cO(H)$ decrease compared to the lower bound for standard linear MDP. %
\end{remark}

Next, we study the deployment complexity of \Cref{alg:DR-LSVI-UCB+}, which constitutes two sources of cost. The first source is the policy switching cost, say, the total number of changes in the exploration policy. This might be the main bottleneck in applications where changing the exploration policy is costly or risky \citep{bai2019provably, wang2021provably}. The second source is the computation cost in solving the dual optimization in \eqref{eq:dual optimization}. 
Recall in \Cref{remark: low switching} we discuss that \Cref{alg:DR-LSVI-UCB+} adopts the `rare-switching' update strategy, which significantly reduces the two sources of cost. Next, we formally define them as follows.

\begin{definition}[Global Switching Cost]
\label{def:Global Switching Cost}
    We define the {\it global switching cost} of an algorithm that runs for $K$ episodes as $N_{\text{switch}}^{gl} \textstyle:= \sum_{k=1}^K\ind\{\pi_k\neq\pi_{k+1}\}$.
\end{definition}

\begin{definition}[Dual Oracle]
\label{def:dual oracle}
    We assume access to a maximization oracle, which takes a function $z:[0,H]\rightarrow \RR$ and a fixed constant $\rho>0$ as input, and outputs the maximum value $z_{\max}$ and the maximizer $\alpha_{\max}$ defined as $z_{\max} \textstyle= \max_{\alpha\in[0,H]}\{z(\alpha)-\rho\alpha\}$ and $\alpha_{\max} = \argmax_{\alpha\in[0,H]}\{z(\alpha)-\rho\alpha\}$.
For an algorithm, we define the {\it oracle complexity} as the number of calls of the dual oracle. Finally, we show that \algname\ admits low switching cost and low oracle complexity. 
\end{definition}

Next, we formally present theoretical results on the deployment complexity of \Cref{alg:DR-LSVI-UCB+}.
\begin{proposition}
\label{prop: low switching}
    Under the same setting as \Cref{thm:AveSubopt_origin}, the switching cost of \algname\ is upper bounded by $dH\log(1+H^2K)$, and the oracle complexity of \algname\ is upper bounded by $2d^2H\log(1+H^2K)$.
\end{proposition}

\begin{remark}
\label{remark: low switching}
    The switching cost of the state-of-the-art algorithm DR-LSVI-UCB \citep{liu2024distributionally} is $K$ and the oracle complexity is $dK$. 
    Thus, \algname\ improves both the switching cost and oracle cost by a factor of $K$. 
    We highlight that different from the standard linear MDP setting, where the main computation complexity only comes from the policy update \citep{he2023nearly}, in the linear DRMDP setting, the calls of dual oracle, besides policy updates, are also a main source of computational burden. The update rule in Line \ref{line:update_rule} guarantees that \algname\ calls the dual oracle and updates the policy only when the criterion is met. Actually, \Cref{alg:DR-LSVI-UCB+} is the first DRMDP algorithm that admits low deployment complexity.
\end{remark}

\section{Discussion on the Tightness of the Upper and Lower Bounds}
\label{sec:discussion_upper_lower_bound}

There is a $\widetilde{\cO}(\sqrt{H})$ gap between the upper bound presented in \Cref{remark:simplified upper bound} and the lower bound derived in \Cref{thm:lower_bound}. 
We note that in our current analysis, we individually bound each term in the variance-weighted $d$-rectangular estimation error in \eqref{eq:AveSubopt_thm_origin}. 
However, in the analysis of non-robust MDPs \citep{azar2017minimax, jin2018q, he2023nearly} and tabular DRMDPs \citep{lu2024distributionally}, a tight dependence on $H$ is often achieved by exploiting the total variance law of the value function at each episode. We conjecture a tight upper bound can be achieved by first bounding the variance-weighted $d$-rectangular estimation error as a whole by the square root of the total variance and then invoking the total variance law.  In particular, inspired by the total variance law in Lemma C.6 of \citet{lu2024distributionally}, the total variance  should be in the order of $\cO(H\min\{1/\rho, H\})$. Together with an additional $\sqrt{H}$ arising in the suboptimality analysis, we conjecture the dependence of the upper bound on $H$ could be improved to  $\cO(\sqrt{H^2\min\{1/\rho, H\}})$.

Based on the conjectured total variance analysis, when $\rho = \cO(1/H)$, the improved dependence on $H$ is in the order of $\cO(H^{3/2})$, matching the lower bound we present in \eqref{eq:lower bound}. This implies that our lower bound is tight and our upper bound is loose by $\widetilde{\cO}(\sqrt{H})$. When $\rho = \cO(1)$, the improved dependence on $H$ is in the order of $\cO(H)$, which means the total variance analysis does not further improve the upper bound. In this case, we conjecture that a tighter lower bound is needed to showcase the fundamental limit of online linear DRMDPs. %
\Cref{table:comparison of upper and lower bounds} provides an illustration of our conjecture and comparison. Currently, we find essential difficulties in relating the variance-weighted $d$-rectangular estimation error with the total variance to show a tighter upper bound.
We leave the improvement of $\cO(\sqrt{H})$ on both upper and lower bounds for future research.
\begin{table}[t]
\centering
\caption{Summary of the upper and lower bounds of \algname, and a conjectured minimax lower bound. The bound in red represents it matches the conjectured minimax lower bound.\label{table:comparison of upper and lower bounds}}
\begin{tabular}{lcc}
\hline
& $\rho=\cO(1/H)$                      & $\rho = \cO(1)$                            \\ \hline
Upper Bound \eqref{eq:AveSubopt_thm}                      & $\widetilde{\cO}\Big(\frac{dH^2}{\sqrt{K}}\Big)$ & {\color{red} $\widetilde{\cO}\Big(\frac{dH}{\sqrt{K}}\Big)$ }\\ 
Lower Bound \eqref{eq:lower bound}                         & {\color{red} $\Omega\Big(\frac{dH^{3/2}}{\sqrt{K}}\Big)$ }   & $\Omega(\frac{dH^{1/2}}{\sqrt{K}})$      \\ 
\begin{tabular}[c]{@{}l@{}}Minimax Lower Bound\\ (Conjectured)\end{tabular} & $\Omega\Big(\frac{dH^{3/2}}{\sqrt{K}}\Big)$    & $\Omega\Big(\frac{dH}{\sqrt{K}}\Big)$    \\ \hline
\end{tabular}
\end{table}

\section{Experiments on Simulated Linear DRMDPs}
We conduct numerical experiments to illustrate the performances of our proposed algorithm, We-DRIVE-U, and compare it with the state-of-the-art algorithm for $d$-rectangular linear DRMDPs, DR-LSVI-UCB \citep{liu2024distributionally}, as well as their non-robust counterpart, LSVI-UCB \citep{jin2020provably}. All numerical experiments were conducted on a MacBook Pro with a 2.6 GHz 6-Core Intel CPU.

We leverage the simulated linear MDP setting proposed by \citet{liu2024distributionally}. For completeness, we recall the experiment setting as follows. The source and target linear MDP environment are shown in \Cref{fig:mdp_5states} and \Cref{fig:perturbed_mdp}. The state space is $\cS = \{x_1,\cdots,x_5\}$ and action space $\cA=\{-1,1\}^4\subset\RR^4$. At each episode, the initial state is always $x_1$, and it can transit to $x_2, x_4, x_5$ with probability defined in the figures. $x_2$ is an intermediate state from which the next state can be $x_3, x_4, x_5$. $x_4$ is the fail state with reward $0$ and $x_5$ is an absorbing state with reward 1. For the reward functions and transition probabilities, they are designed to depend on $\la \bxi,a\ra$, where $\bxi\in\RR^4$ is a hyperparameter controls the MDP instances. The target environment is constructed by only perturbing the transition probability at $x_1$ of the source domain, and the extend of perturbation is controlled by a hyperparameter $q\in(0,1)$. We refer more details on the construction of the linear DRMDP to the Supplementary A.1 of \cite{liu2024distributionally}.
\begin{figure*}[t]
    \centering
    \subfigure[The source MDP environment.]{
        \begin{tikzpicture}[->,>=stealth',shorten >=1pt,auto,node distance=3.4cm,thick]
            \tikzstyle{every state}=[fill=red,draw=none,text=white,minimum size=0.5cm]
            \node[state] (S1) {$x_1$};
            \node[state] (S2) [right of=S1] {$x_2$};
            \node[state] (S3) [right of=S2] {$x_3$};
            \node[state] (S4) [above=2cm of S2] {$x_4$};
            \node[state] (S5) [below=2cm of S2] {$x_5$};
            
            \path   (S1) edge[draw=blue!20] node[below] {\tiny $(1-p)(1-\delta-\la\xi,a\ra)$} (S2)
                         edge[draw=blue!20] node[below] {\tiny$p(1-\delta-\la\xi,a\ra)$} (S4)
                         edge[draw=blue!20] node[above] {\tiny$\delta+\la\xi,a\ra$} (S5)
                    (S2) edge[draw=blue!20] node[below] {\tiny$(1-p)(1-\delta-\la\xi,a\ra)$} (S3)
                         edge[draw=blue!20] node[above] {\tiny$p(1-\delta-\la\xi,a\ra)$} (S4)
                         edge[draw=blue!20] node[below] {\tiny$\delta+\la\xi,a\ra$} (S5)
                    (S3) edge[draw=blue!20] node[below] {\tiny$1-\delta-\la\xi,a\ra$} (S4)
                         edge[draw=blue!20] node[above] {\tiny$\delta+\la\xi,a\ra$} (S5)
                    (S4) edge[draw=blue!20] [loop above] node {\tiny 1} (S4)
                    (S5) edge[draw=blue!20] [loop below] node {\tiny 1} (S5);
        \end{tikzpicture}
        \label{fig:mdp_5states}
    }
    \subfigure[The target MDP environment.]{
        \begin{tikzpicture}[->,>=stealth',shorten >=1pt,auto,node distance=3.4cm,thick]
            \tikzstyle{every state}=[fill=red,draw=none,text=white,minimum size=0.5cm]
            \node[state] (S1) {$x_1$};
            \node[state] (S2) [right of=S1] {$x_2$};
            \node[state] (S3) [right of=S2] {$x_3$};
            \node[state] (S4) [above=2cm of S2] {$x_4$};
            \node[state] (S5) [below=2cm of S2] {$x_5$};
            
            \path   (S1) edge[draw=blue!20] node[below] {\tiny$(1-\delta-\la\xi,a\ra)$} (S2)
                         edge[draw=blue!20] node[below] {\tiny$q(\delta+\la\xi,a\ra)$} (S4)
                         edge[draw=blue!20] node[above] {\tiny$(1-q)(\delta+\la\xi,a\ra)$} (S5)
                    (S2) edge[draw=blue!20] node[below] {\tiny$(1-p)(1-\delta-\la\xi,a\ra)$} (S3)
                         edge[draw=blue!20] node[above] {\tiny$p(1-\delta-\la\xi,a\ra)$} (S4)
                         edge[draw=blue!20] node[below] {\tiny$\delta+\la\xi,a\ra$} (S5)
                    (S3) edge[draw=blue!20] node[below] {\tiny$1-\delta-\la\xi,a\ra$} (S4)
                         edge[draw=blue!20] node[above] {\tiny$\delta+\la\xi,a\ra$} (S5)
                    (S4) edge[draw=blue!20] [loop above] node {\tiny 1} (S4)
                    (S5) edge[draw=blue!20] [loop below] node {\tiny 1} (S5);
        \end{tikzpicture}
        \label{fig:perturbed_mdp}
    }
    \caption{The source and the target linear MDP environments. The value on each arrow represents the transition probability. For the source MDP, there are five states and three steps, with the initial state being  $x_1$, the fail state being $x_4$, and $x_5$ being an absorbing state with reward 1. The target MDP on the right is obtained by perturbing the transition probability at the first step of the source MDP, with others remaining the same. }
\end{figure*}
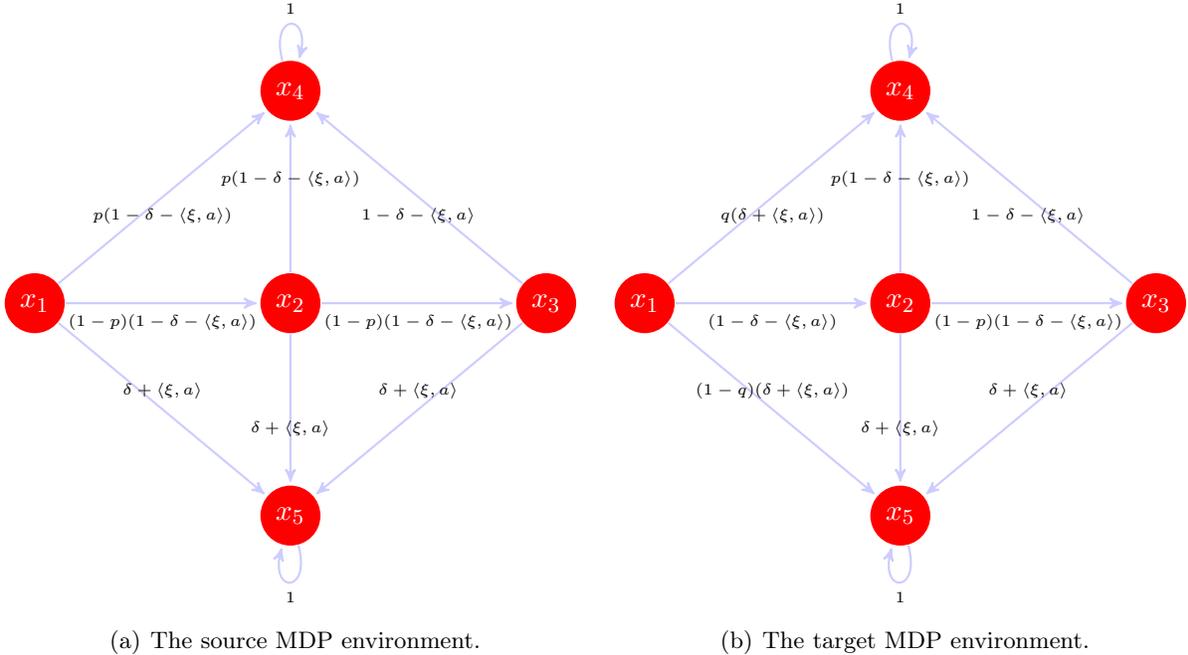

We set $\bxi = (1/\|\bxi\|_1, 1/\|\bxi\|_1, 1/\|\bxi\|_1, 1/\|\bxi\|_1)^{\top}$ and consider different choices of $\|\bxi\|_1$ from the set $\{0.1, 0.2, 0.3\}$. Following the implementation in \cite{liu2024distributionally}, we use heterogeneous uncertainty level and set $\rho_{1,4}=0.5$ and $\rho_{h,i}=0$ for all other cases.
We set the number of interactions with the nominal environment to $200$.  
We evaluate policies learned by \algname, DR-LSVI-UCB \citep{liu2024distributionally} and LSVI-UCB \citep{jin2020provably} by the accumulative rewards achieved in the target domain, which are illustrated in \Cref{fig:simulation-results-app}.
\Cref{fig:simulation-results-app} shows that: 1) policies learned by We-DRIVE-U are robust to environmental perturbation, and the extent of the robustness depends on the pre-specified parameter $\rho$; 2) In most cases, We-DRIVE-U outperforms DR-LSVI-UCB, meaning it being more robust to environment perturbation. Moreover, \Cref{tab:switch_time} demonstrates the low-switching property of We-DRIVE-U. During $200$ interactions of the training process, We-DRIVE-U switches policies only around $24$ times, which stands in stark contrast to the $200$ policy switches by LSVI-UCB and DR-LSVI-UCB. These numerical results prove the superiority of our proposed algorithm We-DRIVE-U and align well with our theoretical findings.

\begin{figure*}[t]%
    \centering
      \subfigure[$\Vert\xi\Vert_1 = 0.1$, $\rho_{1,4}=0.1$]{\includegraphics[scale=0.39]{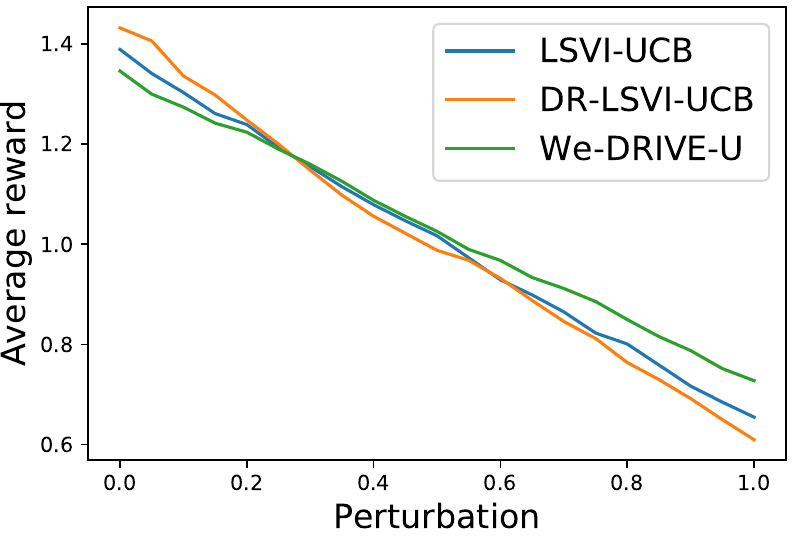}
      \label{fig:simulated_MDP_xi01_rho05}}
      \subfigure[$\Vert\xi\Vert_1 = 0.1$, $\rho_{1,4}=0.2$]{\includegraphics[scale=0.39]{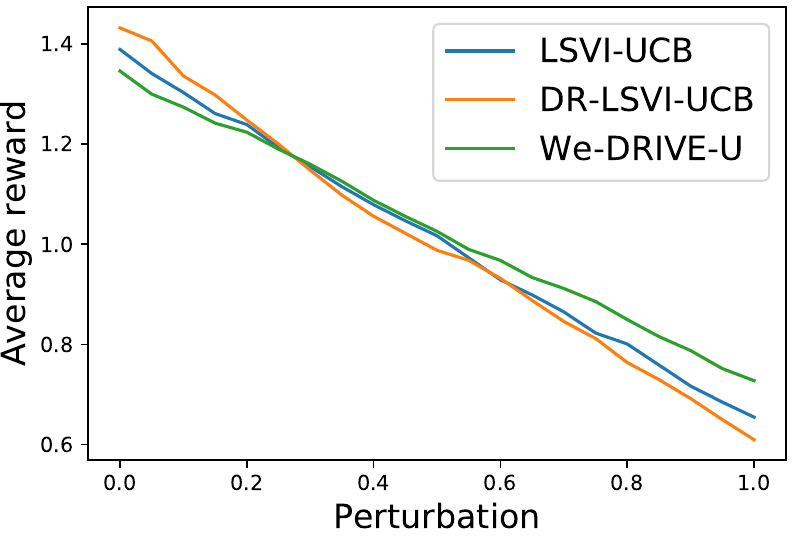}}
      \subfigure[$\Vert\xi\Vert_1 = 0.1$, $\rho_{1,4}=0.3$]{\includegraphics[scale=0.39]{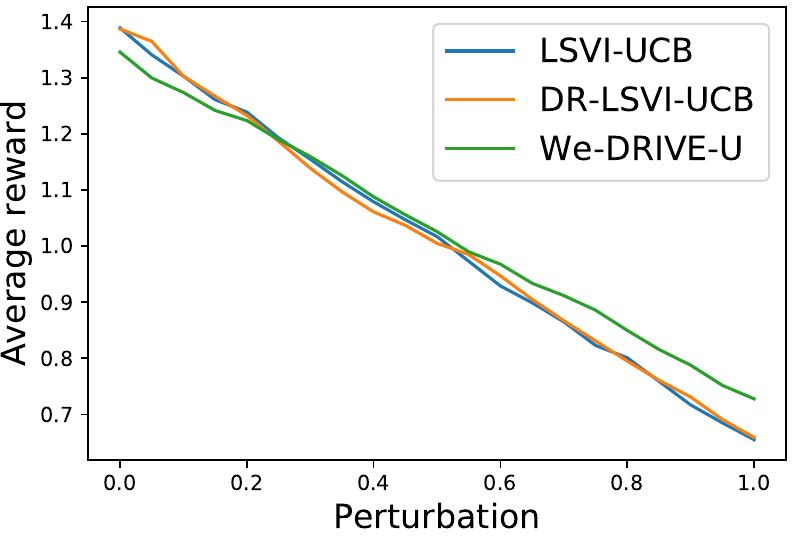}}\\
      \subfigure[$\Vert\xi\Vert_1 = 0.2$, $\rho_{1,4}=0.1$]{\includegraphics[scale=0.39]{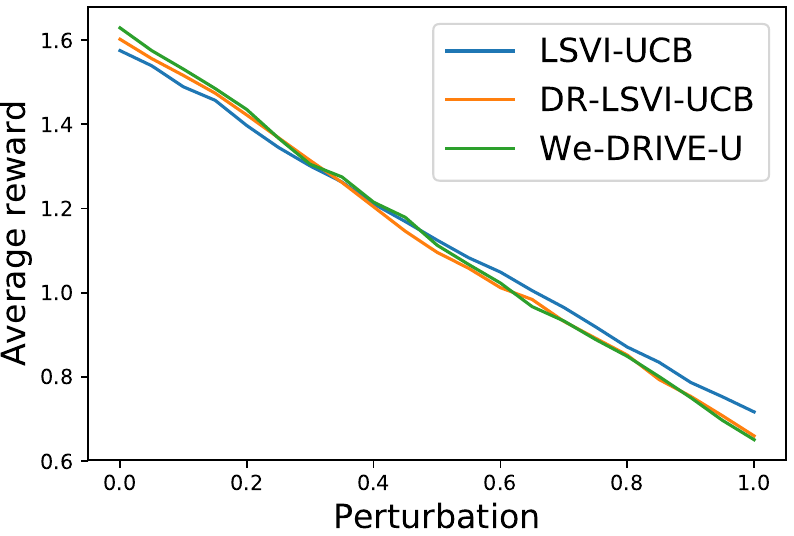}
      \label{fig:simulated_MDP_xi01_rho04}}
      \subfigure[$\Vert\xi\Vert_1 = 0.2$, $\rho_{1,4}=0.2$]{\includegraphics[scale=0.39]{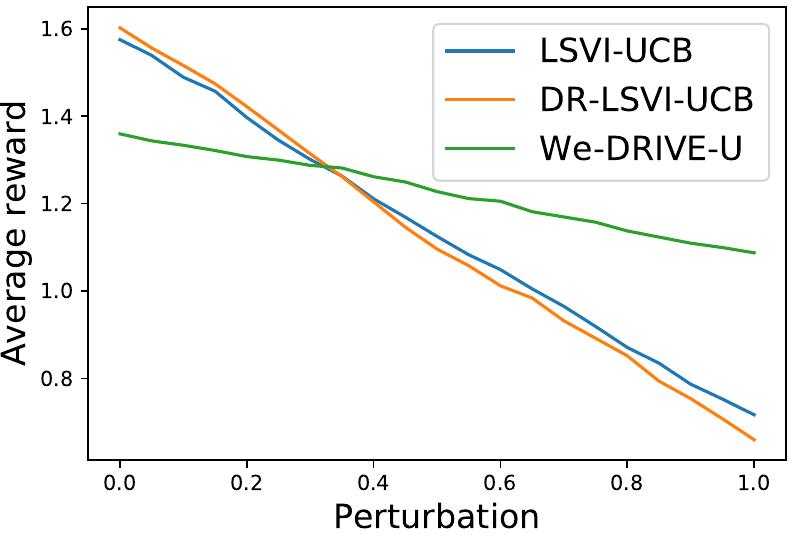}}
      \subfigure[$\Vert\xi\Vert_1 = 0.2$, $\rho_{1,4}=0.3$]{\includegraphics[scale=0.39]{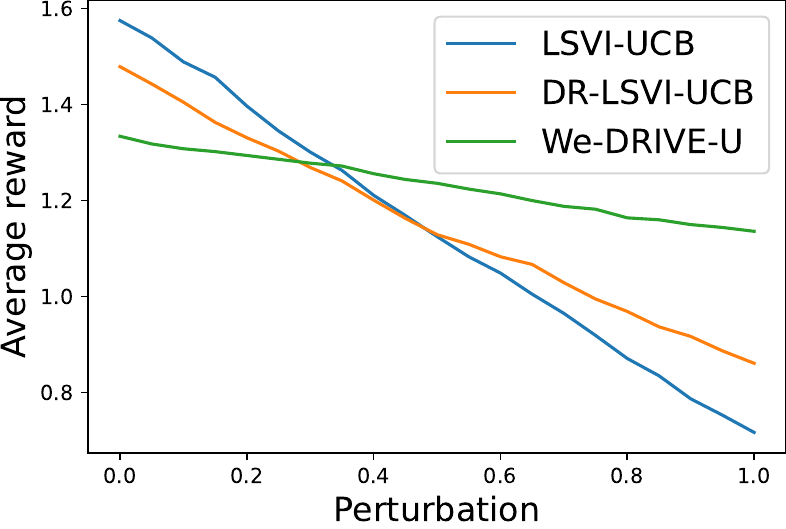}
      \label{fig:simulated_MDP_xi03_rho04}}\\
      \subfigure[$\Vert\xi\Vert_1 = 0.3$, $\rho_{1,4}=0.1$]{\includegraphics[scale=0.39]{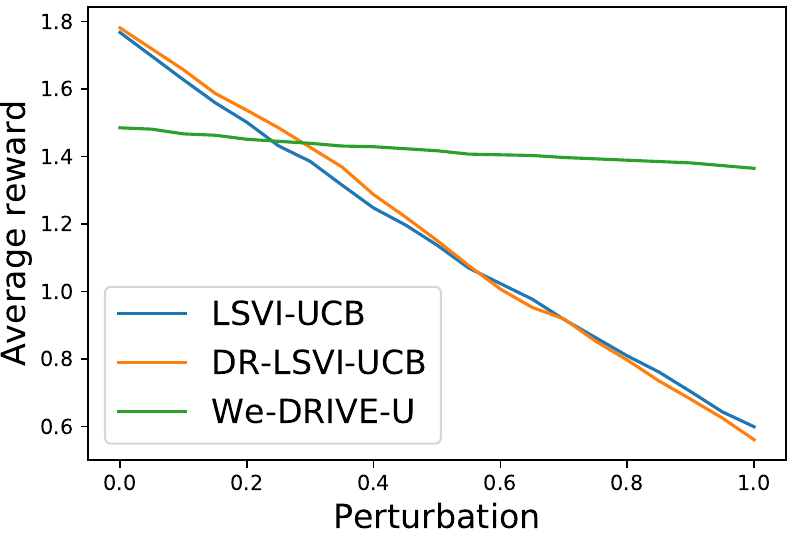}
      \label{fig:simulated_MDP_xi01_rho03}}
      \subfigure[$\Vert\xi\Vert_1 = 0.3$, $\rho_{1,4}=0.2$]{\includegraphics[scale=0.39]{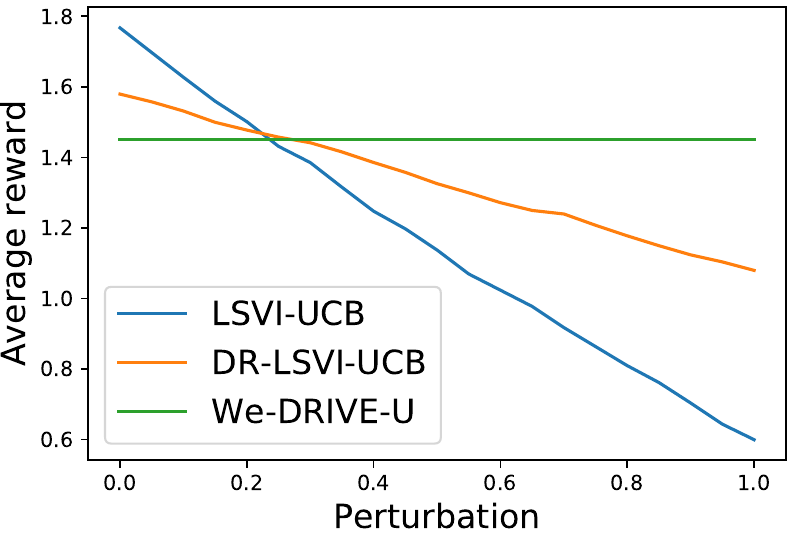}
      \label{fig:simulated_MDP_xi02_rho03}}
      \subfigure[$\Vert\xi\Vert_1 = 0.3$, $\rho_{1,4}=0.3$]{\includegraphics[scale=0.39]{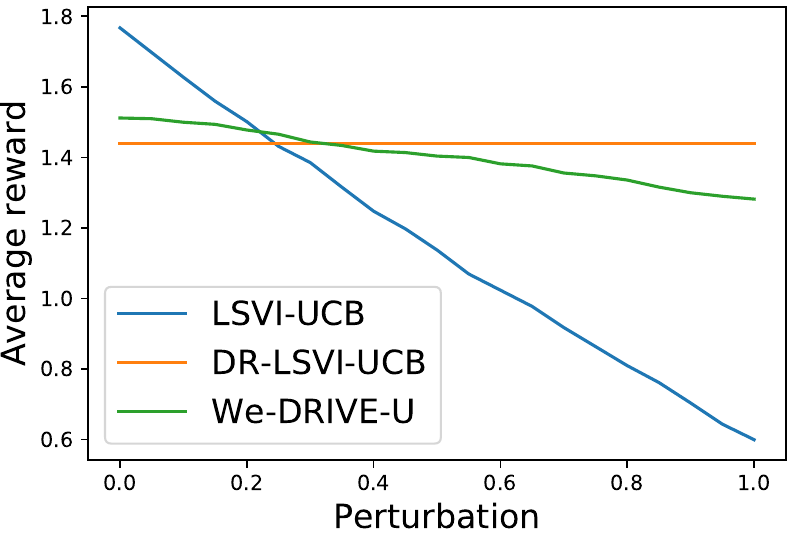}
      \label{fig:simulated_MDP_xi03_rho03}}
    \caption{Simulation results under different source domains. The $x$-axis represents the perturbation level corresponding to different target environments. $\rho_{1,4}$ is the input uncertainty level for our \algname\ algorithm. $\Vert\xi\Vert_1$ is the hyperparameter of the linear DRMDP environment.}
    \label{fig:simulation-results-app}
\end{figure*}

\begin{table}[h!]
\centering
\caption{Simulation results of the switch complexity of \algname. We present the average policy switch times of \algname\ during 200 interactions with the nominal environment, averaged over 10 replications. As a comparison, the policy switch times for LSVI-UCB and DR-LSVI-UCB are both \textbf{200} under each setting.}
\label{tab:switch_time}
\begin{tabular}{cccc}
\toprule
& {$\rho$=0.1} & {$\rho$=0.2} & {$\rho$=0.3} \\
\midrule
{$\|\xi\|_1$=0.1} & 23.8 & 24.0 & 23.8 \\ 
{$\|\xi\|_1$=0.2} & 24.2 & 24.4 & 24.0 \\ 
{$\|\xi\|_1$=0.3} & 24.3 & 23.6 & 24.8 \\ 
\bottomrule
\end{tabular}
\end{table}

\section{Conclusion} 
We studied upper and lower bounds in the setting of online linear DRMDPs. We proposed an algorithm, \algname, leveraging the variance-weighted ridge regression and low policy-switching techniques. Under assumptions on the structure of the MDP, we  showed that the average suboptimality of \algname\ is of order $\widetilde{\cO}(dH\min\{1/\rho, H\}/\sqrt{K})$. We further established an lower bound $\Omega(dH^{1/2}\min\{1/\rho,H\}/\sqrt{K})$, suggesting that \algname\ is near-optimal up to $\widetilde{\cO}(\sqrt{H})$ among the full range of the uncertainty level.

\newpage
\appendix

\section{Proof of \Cref{prop:hardness result}}
\label{sec:proof_of_proposition}
\begin{proof}
    We instantiate the hard example in Example 3.1 of \citet{lu2024distributionally} in terms of the formulation of $d$-rectangular linear DRMDP satisfying \Cref{assumption:linear_mdp}. 
    Consider two $d$-rectangular linear DRMDPs , $\cM_0$ and $\cM_1$. The state space $\cS = \{s_{\text{good}}, s_{\text{bad}}\}$, and the action space is $\cA = \{0,1\}$. We define the feature mapping as
    \begin{align*}
        \bphi^{\varrho}(s_{\text{good}},a) = \left(\begin{aligned}
            1\\0\\0\\0\\0
        \end{aligned}\right), 
        \forall a\in\cA,\ \bphi^{\varrho}(s_{\text{bad}}, 0) = \left(\begin{aligned}
           0&\\ p&(1-\varrho)\\ q&\varrho\\ (1-p)&(1-\varrho)\\ (1-q)&\varrho
        \end{aligned} \right), \ \bphi^{\varrho}(s_{\text{bad}}, 1) = \left(\begin{aligned}
           0&\\ p&\varrho\\ q&(1-\varrho)\\ (1-p)&\varrho\\ (1-q)&(1-\varrho)
        \end{aligned} \right),
    \end{align*}
    where $\varrho \in \{0,1\}$ is the index of the $d$-rectangular linear DRMDP instance. Define the factor distributions $\bmu = (\delta_{s_{\text{good}}}, \delta_{s_{\text{good}}}, \delta_{s_{\text{good}}}, \delta_{s_{\text{bad}}}, \delta_{s_{\text{bad}}})^{\top}$ and the reward parameter $\btheta = (1,0,0,0,0)^\top$. Then it is trivial to check that equipped with the $d$-rectangular TV divergence uncertainty set, this example recover the hard example in Example 3.1 of \citet{lu2024distributionally}.
\end{proof}

\section{Proof of the Upper Bound on the Suboptimality of \algname} 
\label{sec:upper_bound_proof}

In this section, we present the proofs of our main theoretical results \Cref{thm:AveSubopt_origin,thm:DRLSVIUCB}. We start with presenting the technical lemmas in \Cref{sec:technical_lemmas}, and then we derive the upper bound on the suboptimality of \algname\ in \Cref{sec:proof_of_thm_1,sec:proof_of_thm_2}.

\subsection{Technical Lemmas}
\label{sec:technical_lemmas}

\begin{definition}[Good event]
\label{lem:covering_concentration}
Under \Cref{assumption:linear_mdp,assumption:fail_state}, then for any fixed $\delta \in (0,1)$, $\alpha^\prime \in [0,H]$ and $\rho \in (0,1]$, we define $\cE_h$ be the event that for all episode $k\in [K]$, stage $h \leq h^\prime \leq H$, 
\begin{align}
\label{eq:covering_concentration}
    \bigg\Vert \sum_{\tau=1}^{k-1}\bar{\sigma}^{-2}_{\tau,h^\prime}\bphi_{h^\prime}^{\tau}\Big[\big[\hat{V}_{k,h^\prime+1}^{\rho}(s_{h^\prime+1}^{\tau})\big]_{\alpha^\prime}-\big[\PP_{h^\prime}^0 \big[\hat{V}_{k, h^\prime+1}^{\rho}\big]_{\alpha^\prime}\big](s_{h^\prime}^{\tau}, a_{h^\prime}^{\tau}) \Big]\bigg\Vert_{\bSigma_{k, h^\prime}^{-1}} \leq \gamma,
\end{align}
where $\gamma =\widetilde{\cO}\big(\sqrt{d}\big)$. 
\end{definition}

\begin{lemma}
\label{lem:coarse_event}
We define $\bar{\mathcal{E}}$ as the event that the following inequalities hold for all $(s, a)\in \mathcal{S} \times \mathcal{A}$, $k \in [K]$, $ h \in [H]$, %
\begin{align*}
    \big|\bphi(s, a)^\top \hat{\bz}_{h,1}^k - \big[\mathbb{P}_h^0 \hat{V}_{k, h+1}^\rho \big](s, a)\big| &\leq \bar{\beta} \sqrt{\bphi(s, a)^{\top} \bLambda_{k, h}^{-1} \bphi(s, a)}, \\
    \big|\bphi(s, a)^\top \check{\bz}_{h,1}^k - \big[\mathbb{P}_h^0 \check{V}_{k, h+1}^\rho \big](s, a)\big| &\leq \bar{\beta} \sqrt{\bphi(s, a)^{\top} \bLambda_{k, h}^{-1} \bphi(s, a)}, \\
    \big|\bphi(s, a)^\top \widetilde{\bz}^k_{h,2} - \big[\mathbb{P}_h^0 \big(\hat{V}_{k, h+1}^\rho\big)^2 \big](s, a)\big| &\leq \widetilde{\beta} \sqrt{\bphi(s, a)^{\top} \bLambda_{k, h}^{-1} \bphi(s, a)},
\end{align*}
where $\bar{\beta} = \widetilde{\cO}\big(H\sqrt{d\lambda} + \sqrt{d^3 H^3}\big)$ and $\widetilde{\beta} = \widetilde{\cO}\big(H^2\sqrt{d\lambda} + \sqrt{d^3 H^6}\big)$.
Then event $\bar{\mathcal{E}}$ holds with probability at least $1-\delta$.
\end{lemma}

\begin{lemma}[Variance error] 
\label{lem:variance_error}
On the event $\mathcal{E}_{h+1}$ and $\bar{\mathcal{E}}$, for all episode $k \in[K]$, the estimated variance satisfies
\begin{align*}
    \big|\big[\bar{\mathbb{V}}_h \hat{V}_{k, h+1}^\rho\big]\big(s_h^k, a_h^k\big)-[\mathbb{V}_h \hat{V}_{k, h+1}^\rho]\big(s_h^k, a_h^k\big)\big| &\leq E_{k, h}, \\
    \big|\big[\bar{\mathbb{V}}_h \hat{V}_{k, h+1}^\rho\big]\big(s_h^k, a_h^k\big)-[\mathbb{V}_h V_{h+1}^{*,\rho}]\big(s_h^k, a_h^k\big)\big| &\leq E_{k, h}+D_{k, h}.
\end{align*}
Thus we also have
\begin{align*}
    \bar{\sigma}_{k, h}^2 \geq\big[\bar{\mathbb{V}}_{h} \hat{V}_{k, h+1}^\rho\big]\big(s_h^k, a_h^k\big) + E_{k, h} + D_{k, h} \geq \big[\mathbb{V}_h V_{h+1}^{*,\rho}\big]\big(s_h^k, a_h^k\big).
\end{align*}
\end{lemma}

\begin{lemma}
\label{lem:error_bound}
For any fixed policy $\pi$, on the event $\mathcal{E}_h$ and $\bar{\mathcal{E}}$, for all $(s,a,k) \in \cS/\{s_f\} \times \cA \times [K]$, for stage $h \leq h^\prime \leq H$, we have
\begin{align*}
    \big(r_{h^\prime}(s,a) + \bphi(s,a)^\top \hat{\bnu}_{h^\prime}^{\rho,k}\big) - Q_{h^\prime}^{\pi, \rho}(s,a) &= \inf_{P_{h^\prime}(\cdot|s,a) \in \cU_{h^\prime}^{\rho}(s,a;\bmu_{h^\prime}^0)}\big[\PP_{h^\prime} \hat{V}_{k, h^\prime+1}^\rho\big](s,a)  \\
    &\qquad- \inf_{P_{h^\prime}(\cdot|s,a) \in \cU_{h^\prime}^{\rho}(s,a;\bmu_{h^\prime}^0)} \big[\PP_{h^\prime} V_{h^\prime+1}^{\pi, \rho}\big](s,a) + \Delta_{h^\prime}^k(s,a),
\end{align*}
where $\Delta_{h^\prime}^k(s,a)$ that satisfies $|\Delta_{h^\prime}^k(s,a)| \leq \hat{\Gamma}_{k,h^\prime}(s,a) = \beta \sum_{i=1}^d \phi_i(s,a) \sqrt{\mathbf{1}_i^{\top}\bSigma_{k,h^\prime}^{-1}\mathbf{1}_i}$ where $\beta = \widetilde{\cO} \Big(\sqrt{\lambda d}H +\sqrt{d}\Big)$.
\end{lemma}

\begin{lemma}[Optimism and pessimism] \label{lem:optimism_pessimism}
On the event $\mathcal{E}_h$ and $\bar{\mathcal{E}}$, for all episode $k \in[K]$ and stage $h \leq h^{\prime} \leq H$, for all $(s,a) \in \mathcal{S} \times \mathcal{A}$, we have $\hat{Q}_{k, h^\prime}^\rho(s, a) \geq Q_{h^{\prime}}^{*,\rho}(s, a) \geq \check{Q}_{k, h^\prime}^\rho(s, a)$. In addition, we have $\hat{V}_{k, h^\prime}^\rho(s) \geq V_{h^{\prime}}^{*,\rho}(s) \geq \check{V}_{k, h^\prime}^{\rho}(s)$.
\end{lemma}

\begin{lemma}
\label{lem:good_event_probability}
On the event $\bar{\mathcal{E}}$, event $\mathcal{E}=\cE_1$ holds with probability at least $1-\delta$.
\end{lemma}

\begin{lemma}
\label{lem:est_var_upper_bound}
Under the assumption \eqref{assumption:lambda_lower_bound} and events $\mathcal{E}$ and $\bar{\mathcal{E}}$, for any $h \in [H]$, set $\lambda = 1/H^2$ and $\rho \in (0,1]$. Then when $k \geq \widetilde{K}$ where $\widetilde{K}=\widetilde{\cO} \big( d^{15} H^{12} \big)$, with probability at least $1-\delta$, then we have
\begin{align*}
    \bar{\sigma}_{k,h}^2 \leq \cO\Big(\min \Big\{\frac{1}{\rho^2}, H^2\Big\}\Big).
\end{align*}
\end{lemma}

\subsection{Proof of \Cref{thm:AveSubopt_origin}}
\label{sec:proof_of_thm_1}

\begin{proof}[Proof of \Cref{thm:AveSubopt_origin}]
Conditioned on the event $\mathcal{E}$ and $\bar{\mathcal{E}}$, we first do the following decomposition
{\small\begin{align*}
    &\hat{V}_{k,h}^\rho\big(s^k_h\big) -  V_{h}^{\pi^k,\rho}\big(s^k_h\big) \notag \\ 
    &= \hat{Q}_{k,h}^\rho\big(s^k_h, a^k_h\big) -  Q_{h}^{\pi^k,\rho}\big(s^k_h, a^k_h\big) \notag \\
    &\leq r_h\big(s^k_h, a^k_h\big) + \bphi\big(s^k_h,a^k_h\big)^\top \hat{\bnu}_h^{\rho,k_{\text{last}}} + \hat{\Gamma}_{k_{\text{last}},h}\big(s^k_h,a^k_h\big) -  Q_{h}^{\pi^k,\rho}\big(s^k_h, a^k_h\big) \notag\\
    &\leq \inf_{P_h(\cdot|s,a)\in\cU_h^{\rho}(s,a;\bmu_h^0)}\big[\PP_h \hat{V}_{k,h+1}^\rho\big]\big(s^k_h,a^k_h\big) - \inf_{P_h(\cdot|s,a)\in\cU_h^{\rho}(s,a;\bmu_h^0)}\big[\PP_h V_{h+1}^{\pi^k,\rho}\big]\big(s^k_h,a^k_h\big) + 2\hat{\Gamma}_{k_{\text{last}},h}\big(s^k_h,a^k_h\big) \\
    &\leq \inf_{P_h(\cdot|s,a)\in\cU_h^{\rho}(s,a;\bmu_h^0)}\big[\PP_h \hat{V}_{k,h+1}^\rho\big]\big(s^k_h,a^k_h\big) - \inf_{P_h(\cdot|s,a)\in\cU_h^{\rho}(s,a;\bmu_h^0)}\big[\PP_h V_{h+1}^{\pi^k,\rho}\big]\big(s^k_h,a^k_h\big) + 4\hat{\Gamma}_{k,h}\big(s^k_h,a^k_h\big),
\end{align*}}%
where the first equality holds due to the selection of $\pi_{h}^{k}$, the first inequality holds due to the definition of $\hat{Q}_{k,h}^\rho$, the second inequality hold from \Cref{lem:error_bound}, the third inequality holds from \Cref{lem12_abbasi}. Note that  
{\small\begin{align*}
    &\inf_{P_h(\cdot|s_h^k,a_h^k) \in \cU_{h}^{\rho}(s_h^k,a_h^k;\bmu_h^0)}\big[\PP_h\hat{V}_{k,h+1}^{\rho}\big]\big(s_h^k,a_h^k\big)\notag  - \inf_{P_h(\cdot|s_h^k,a_h^k) \in \cU_{h}^{\rho}(s_h^k,a_h^k;\bmu_h^0)}\big[\PP_hV_{h+1}^{\pi^k, \rho}\big]\big(s_h^k,a_h^k\big)\\
    &= \bigg\la \bphi(s_h^k, a_h^k), \bigg[\max_{\alpha_i\in[0,H]}\Big\{\EE^{\mu_{h,i}^0}\big[\hat{V}_{k,h+1}^{\rho}(s) \big]_{\alpha_i} - \rho\alpha_i \Big\} \bigg]_{i\in [d]} \bigg\ra \notag\\
    &\qquad- \bigg\la \bphi(s_h^k, a_h^k), \bigg[\max_{\alpha_i\in[0,H]}\Big\{\EE^{\mu_{h,i}^0}\big[V_{h+1}^{\pi^k, \rho}(s) \big]_{\alpha_i} - \rho\alpha_i \Big\} \bigg]_{i \in [d]}\bigg\ra\\
    &\leq \bigg\la \bphi(s_h^k, a_h^k), \bigg[\max_{\alpha_i\in[0,H]}\Big\{\EE^{\mu_{h,i}^0}\big[\hat{V}_{k,h+1}^{\rho}(s) \big]_{\alpha_i} - \EE^{\mu_{h,i}^0}\big[V_{h+1}^{\pi^k, \rho}(s) \big]_{\alpha_i} \Big\} \bigg]_{i\in [d]} \bigg\ra \\
    &\leq \big\la \bphi(s_h^k, a_h^k), \EE^{\bmu_{h}^0}\big[\hat{V}_{k,h+1}^{\rho}(s) - V_{h+1}^{\pi^k, \rho}(s)\big]\big\ra \\ 
    &= \PP_h^0\big[\hat{V}_{k,h+1}^\rho - V_{h+1}^{\pi^k, \rho}\big]\big]\big(s_h^k,a_h^k\big) \\
    &= \big[\PP^0_h\big[\hat{V}_{k,h+1}^\rho - V_{h+1}^{\pi^k, \rho}\big]\big]\big(s_h^k,a_h^k\big) - \big[\hat{V}_{k,h+1}^\rho(s_{h+1}^k) - V_{h+1}^{\pi^k, \rho}(s_{h+1}^k)\big] + \big[\hat{V}_{k,h+1}^\rho(s_{h+1}^k) - V_{h+1}^{\pi^k, \rho}(s_{h+1}^k)\big],
 \end{align*}}%
where the second inequality holds from \Cref{lem:optimism_pessimism}. Then we have
{\small\begin{align}
\label{equ:regret_decomposition}
    &\hat{V}_{k,h}^\rho\big(s^k_h\big) -  V_{h}^{\pi^k,\rho}\big(s^k_h\big) \notag \\
    &\leq \big[\hat{V}_{k,h+1}^\rho(s_{h+1}^k) - V_{h+1}^{\pi^k, \rho}(s_{h+1}^k)\big] + \big[\PP^0_h\big[\hat{V}_{k,h+1}^\rho - V_{h+1}^{\pi^k, \rho}\big]\big]\big(s_h^k,a_h^k\big) - \big[\hat{V}_{k,h+1}^\rho(s_{h+1}^k) - V_{h+1}^{\pi^k, \rho}(s_{h+1}^k)\big] \notag\\
    &\qquad+ 4\hat{\Gamma}_{k,h}\big(s^k_h,a^k_h\big),
\end{align}}%
Then by applying \eqref{equ:regret_decomposition} iteratively and applying Azuma-Hoeffding inequality, with probability at least $1-\delta/3$, we have
\begin{align*}
    K \times \text{AveSubopt}(K) &= \sum_{k=1}^K \Big(V_1^{*,\rho}\big(s_1^k\big) - V_1^{\pi^k,\rho}\big(s_1^k\big)\Big) \\   
    &\leq \sum_{k=1}^K \Big(\hat{V}_{k,1}^\rho\big(s^k_1\big) -  V_1^{\pi^k,\rho}\big(s^k_1\big) \Big)\\  
    &\leq \sum_{k=1}^K \sum_{h=1}^H \Big(\big[\PP_h\big[\hat{V}_{k,h+1}^\rho - V_{h+1}^{\pi^k, \rho}\big]\big]\big(s_h^k,a_h^k\big) - \big[\hat{V}_{k,h+1}^\rho(s_{h+1}^k) - V_{h+1}^{\pi^k, \rho}(s_{h+1}^k)\big]\Big) \\
    &\qquad+ \sum_{k=1}^K \sum_{h=1}^H 4\hat{\Gamma}_{k,h}\big(s^k_h,a^k_h\big) \\
    &\leq 2 \sqrt{2 H^3 K \log (6 / \delta)} + 4\beta \sum_{k=1}^K\sum_{h=1}^H\sum_{i=1}^d \phi_{i}\big(s^k_h,a^k_h\big) \sqrt{\mathbf{1}_i^\top \bSigma_{k,h}^{-1}\mathbf{1}_i},
\end{align*}
where the first inequality holds from \Cref{lem:optimism_pessimism}, the second inequality holds from \eqref{equ:regret_decomposition}, the third inequality holds from Azuma-Hoeffding inequality and the definition of $\hat{\Gamma}_{k,h}\big(s^k_h,a^k_h\big)$. Finally, by taking probability union bound over $\mathcal{E}$ and $\bar{\mathcal{E}}$, with probability at least $1-\delta$, we can get the result of \Cref{thm:DRLSVIUCB},
\begin{align*}
   \text{AveSubopt}(K) \leq 2 \sqrt{2 H^3 \log (6 / \delta)/K} +{4\beta}/{K} \sum_{k=1}^K\sum_{h=1}^H\sum_{i=1}^d\phi_{h,i}^k\sqrt{\mathbf{1}_i^\top \bSigma_{k,h}^{-1}\mathbf{1}_i}.
\end{align*}
This completes the proof.
\end{proof}

\subsection{Proof of \Cref{thm:DRLSVIUCB}}
\label{sec:proof_of_thm_2}

\begin{proof}[Proof of \Cref{thm:DRLSVIUCB}]
Conditioned on the event $\mathcal{E}$ and $\bar{\mathcal{E}}$, we first do the decomposition as follows
\begin{align*}
    K \times \text{AveSubopt}(K) &= \sum_{k=1}^K \big(V_1^{*,\rho}\big(s_1^k\big) - V_1^{\pi^k,\rho}\big(s_1^k\big)\big) \\
    &= \sum_{k=1}^{\widetilde{K}} \big(V_1^{*,\rho}\big(s_1^k\big) - V_1^{\pi^k,\rho}\big(s_1^k\big)\big) + \sum_{k=\widetilde{K}+1}^{K} \big(V_1^{*,\rho}\big(s_1^k\big) - V_1^{\pi^k,\rho}\big(s_1^k\big)\big) \\ 
    &\leq H\widetilde{K} + \sum_{k=\widetilde{K}+1}^{K} \big(V_1^{*,\rho}\big(s_1^k\big) - V_1^{\pi^k,\rho}\big(s_1^k\big)\big). 
\end{align*}
Recall from \eqref{equ:regret_decomposition} in the proof of \Cref{thm:AveSubopt_origin}, we have
{\small\begin{align*}
    &\hat{V}_{k,h}^\rho\big(s^k_h\big) -  V_{h}^{\pi^k,\rho}\big(s^k_h\big)  \\
    &\leq \big[\hat{V}_{k,h+1}^\rho(s_{h+1}^k) - V_{h+1}^{\pi^k, \rho}(s_{h+1}^k)\big] + \big[\PP_h\big[\hat{V}_{k,h+1}^\rho - V_{h+1}^{\pi^k, \rho}\big]\big]\big(s_h^k,a_h^k\big) - \big[\hat{V}_{k,h+1}^\rho(s_{h+1}^k) - V_{h+1}^{\pi^k, \rho}(s_{h+1}^k)\big] \\
    &\qquad+ 4\hat{\Gamma}_{k,h}\big(s^k_h,a^k_h\big).
\end{align*}}%
Then by applying \eqref{equ:regret_decomposition} iteratively and applying Azuma-Hoeffding inequality, with probability at least $1-\delta/4$, we have
\begin{align*}
    K \times \text{AveSubopt}(K) & \leq H\widetilde{K} + \sum_{k=\widetilde{K}+1}^{K} \big(V_1^{*,\rho}\big(s_1^k\big) - V_1^{\pi^k,\rho}\big(s_1^k\big)\big) \\  
    &\leq H\widetilde{K} + \sum_{k=\widetilde{K}+1}^K \big(\hat{V}_{k,1}^\rho\big(s^k_1\big) -  V_1^{\pi^k,\rho}\big(s^k_1\big)\big) \\  
    &\leq H\widetilde{K} + \sum_{k=\widetilde{K}+1}^K \sum_{h=1}^H \Big(\big[\PP^0_h\big[\hat{V}_{k,h+1}^\rho - V_{h+1}^{\pi^k, \rho}\big]\big]\big(s_h^k,a_h^k\big) \\
    &\qquad - \big[\hat{V}_{k,h+1}^\rho(s_{h+1}^k) - V_{h+1}^{\pi^k, \rho}(s_{h+1}^k)\big]\Big)+ \sum_{k=\widetilde{K}+1}^K \sum_{h=1}^H 4\hat{\Gamma}_{k,h}\big(s^k_h,a^k_h\big) \\
    &\leq  H\widetilde{K} + 2 \sqrt{2 H^3 K \log (8 / \delta)} + 4\beta \sum_{k=\widetilde{K}+1}^K\sum_{h=1}^H\sum_{i=1}^d \phi_{i}\big(s^k_h,a^k_h\big) \sqrt{\mathbf{1}_i^\top \bSigma_{k,h}^{-1}\mathbf{1}_i},
\end{align*}
where the second inequality holds from \Cref{lem:optimism_pessimism}, the third inequality holds from \eqref{equ:regret_decomposition} and the last inequality holds from Azuma-Hoeffding inequality and the definition of $\hat{\Gamma}_{k,h}\big(s^k_h,a^k_h\big)$. Based on \eqref{assumption:lambda_lower_bound} and \Cref{lem:est_var_upper_bound}, with probability at least $1-\delta/4$, we can further have
\begin{align*}
    &4\beta \sum_{k=\widetilde{K}+1}^K\sum_{h=1}^H\sum_{i=1}^d \phi_{i}\big(s^k_h,a^k_h\big) \sqrt{\mathbf{1}_i^\top \bSigma_{k,h}^{-1}\mathbf{1}_i} \\
    &\leq 4c_1\beta \min \Big\{\frac{1}{\rho}, H\Big\} \cdot \sum_{k=\widetilde{K}+1}^K\sum_{h=1}^H\sum_{i=1}^d \phi_{i}\big(s^k_h,a^k_h\big) \sqrt{\mathbf{1}_i^\top \bLambda_{k,h}^{-1}\mathbf{1}_i} \\
    &\leq 4c_1\beta \min \Big\{\frac{1}{\rho}, H\Big\} \cdot \sum_{k=\widetilde{K}+1}^K\sum_{h=1}^H\sum_{i=1}^d \phi_{i}\big(s^k_h,a^k_h\big) \sqrt{\lambda_{\text{max}}\big(\bLambda_{k, h}^{-1}\big)} \\
    &\leq 4c_1\beta \min \Big\{\frac{1}{\rho}, H\Big\} \cdot \sum_{k=\widetilde{K}+1}^K\sum_{h=1}^H\sqrt{\frac{1}{\lambda_{\text{min}}\big(\bLambda_{k, h}\big)}} \\
    &\leq 4c_1\beta \min \Big\{\frac{1}{\rho}, H\Big\} \cdot \sum_{k=\widetilde{K}+1}^K\sum_{h=1}^H \sqrt{\frac{2d}{k\cdot c}} \\
    &\leq 4c_1\sqrt{2d}\beta \frac{H}{\sqrt{c}} \cdot \min \Big\{\frac{1}{\rho}, H\Big\} \cdot \int_{\widetilde{K}+1}^K \frac{1}{\sqrt{k-1}} dk \\
    &\leq 4c_1\sqrt{2d}\beta \frac{H}{\sqrt{c}} \cdot \min \Big\{\frac{1}{\rho}, H\Big\} \cdot 2\sqrt{K} \\
    &\leq \widetilde{\cO}\Big(dH\sqrt{K}  \cdot \min \Big\{\frac{1}{\rho}, H\Big\} \Big),
\end{align*}
where $c_1>0$ is an absolute constant.
The first inequality holds from \Cref{lem:est_var_upper_bound}, the third inequality holds because $\sum_{i=1}^d\phi_i(s,a)=1$ and the fourth inequality holds due to \eqref{eq:lambda_min_lower_bound} with $\widetilde{K} > 512/\eta^2 \log(dKH/\delta)$. Therefore, we can further bound the regret that
\begin{align*}
    K \times \text{AveSubopt}(K) & \leq H\widetilde{K} + 2 \sqrt{2 H^3 K \log (8 / \delta)} + 4\beta \sum_{k=\widetilde{K}+1}^K\sum_{h=1}^H\sum_{i=1}^d \phi_{i}\big(s^k_h,a^k_h\big) \sqrt{\mathbf{1}_i^\top \bSigma_{k,h}^{-1}\mathbf{1}_i} \\
   &\leq \widetilde{\cO}\Big( {dH\sqrt{K}}\cdot \min \Big\{\frac{1}{\rho}, H\Big\} + H^\frac{3}{2} \sqrt{K} + d^{15} H^{13} \Big).
\end{align*}
Finally, by taking probability union bound over $\mathcal{E}$ and $\bar{\mathcal{E}}$, with probability at least $1-\delta$, we can bound the average suboptimality of \algname\ as follows
\begin{align}
\label{eq:AveSubopt-final}
    \text{AveSubopt}(K) \leq \widetilde{\cO}\Bigg( \frac{d H \cdot \min \big\{\frac{1}{\rho}, H\big\} + H^\frac{3}{2} }{\sqrt{ K}} + \frac{d^{15} H^{13}}{K}\Bigg).
\end{align}
We complete the proof by substituting $\eta=O(1/d)$ into \eqref{eq:AveSubopt-final}. 
\end{proof}

\section{Proof of the Technical Lemmas}

\subsection{Proof of \Cref{lem:coarse_event}}

Before the proof of \Cref{lem:coarse_event}, we first present a lemma that defines the optimistic value function class and gives a upper bound for its covering number.

\begin{lemma}[Function class covering number]
\label{lem:function_class_covering_number} In \Cref{alg:DR-LSVI-UCB+}, for each episode $k \in [K]$ and $h \in [H]$, the optimistic value function $\hat{V}_{k,h}^\rho$ belongs to the following function class
{\small
\begin{align*}
    \cV_h &= \bigg\{V \Big| V(\cdot) = \max_a \max_{1 \leq j \leq \ell} \min \Big\{ r_h(\cdot,a) + \bphi(\cdot,a)^\top \wb_j + \beta \sum_{i=1}^d\phi_i(\cdot,a)\sqrt{\mathbf{1}_i^\top \bGamma_j\mathbf{1}_i}, H \Big\},\\
   &\qquad \quad \|\wb_j\| \leq L, \|\bGamma_j\|_F \leq \lambda^{-1} \sqrt{d} \bigg\},
\end{align*}}%
where $\ell \leq dH\log(1+K/\lambda)$ is the number of value function updates from \Cref{lem:number_of_value_function_updates} and $L = 2H\sqrt{dK/\lambda}$ from \Cref{lem:dual_parameter_bound}. Define $\cN_{\epsilon}$ be the $\epsilon$-covering number of $\cV_h$ with respect to the distance $\text{dist}(V_1, V_2)=\sup_s|V_1(s)-V_2(s)|$. Then the covering entropy can be bounded by
\begin{align*}
    \log \cN_\epsilon \leq d \ell \log (1 + 4L/\epsilon) + d^2 \ell \log \big(1+8 \sqrt{d} \beta^2 / \lambda \epsilon^2\big).
\end{align*}
\end{lemma}

\begin{proof}[Proof of \Cref{lem:function_class_covering_number}]
For any two function $V_1, V_2 \in \cV_h$, we can write $V_1, V_2$ as follows
\begin{align*}
    V_1(\cdot) &= \max_a \max_{1 \leq j \leq \ell} \min \Big\{ r_h(\cdot,a) + \bphi(\cdot,a)^\top \wb_{1,j} + \beta \sum_{i=1}^d\phi_i(\cdot,a)\sqrt{\mathbf{1}_i^\top \bGamma_{1,j}\mathbf{1}_i}, H \Big\}, \\
    V_2(\cdot) &= \max_a \max_{1 \leq j \leq \ell} \min \Big\{ r_h(\cdot,a) + \bphi(\cdot,a)^\top \wb_{2,j} + \beta \sum_{i=1}^d\phi_i(\cdot,a)\sqrt{\mathbf{1}_i^\top \bGamma_{2,j}\mathbf{1}_i}, H \Big\},
\end{align*}
where $\|\wb_{1,j}\|, \|\wb_{2,j}\| \leq L$, $\bGamma_{1,j}, \bGamma_{2,j} \preccurlyeq \lambda^{-1} \Ib$ and $\|\bGamma_{1,j}\|_F, \|\bGamma_{2,j}\|_F \leq \lambda^{-1} \sqrt{d}$. Then we have
\begin{align}
\label{equ:optimistic_function_distance}
    \text{dist}(V_1, V_2) &= \sup_s|V_1(s)-V_2(s)| \notag \\
    &\leq \sup_{1\leq j\leq \ell, s\in \cS, a\in \cA} \bigg| \bphi(s,a)^\top \wb_{1,j} + \beta \sum_{i=1}^d\phi_i(s,a)\sqrt{\mathbf{1}_i^\top \bGamma_{1,j}\mathbf{1}_i} \notag \\
    &\qquad - \bphi(s,a)^\top \wb_{2,j} - \beta \sum_{i=1}^d\phi_i(s,a)\sqrt{\mathbf{1}_i^\top \bGamma_{2,j}\mathbf{1}_i}\bigg| \notag \\
    &\leq \beta \sup_{1\leq j\leq \ell, s\in \cS, a\in \cA} \bigg|\sum_{i=1}^d\phi_i(s,a)\Big(\sqrt{\mathbf{1}_i^\top \bGamma_{1,j}\mathbf{1}_i} - \sqrt{\mathbf{1}_i^\top \bGamma_{2,j}\mathbf{1}_i}\Big)\bigg| \notag \\
    &\qquad+ \sup_{1\leq j\leq \ell, s\in \cS, a\in \cA} \big|\bphi(s,a)^\top (\wb_{1,j} - \wb_{2,j})\big| \notag \\
    &\leq \beta \sup_{1\leq j\leq \ell, s\in \cS, a\in \cA} \bigg|\sum_{i=1}^d \sqrt{\phi_i(s,a)\mathbf{1}_i^\top \big(\bGamma_{1,j}-\bGamma_{2,j}\big)\phi_i(s,a)\mathbf{1}_i} \bigg| \notag \\
    &\qquad+ \sup_{1\leq j\leq \ell, s\in \cS, a\in \cA} \big|\bphi(s,a)^\top (\wb_{1,j} - \wb_{2,j})\big| \notag \\
    &\leq \beta \sup_{1\leq j\leq \ell}\sqrt{\big\|\bGamma_{1,j}-\bGamma_{2,j}\big\|_F} + \sup_{1\leq j\leq \ell} \big\|\wb_{1,j} - \wb_{2,j}\big\|_2
\end{align}
where the third inequality holds because $|\sqrt{x-y}|\geq |\sqrt{x} - \sqrt{y}|$, the fourth inequality holds because Cauchy-Schwarz inequality, $\|\bphi(s,a)\|_2 \leq 1$ and $\sum_{i=1}^d \phi_i =1$. Moreover, $\|\cdot\|_F$ is the Frobenius norm.

Now, we denote $\mathcal{C}_{\mathbf{w}}$ as a $\epsilon / 2$-cover of the set $\big\{\mathbf{w} \in \mathbb{R}^d | \| \mathbf{w} \|_2 \leq L\big\}$ and $\mathcal{C}_{\bGamma}$ as a $\epsilon^2 /4 \beta^2$-cover of the set $\{\bGamma \in \mathbb{R}^{d \times d} \mid\|\bGamma\|_F \leq \lambda^{-1} \sqrt{d}\}$ with respect to the Frobenius norm. Then according to \Cref{lem:vershynin2018high}, we have 
\begin{align*}
    |\mathcal{C}_{\mathbf{w}}| \leq(1+4 L / \epsilon)^d,|\mathcal{C}_{\bGamma}| \leq \big(1+8 \sqrt{d} \beta^2 / \lambda \epsilon^2\big)^{d^2}.
\end{align*}
Then for any function $V_1 \in \mathcal{V}_h$ with parameters $\wb_{1, j}, \bGamma_{1, j}, 1 \leq j \leq \ell$, we can find parameters $\wb_{2, j} \in \mathcal{C}_{\wb}, \bGamma_{2, j}\in \mathcal{C}_{\bGamma}, 1 \leq j \leq \ell$, such that $\|\wb_{2, j}-\wb_{1, j}\|_2 \leq \epsilon / 2,\|\bGamma_{2, j}-\bGamma_{1, j}\|_F \leq \epsilon^2 /4 \beta^2$. Thus we have
\begin{align*}
    \text{dist}(V_1, V_2) \leq \beta \sup_{1 \leq j \leq \ell} \sqrt{\|\bGamma_{1, j}-\bGamma_{2, j}\|_F} + \sup_{1 \leq j \leq \ell}\|\wb_{1, j}-\wb_{2, j}\|_2 \leq \epsilon,
\end{align*}
where the inequality holds from \eqref{equ:optimistic_function_distance}. Therefore, the $\epsilon$-covering number of optimistic function class $\mathcal{V}_h$ is bounded by $\mathcal{N}_\epsilon \leq|\mathcal{C}_{\wb}|^\ell \cdot|\mathcal{C}_{\bGamma}|^\ell$, thus we have
\begin{align*}
    \log \mathcal{N}_\epsilon \leq d \ell \log (1+4 L / \epsilon) + d^2 \ell \log \big(1 + 8 \sqrt{d} \beta^2 / \lambda\epsilon^2\big),
\end{align*}
which completes the proof.
\end{proof}

Now we are ready to prove \Cref{lem:coarse_event}.

\begin{proof}[Proof of \Cref{lem:coarse_event}]
For any stage $h \in[H]$ and the optimistic value function $\hat{V}_{k, h+1}^\rho$, according to \Cref{lem:linear_form_and_bound}, there exists a vector $\bz_h^k$ such that $\mathbb{P}_h^0 \hat{V}_{k, h+1}^\rho(s, a)$ can be represented by $\bphi(s,a)^\top \bz_h^k$ and $\|\bz_h^k\|_2 \leq H \sqrt{d}$. Therefore, the parameter estimation error can be decomposed as
\begin{align*}
    &\big\|\hat{\bz}_{h,1}^k - \bz_h^k\big\|_{\bLambda_{k, h}} \\
    &\leq \bigg\|\bLambda_{k, h}^{-1}\sum_{\tau=1}^{k-1}  \bphi\big(s_h^\tau, a_h^\tau\big) \hat{V}_{k, h+1}^\rho\big(s_{h+1}^\tau\big) - \bLambda_{k, h}^{-1}\bigg(\lambda\Ib + \sum_{\tau=1}^{k-1}  \bphi\big(s_h^\tau, a_h^\tau\big)\bphi\big(s_h^\tau, a_h^\tau\big)^\top \bigg)\bz_h^k\bigg\|_{\bLambda_{k, h}} \\
    &\leq \bigg\|\bLambda_{k, h}^{-1}\sum_{\tau=1}^{k-1}  \bphi\big(s_h^\tau, a_h^\tau\big)\big( \hat{V}_{k, h+1}^\rho\big(s_{h+1}^\tau\big)-\mathbb{P}_h^0 \hat{V}_{k, h+1}^\rho(s^\tau_h, a^\tau_h)\big) - \lambda \bLambda_{k, h}^{-1}\bz_h^k\bigg\|_{\bLambda_{k, h}} \\
    &\leq \underbrace{\big\|\lambda \bLambda_{k, h}^{-1}\bz_h^k\big\|_{\bLambda_{k, h}}}_{I_1} + \underbrace{\bigg\|\bLambda_{k, h}^{-1}\sum_{\tau=1}^{k-1}  \bphi\big(s_h^\tau, a_h^\tau\big)\big( \hat{V}_{k, h+1}^\rho\big(s_{h+1}^\tau\big)-\mathbb{P}_h^0 \hat{V}_{k, h+1}^\rho(s^\tau_h, a^\tau_h)\big)\bigg\|_{\bLambda_{k, h}}}_{I_2}.
\end{align*}
\textbf{Bound term $I_1$:}
\begin{align*}
    I_1 = \big\|\lambda \bLambda_{k, h}^{-1}\bz_h^k\big\|_{\bLambda_{k, h}} = \lambda\big\|\bz_h^k\big\|_{\bLambda_{k, h}^{-1}} \leq \sqrt{\lambda}\big\|\bz_h^k\big\|_2 \leq H\sqrt{d\lambda},  
\end{align*}
where we have $\bLambda_{k, h} \succcurlyeq \lambda \Ib$ and $\|\bz_h^k\|_2 \leq H\sqrt{d}$. \\
\textbf{Bound term $I_2$:} we apply \Cref{lem:jin_D.4} with the optimistic value function class $\cV_h$ and $\epsilon = H\sqrt{\lambda}/K$, then for any fixed $h \in [H]$, with probability at least $1 - \delta/3H$, for all episode $k \in [K]$, we have
\begin{align*}
    I_2 &= \bigg\|\sum_{\tau=1}^{k-1}  \bphi\big(s_h^\tau, a_h^\tau\big)\big( \hat{V}_{k, h+1}^\rho\big(s_{h+1}^\tau\big)-\mathbb{P}_h^0 \hat{V}_{k, h+1}^\rho(s^\tau_h, a^\tau_h)\big)\bigg\|_{\bLambda_{k, h}^{-1}} \\
    &\leq \sqrt{4 H^2\bigg[\frac{d}{2} \log \bigg(\frac{k+\lambda}{\lambda}\bigg)+\log \frac{\mathcal{N}_{\varepsilon}}{\delta}\bigg]+\frac{8 k^2 \varepsilon^2}{\lambda}} \\
    &\leq \widetilde{\cO}\big(\sqrt{d^3 H^3}\big),
\end{align*}
where the first inequality holds because of \Cref{lem:jin_D.4}, the second inequality holds from 
\Cref{lem:function_class_covering_number}. Thus we have
\begin{align*}
    \big\|\hat{\bz}_{h,1}^k - \bz_h^k\big\|_{\bLambda_{k, h}} \leq I_1 + I_2 = \widetilde{\cO}\big(H\sqrt{d\lambda} + \sqrt{d^3 H^3} \big) = \bar{\beta}.
\end{align*}
Therefore, the estimation error can be bounded by
\begin{align*}
    \big|\bphi(s, a)^\top \hat{\bz}_{h,1}^k - \big[\mathbb{P}_h^0 \hat{V}_{k, h+1}^\rho \big](s, a)\big| &= \big|\bphi(s, a)^\top\hat{\bz}_{h,1}^k - \bphi(s, a)^\top\bz_h^k\big|\\ 
    &\leq \big\|\hat{\bz}_{h,1}^k - \bz_h^k\big\|_{\bLambda_{k, h}} \cdot \|\bphi(s,a)\|_{\bLambda_{k, h}^{-1}} \\
    &\leq \bar{\beta} \sqrt{\bphi(s, a)^{\top} \bLambda_{k, h}^{-1} \bphi(s, a)},
\end{align*}
where the first inequality holds from Cauchy-Schwarz inequality. Similarly, for the pessimistic function class $\check{\cV}_h$ (or squared value function class $\cV^2_h$), we have the similar result as follows
\begin{align*}
    & \big|\bphi(s, a)^\top \check{\bz}_{h,1}^k - \big[\mathbb{P}_h^0 \check{V}_{k, h+1}^\rho \big](s, a)\big| \leq \bar{\beta} \sqrt{\bphi(s, a)^{\top} \bLambda_{k, h}^{-1} \bphi(s, a)}, \\
    & \big|\bphi(s, a)^\top \widetilde{\bz}^k_{h,2} - \big[\mathbb{P}_h^0 \big(\hat{V}_{k, h+1}^\rho\big)^2 \big](s, a)\big| \leq \widetilde{\beta} \sqrt{\bphi(s, a)^{\top} \bLambda_{k, h}^{-1} \bphi(s, a)},
\end{align*}
where $\bar{\beta} = \widetilde{\cO}\big(H\sqrt{d\lambda} + \sqrt{d^3 H^3}\big)$ and $\widetilde{\beta} = \widetilde{\cO}\big(H^2\sqrt{d\lambda} + \sqrt{d^3 H^6}\big)$. By taking union bound over $h \in [H]$ and three function classes, we have that the event $\bar{\mathcal{E}}$ holds with probability at least $1-\delta$. This completes the proof.
\end{proof}

\subsection{Proof of \Cref{lem:variance_error}}

\begin{proof}[Proof of \Cref{lem:variance_error}]
First, recall from \eqref{eq:estimated_variance_optimal_v_function}, we have
\begin{align*}
    \big[\mathbb{V}_h \hat{V}_{k,h+1}^\rho\big](s, a) \approx \big[\bar{\mathbb{V}}_h \hat{V}_{k,h+1}^\rho\big](s, a) = \big[\bphi(s,a)^\top \widetilde{\bz}^k_{h,2} \big]_{[0,H^2]} -\big[\bphi(s,a)^\top \hat{\bz}^k_{h,1}\big]_{[0,H]}^2,
\end{align*}
where $\hat{\bz}^k_{h,1}$ and $\widetilde{\bz}^k_{h,2}$ is the solution of the following ridge regression problems
\begin{align*}
    \widetilde{\bz}^k_{h,2} &= \argmin_{\bz \in \RR^d}\sum_{\tau=1}^{k-1} \Big(\bz^\top \bphi\big(s_h^\tau, a_h^\tau\big) - \big(\hat{V}_{k, h+1}^\rho\big(s_{h+1}^\tau\big)\big)^2\Big)^2 + \lambda\Vert \bz \Vert_2^2, \\
    \hat{\bz}^k_{h,1} &= \argmin_{\bz \in \RR^d}\sum_{\tau=1}^{k-1} \Big(\bz^\top \bphi\big(s_h^\tau, a_h^\tau\big) - \hat{V}_{k, h+1}^\rho\big(s_{h+1}^\tau\big)\Big)^2 + \lambda\Vert \bz \Vert_2^2.
\end{align*}
Then we have
{\small
\begin{align*}
    &\big|\big[\bar{\mathbb{V}}_h \hat{V}_{k, h+1}^\rho\big]\big(s_h^k, a_h^k\big)-[\mathbb{V}_h \hat{V}_{k, h+1}^\rho]\big(s_h^k, a_h^k\big)\big| \\
    &\leq \Big| \big[\bphi\big(s_h^k, a_h^k\big)^\top \widetilde{\bz}^k_{h,2} \big]_{[0,H^2]} -\big[\bphi\big(s_h^k, a_h^k\big)^\top \hat{\bz}^k_{h,1}\big]_{[0,H]}^2 - \big[\PP_h^0 \big(\hat{V}_{k, h+1}^\rho\big)^2 \big]\big(s_h^k, a_h^k\big) + \big(\big[\PP_h^0 \hat{V}_{k, h+1}^\rho \big]\big(s_h^k, a_h^k\big)\big)^2 \Big| \\
    &\leq \Big| \big[\bphi\big(s_h^k, a_h^k\big)^\top \widetilde{\bz}^k_{h,2} \big]_{[0,H^2]} - \big[\PP_h^0 \big(\hat{V}_{k, h+1}^\rho\big)^2 \big]\big(s_h^k, a_h^k\big)\Big| + \Big|\big[\bphi\big(s_h^k, a_h^k\big)^\top \hat{\bz}^k_{h,1}\big]_{[0,H]}^2 - \big(\big[\PP_h^0 \hat{V}_{k, h+1}^\rho \big]\big(s_h^k, a_h^k\big)\big)^2 \Big| \\
    &\leq \Big| \big[\bphi\big(s_h^k, a_h^k\big)^\top \widetilde{\bz}^k_{h,2} \big]_{[0,H^2]} - \big[\PP_h^0 \big(\hat{V}_{k, h+1}^\rho\big)^2 \big]\big(s_h^k, a_h^k\big)\Big| + 2H \Big|\big[\bphi\big(s_h^k, a_h^k\big)^\top \hat{\bz}^k_{h,1}\big]_{[0,H]} - \big[\PP_h^0 \hat{V}_{k, h+1}^\rho \big]\big(s_h^k, a_h^k\big) \Big| \\
    &\leq \min \Big\{\widetilde{\beta} \big\|\bphi\big(s_h^k, a_h^k\big)\big\|_{\bLambda_{k, h}^{-1}}, H^2\Big\} + \min \Big\{2 H \bar{\beta}\big\|\bphi\big(s_h^k, a_h^k\big)\big\|_{\bLambda_{k, h}^{-1}}, H^2\Big\} \\
    &= E_{k,h},
\end{align*}}%
where the last inequality holds from \Cref{lem:coarse_event}. For the second result, we have
{\small\begin{align*}
    &\big|\big[\mathbb{V}_h \hat{V}_{k, h+1}^\rho\big]\big(s_h^k, a_h^k\big)-[\mathbb{V}_h V_{h+1}^{*,\rho}]\big(s_h^k, a_h^k\big)\big| \\
    &= \Big|\big[\PP_h^0 \big(\hat{V}_{k, h+1}^\rho\big)^2 \big]\big(s_h^k, a_h^k\big) - \big(\big[\PP_h^0 \hat{V}_{k, h+1}^\rho \big]\big(s_h^k, a_h^k\big)\big)^2 - \big[\PP_h^0 \big(V_{h+1}^{*,\rho}\big)^2 \big]\big(s_h^k, a_h^k\big) +\big(\big[\PP_h^0 V_{h+1}^{*,\rho} \big]\big(s_h^k, a_h^k\big)\big)^2\Big| \\
    &\leq \Big|\big[\PP_h^0 \big(\hat{V}_{k, h+1}^\rho\big)^2 \big]\big(s_h^k, a_h^k\big) - \big[\PP_h^0 \big(V_{h+1}^{*,\rho}\big)^2 \big]\big(s_h^k, a_h^k\big)\Big| + \Big| \big(\big[\PP_h^0 \hat{V}_{k, h+1}^\rho \big]\big(s_h^k, a_h^k\big)\big)^2 - \big(\big[\PP_h^0 V_{h+1}^{*,\rho} \big]\big(s_h^k, a_h^k\big)\big)^2\Big| \\
    &\leq 4H \Big|\big[\PP_h^0 \hat{V}_{k, h+1}^\rho \big]\big(s_h^k, a_h^k\big) - \big[\PP_h^0 V_{h+1}^{*,\rho} \big]\big(s_h^k, a_h^k\big)\Big| \\
    &\leq 4H \Big(\big[\PP_h^0 \hat{V}_{k, h+1}^\rho \big]\big(s_h^k, a_h^k\big) - \big[\PP_h^0 V_{h+1}^{*,\rho} \big]\big(s_h^k, a_h^k\big)\Big) \\
    &\leq 4H \Big(\big[\PP_h^0 \hat{V}_{k, h+1}^\rho \big]\big(s_h^k, a_h^k\big) - \big[\PP_h^0 \check{V}_{k, h+1}^\rho \big]\big(s_h^k, a_h^k\big)\Big) \\
    &\leq \min \Big\{4 H \Big(\bphi\big(s_h^k, a_h^k\big)^\top \hat{\bz}^k_{h,1} - \bphi\big(s_h^k, a_h^k\big)^\top \check{\bz}^k_{h,1}
    + 2 \bar{\beta}\big\|\bphi\big(s_h^k, a_h^k\big)\big\|_{\bLambda_{k, h}^{-1}}\Big), H^2\Big\} \\
    &= D_{k,h}.
\end{align*}}%
where the second inequality holds because $0 \leq V_{h+1}^{*,\rho}, \hat{V}_{k, h+1}^\rho \leq H$, the third and fourth inequality holds because of \Cref{lem:optimism_pessimism}, the fifth inequality holds due to 
\Cref{lem:coarse_event} and the last inequality holds because the trivial result $0 \leq \big[\bar{\mathbb{V}}_h \hat{V}_{k, h+1}^\rho\big]\big(s_h^k, a_h^k\big),[\mathbb{V}_h V_{h+1}^{*,\rho}]\big(s_h^k, a_h^k\big) \leq H^2$. Thus we have
\begin{align*}
    \big|\big[\bar{\mathbb{V}}_h \hat{V}_{k, h+1}^\rho\big]\big(s_h^k, a_h^k\big)-[\mathbb{V}_h V_{h+1}^{*,\rho}]\big(s_h^k, a_h^k\big)\big| \leq E_{k, h}+D_{k, h}.
\end{align*}
Then we also have
\begin{align*}
    \bar{\sigma}_{k, h}^2 \geq\big[\bar{\mathbb{V}}_{h} \hat{V}_{k, h+1}^\rho\big]\big(s_h^k, a_h^k\big) + E_{k, h} + D_{k, h} \geq \big[\mathbb{V}_h V_{h+1}^{*,\rho}\big]\big(s_h^k, a_h^k\big),
\end{align*}
where we use the definition of $\bar{\sigma}_{k, h}$ in \eqref{eq:estimated_variance_optimal_v_function}. This completes the proof.
\end{proof}

\subsection{Proof of \Cref{lem:error_bound}}

\begin{proof}[Proof of \Cref{lem:error_bound}]
For all $(s,a)\in\cS/\{s_f\}\times\cA$, for stage $h \leq h^\prime \leq H$ (we use $h$ to replace $h^\prime$ in this part for simplicity), we have
\begin{align*}
    Q_h^{\pi, \rho}(s,a) = r_h(s,a) + \bphi(s,a)^\top \bnu_{h}^{\pi,\rho} = r_h(s,a) + \inf_{P_h(\cdot|s,a) \in \cU_{h}^{\rho}(s,a;\bmu_h^0)} \big[\PP_hV_{h+1}^{\pi,\rho}\big](s,a).
\end{align*}
We first decompose the gap $\hat{\bnu}_{h}^{\rho, k} - \bnu_{h}^{\pi, \rho}$ into two terms
\begin{align}
\label{eq:decomp_model_error}
    \hat{\bnu}_{h}^{\rho, k} - \bnu_{h}^{\pi, \rho} = \underbrace{\hat{\bnu}_{h}^{\rho, k} - \tilde{\bnu}_h^{\rho, k}}_{\text{I}}+ \underbrace{\tilde{\bnu}_h^{\rho, k} - \bnu_{h}^{\pi, \rho}}_{\text{II}},
\end{align}
where $\tilde{\bnu}_h^{\rho, k} = \big[\tilde{\nu}_{h,i}^{\rho, k}\big]_{i\in[d]}$, and $\tilde{\nu}_{h,i}^{\rho, k} = \max_{\alpha \in [0,H]}\big\{\EE^{\mu_{h,i}^0} \big[\hat{V}_{k, h+1}^{\rho}(s)\big]_{\alpha}-\rho\alpha\big\}$. Then we will bound these two terms separately.
\paragraph{Bound term I in \eqref{eq:decomp_model_error}:} we have
\begin{align*}
    &\hat{\bnu}_{h}^{\rho, k} - \tilde{\bnu}_h^{\rho, k} \leq \Big[\max_{\alpha \in [0,H]} \Big\{\hat{z}_{h,i}^k (\alpha) - \EE^{\mu_{h,i}^0} \Big[\hat{V}_{k, h+1}^{\rho}(s)\Big]_{\alpha} \Big\}\Big]_{i\in[d]}.
\end{align*}
Denote $\alpha_i^k = \argmax_{\alpha \in [0,H]}\big\{ \hat{z}_{h,i}^k (\alpha) - \EE^{\mu_{h,i}^0} \big[\hat{V}_{k, h+1}^{\rho}(s)\big]_{\alpha}\big\},~i=1,\cdots,d$. Then we have
{\small
\begin{align}
    \hat{\bnu}_{h}^{\rho, k} - \tilde{\bnu}_h^{\rho, k} &\leq \bigg[\Big (\bSigma_{k,h}^{-1} \sum_{\tau=1}^{k-1} \bar{\sigma}^{-2}_{\tau,h} \bphi\big(s_h^\tau, a_h^\tau\big) \big[\hat{V}_{k, h+1}^\rho\big(s_{h+1}^\tau\big)\big]_{\alpha_i^k} \Big)_i - \Big(\EE^{\bmu_{h}^0} \big[\hat{V}_{k, h+1}^{\rho}(s)\big]_{\alpha_i^k} \Big)_i \bigg]_{i \in [d]} \notag \\
    &=\bigg[\Big(-\lambda\bSigma_{k, h}^{-1} \EE^{\bmu_{h}^0} \big[\hat{V}_{k, h+1}^{\rho}(s)\big]_{\alpha_i^k}\Big)_i + \Big(\bSigma_{k, h}^{-1}\sum_{\tau=1}^{k-1}\bar{\sigma}^{-2}_{\tau,h}\bphi_h^{\tau}\Big[\big[\hat{V}_{k, h+1}^{\rho}(s_{h+1}^{\tau})\big]_{\alpha_i^k} \notag \\
    &\qquad -\big[\PP_h^0\big[\hat{V}_{k, h+1}^{\rho}\big]_{\alpha_i^k}\big](s_h^{\tau}, a_h^{\tau}) \Big] \Big)_i\bigg]_{i\in[d]}.
\label{eq:difference decomposition}
\end{align}
}
For the first term on the RHS of \eqref{eq:difference decomposition}, 
\begin{align}
    &\bigg|\bigg\la \bphi(s,a), \bigg[\bigg(-\lambda\bSigma_{k, h}^{-1} \EE^{\bmu_{h}^0} \Big[\hat{V}_{k, h+1}^{\rho}(s)\Big]_{\alpha_i^k}\bigg)_i \bigg]_{i\in[d]} \bigg\ra \bigg|\notag \\
    &=\bigg|\sum_{i=1}^d\phi_i(s,a)\mathbf{1}_i^{\top}(-\lambda)\bSigma_{k, h}^{-1}\EE^{\bmu_h^0}\Big[\hat{V}_{k, h+1}^{\rho}(s)\Big]_{\alpha_i^k}  \bigg|\notag \\
    &\leq \lambda \sum_{i=1}^d\sqrt{\phi_i(s,a)\mathbf{1}_i^{\top}\bSigma_{k, h}^{-1}\phi_i(s,a)\mathbf{1}_i} \cdot \bigg\Vert \EE^{\bmu_h^0}\Big[\hat{V}_{k, h+1}^{\rho}(s) \Big]_{\alpha_i^k}\bigg\Vert_{\bSigma_{k, h}^{-1}}\notag\\
    & \leq \sqrt{\lambda d}H \sum_{i=1}^d\sqrt{\phi_i(s,a)\mathbf{1}_i^{\top}\bSigma_{k, h}^{-1}\phi_i(s,a)\mathbf{1}_i},
\label{eq:first term bound}
\end{align}
where $\mathbf{1}_i$ is the vector with the $i$-th entry being 1 and else being 0. The first inequality holds due to the Cauchy-Schwarz inequality.\\
For the second term on the RHS of \eqref{eq:difference decomposition}, given the event $\cE_h$ defined in \Cref{lem:covering_concentration}, we have
{\small\begin{align}
    &\bigg|\bigg\la \bphi(s,a), \bigg[ \bigg(\bSigma_{k, h}^{-1}\sum_{\tau=1}^{k-1}\bar{\sigma}^{-2}_{\tau,h}\bphi_h^{\tau}\bigg[\Big[\hat{V}_{k, h+1}^{\rho}(s_{h+1}^{\tau})\Big]_{\alpha_i^k} - \Big[\PP_h^0 \Big[\hat{V}_{k, h+1}^{\rho}\Big]_{\alpha_i^k}\Big](s_h^{\tau}, a_h^{\tau}) \bigg] \bigg)_i \bigg]_{i\in[d]} \bigg\ra \bigg|\notag\\
    & = \bigg|\sum_{i=1}^d\phi_i(s,a) \mathbf{1}_i^{\top}\bSigma_{k, h}^{-1}\sum_{\tau=1}^{k-1}\bar{\sigma}^{-2}_{\tau,h} \bphi_h^{\tau}\bigg[\Big[\hat{V}_{k, h+1}^{\rho}(s_{h+1}^{\tau})\Big]_{\alpha_i^k} - \Big[\PP_h^0 \Big[\hat{V}_{k, h+1}^{\rho}\Big]_{\alpha_i^k}\Big](s_h^{\tau}, a_h^{\tau})\bigg]\bigg| \notag\\
    & \leq \sum_{i=1}^d\sqrt{\phi_i(s,a)\mathbf{1}_i^{\top}\bSigma_{k, h}^{-1}\phi_i(s,a)\mathbf{1}_i} \cdot \bigg\Vert \sum_{\tau=1}^{k-1}\bar{\sigma}^{-2}_{\tau,h}\bphi_h^{\tau}\Big[\big[\hat{V}_{k,h+1}^{\rho}(s_{h+1}^{\tau})\big]_{\alpha_i^k}-\big[\PP_h^0 \big[\hat{V}_{k, h+1}^{\rho}\big]_{\alpha_i^k}\big](s_h^{\tau}, a_h^{\tau}) \Big]\bigg\Vert_{\bSigma_{k, h}^{-1}}\notag\\
    & \leq \gamma \sum_{i=1}^d\sqrt{\phi_i(s,a)\mathbf{1}_i^{\top}\bSigma_{k, h}^{-1}\phi_i(s,a)\mathbf{1}_i},
\label{eq:second term bound}
\end{align}}%
where the first inequality follows from Cauchy-Schwarz inequality and the second inequality holds from the event $\cE_h$. Combining \eqref{eq:difference decomposition}, \eqref{eq:first term bound} and \eqref{eq:second term bound}, we have
\begin{align*}
    \big\la\bphi(s,a), \hat{\bnu}_{h}^{\rho, k} - \tilde{\bnu}_h^{\rho, k} \big\ra \leq  \Big(\sqrt{\lambda d}H + \gamma \Big) \sum_{i=1}^d \phi_i(s,a) \sqrt{\mathbf{1}_i^{\top}\bSigma_{k, h}^{-1}\mathbf{1}_i},
\end{align*}
On the other hand, we can similarly do analysis for
$\la\bphi(s,a),  \tilde{\bnu}_h^{\rho, k}-\hat{\bnu}_{h}^{\rho, k}  \ra$. Then we have
\begin{align}
\label{eq:bound on estimation error}
    \big|\la\bphi(s,a), \hat{\bnu}_{h}^{\rho, k} - \tilde{\bnu}_h^{\rho, k} \ra\big| \leq  \beta \sum_{i=1}^d \phi_i(s,a) \sqrt{\mathbf{1}_i^{\top}\bSigma_{k, h}^{-1}\mathbf{1}_i},
\end{align}
where $\beta = \Big(\sqrt{\lambda d}H + \gamma \Big) = \widetilde{\cO} \Big(\sqrt{\lambda d}H +\sqrt{d}\Big) $. 

\paragraph{Bound term II in \eqref{eq:decomp_model_error}:} we have
{\small\begin{align*}
    \big \la \bphi(s,a), \tilde{\bnu}_h^{\rho, k} - \bnu_h^{\pi, \rho}\big \ra = \inf_{P_h(\cdot|s,a) \in \cU_{h}^{\rho}(s,a;\bmu_h^0)}\Big[\PP_h \hat{V}_{k, h+1}^{\rho}\Big](s,a) - \inf_{P_h(\cdot|s,a) \in \cU_{h}^{\rho}(s,a;\bmu_h^0)}\Big[\PP_hV_{h+1}^{\pi, \rho}\Big](s,a).
\end{align*}}%
Finally we have 
\begin{align*}
    &\big(r_h(s,a) + \bphi(s,a)^\top \hat{\bnu}_h^{\rho,k}\big) - Q_h^{\pi, \rho}(s,a) \\
    &= \la\bphi(s,a),\hat{\bnu}_{h}^{\rho, k} - \tilde{\bnu}_{h}^{k, \rho} + \tilde{\bnu}_{h}^{k, \rho} -\bnu_{h}^{\pi, \rho}\ra \\
    &= \inf_{P_h(\cdot|s,a) \in \cU_{h}^{\rho}(s,a;\bmu_h^0)}\big[\PP_h \hat{V}_{k, h+1}^\rho\big](s,a) - \inf_{P_h(\cdot|s,a) \in \cU_{h}^{\rho}(s,a;\bmu_h^0)} \big[\PP_h V_{h+1}^{\pi, \rho}\big](s,a) + \Delta_h^k(s,a),
\end{align*}
where $|\Delta_h^k(s,a)|\leq \beta \sum_{i=1}^d \phi_i(s,a) \sqrt{\mathbf{1}_i^{\top}\bSigma_{k,h}^{-1}\mathbf{1}_i}$. This completes the proof.  
\end{proof}

\subsection{Proof of \Cref{lem:optimism_pessimism}}

\begin{proof}[Proof of \Cref{lem:optimism_pessimism}]
We prove this lemma by induction. For last stage $H+1$, it is trivial because for all $(s,a) \in \cS \times \cA$, we have $\hat{Q}_{k, H+1}^\rho(s, a) = Q_{H+1}^{*,\rho}(s, a) = \check{Q}_{k, H+1}^\rho(s, a) = 0$.

Assume that the lemma holds at stage $h^\prime+1$, now consider the situation at stage $h^\prime$ (we use $h$ to replace $h^\prime$ in this part for simplicity). For all episode $k \in [K]$, we have 
{\small\begin{align*}
    &r_h(s,a) + \bphi(s,a)^\top \hat{\bnu}_h^{\rho,k} + \hat{\Gamma}_{k,h}(s,a) - Q_h^{*,\rho}(s, a) \\
    &\geq \inf_{P_h(\cdot|s,a) \in \cU_{h}^{\rho}(s,a;\bmu_h^0)}\big[\PP_h \hat{V}_{k, h+1}^\rho\big](s,a) - \inf_{P_h(\cdot|s,a) \in \cU_{h}^{\rho}(s,a;\bmu_h^0)} \big[\PP_h V_{h+1}^{*, \rho}\big](s,a) + \Delta_h^k(s,a) + \hat{\Gamma}_{k,h}(s,a) \\
    &\geq \inf_{P_h(\cdot|s,a) \in \cU_{h}^{\rho}(s,a;\bmu_h^0)}\big[\PP_h \big(\hat{V}_{k, h+1}^\rho - V_{h+1}^{*, \rho}\big)\big](s,a) \\
    &\geq 0,
\end{align*}}%
where the first inequality holds from \Cref{lem:error_bound}, the second inequality holds because $|\Delta_h^k(s,a)| \leq \hat{\Gamma}_{k,h}(s,a)$, the third inequality holds from induction assumption. Thus we have
\begin{align*}
    Q_h^{*,\rho}(s, a) \leq \min \Big\{ \min_{i\in[k]}  r_h(s,a) + \bphi(s,a)^\top \hat{\bnu}_h^{\rho,i} + \hat{\Gamma}_{i,h}(s,a), H-h+1 \Big\} \leq \hat{Q}_{k, h}^\rho(s, a).
\end{align*}
Thus for value function $V$, we have 
\begin{align*}
    \hat{V}_{k, h}^\rho(s) = \max_a \hat{Q}_{k, h}^\rho(s, a) \geq \max_a Q_h^{*,\rho}(s, a) = V_h^{*,\rho}(s).
\end{align*}
For the pessimistic value function $\check{Q}_{k, h}^\rho(s, a)$, we can do the similar analysis. Finally, by induction, we finish the proof.
\end{proof}

\subsection{Proof of \Cref{lem:good_event_probability}}

\begin{proof}[Proof of \Cref{lem:good_event_probability}]
We use backward induction to prove this lemma. For the base case, the stage $H$, it is trivial to obtain \eqref{eq:covering_concentration} because $\hat{V}_{k,H+1}^{\rho}$ = 0. Assume \eqref{eq:covering_concentration} hold for the stage $h+1$, then we consider the stage $h$.

For all episode $k \in [K]$, we first do the following decomposition
\begin{align}
\label{equ:covering_concentration_decomposition}
    &\bigg\Vert \sum_{\tau=1}^{k-1}\bar{\sigma}^{-2}_{\tau,h}\bphi_h^{\tau}\Big[\big[\hat{V}_{k,h+1}^{\rho}(s_{h+1}^{\tau})\big]_{\alpha^\prime}-\big[\PP_h^0 \big[\hat{V}_{k, h+1}^{\rho}\big]_{\alpha^\prime}\big](s_h^{\tau}, a_h^{\tau}) \Big]\bigg\Vert_{\bSigma_{k, h}^{-1}} \notag\\
    &\leq \underbrace{\bigg\Vert \sum_{\tau=1}^{k-1}\bar{\sigma}^{-2}_{\tau,h}\bphi_h^{\tau}\Big[\big[V_{h+1}^{*,\rho}(s_{h+1}^{\tau})\big]_{\alpha^\prime}-\big[\PP_h^0 \big[V_{h+1}^{*,\rho}\big]_{\alpha^\prime}\big](s_h^{\tau}, a_h^{\tau}) \Big]\bigg\Vert_{\bSigma_{k, h}^{-1}}}_{J_1} \notag\\
    &\qquad+ \underbrace{\bigg\Vert \sum_{\tau=1}^{k-1}\bar{\sigma}^{-2}_{\tau,h}\bphi_h^{\tau}\Big[\Delta_{\alpha^\prime} \hat{V}_{k, h+1}^{\rho} (s_{h+1}^{\tau}) - \big[\PP_h^0 \big(\Delta_{\alpha^\prime} \hat{V}_{k, h+1}^{\rho}\big)\big](s_h^{\tau}, a_h^{\tau}) \Big]\bigg\Vert_{\bSigma_{k, h}^{-1}}}_{J_2},
\end{align}
where $\Delta_{\alpha^\prime} \hat{V}_{k, h+1}^{\rho} (s_{h+1}^{\tau}) = \big[\hat{V}_{k,h+1}^{\rho}(s_{h+1}^{\tau})\big]_{\alpha^\prime} - \big[V_{h+1}^{*,\rho} (s_{h+1}^{\tau})\big]_{\alpha^\prime}$. \\
\paragraph{Bound term $J_1$ in \eqref{equ:covering_concentration_decomposition}:} For term $J_1$, we apply \Cref{lem:bernstein_concentration} with $\xb_i = \bar{\sigma}_{i,h}^{-1}\bphi\big(s^i_h,a^i_h\big)$ and $\eta_i = \bar{\sigma}_{i,h}^{-1}\big( \big[V_{h+1}^{*,\rho}(s_{h+1}^i)\big]_{\alpha^\prime}-\big[\PP_h^0 \big[V_{h+1}^{*,\rho}\big]_{\alpha^\prime}\big]\big(s_h^i, a_h^i\big) \big)$. Note that based on \Cref{lem:variance_error}, we have $\bar{\sigma}_{k, h}^2 \geq \big[\mathbb{V}_h V_{h+1}^{*,\rho}\big]\big(s_h^k, a_h^k\big) \geq \big[\mathbb{V}_h \big[V_{h+1}^{*,\rho}\big]_{\alpha^\prime}\big]\big(s_h^k, a_h^k\big)$. Then for $\xb_i$ and $\eta_i$, we have
\begin{align*}
    &\|\xb_i\|_2 = \big\|\bphi\big(s^i_h,a^i_h\big)\big\|_2 / \bar{\sigma}_{i,h} \leq 1, \\
    &\EE[\eta_i|\cF_i] = 0, |\eta_i| \leq \big|\bar{\sigma}_{i,h}^{-1}\big( \big[V_{h+1}^{*,\rho}(s_{h+1}^i)\big]_{\alpha^\prime}-\big[\PP_h^0 \big[V_{h+1}^{*,\rho}\big]_{\alpha^\prime}\big]\big(s_h^i, a_h^i\big) \big)\big| \leq 2H,\\
    &\EE[\eta_i^2|\cF_i] = \EE \big[ \bar{\sigma}_{i,h}^{-2}\big( \big[V_{h+1}^{*,\rho}(s_{h+1}^i)\big]_{\alpha^\prime}-\big[\PP_h^0 \big[V_{h+1}^{*,\rho}\big]_{\alpha^\prime}\big]\big(s_h^i, a_h^i\big) \big)^2\big] \leq 1 = \sigma^2, \\
    &\max_{1 \leq i \leq k} \Big\{|\eta_i|\cdot \min \big\{1, \|\xb_i\|_{\bSigma^{-1}_{i,h}}\big\} \Big\} \leq \max_{1 \leq i \leq k} \Big\{ 2H \bar{\sigma}_{i,h}^{-1} \|\xb_i\|_{\bSigma^{-1}_{i,h}} \Big\} \leq \sqrt{d},
\end{align*}
where we use the definition of $\bar{\sigma}_{i,h}$ in \eqref{eq:weight}. Then for all $k \in [K]$, with probability at least $1-\delta/2H$, we have
\begin{align*}
    J_1 = \bigg\|\sum_{i=1}^{k-1} \xb_i\eta_i\bigg\|_{\bSigma_{k,h}^{-1}} \leq \widetilde{\cO}\Big(\sigma \sqrt{d}+\max _{1 \leq i \leq k}|\eta_i| \min \big\{1,\|\xb_i\|_{\bSigma^{-1}_{i,h}}\big\}\Big) = \widetilde{O}\big(\sqrt{d}\big).
\end{align*}
\paragraph{Bound term $J_2$ in \eqref{equ:covering_concentration_decomposition}:} To bound term $J_2$, we need to use $\epsilon$-covering for function class $\widetilde{\cV}_{h+1} - \big[V_{h+1}^{*,\rho}\big]_{\alpha^\prime}$ where $\widetilde{\cV}_{h+1} = \big\{[V]_\alpha | V \in \cV_{h+1}, \alpha \in [0,H]\big\}$ is the truncated optimistic value function class. For any two function $\widetilde{V}_1, \widetilde{V}_2 \in \widetilde{\cV}_{h+1}$, we can write that
\begin{align*}
    \widetilde{V}_1 = [V_1]_{\alpha_1}, \widetilde{V}_2 = [V_2]_{\alpha_2}, 
\end{align*}
where $V_1, V_2 \in \cV_{h+1}$, $\alpha_1, \alpha_2 \in [0,H]$. Then we have
\begin{align*}
    \text{dist}(\widetilde{V}_1, \widetilde{V}_2) &= \sup_s \big|\widetilde{V}_1(s)-\widetilde{V}_2(s)\big| \\
     &= \sup_s \big|[V_1]_{\alpha_1}(s)-[V_2]_{\alpha_2}(s)\big| \\
     &\leq \sup_s \big|[V_1]_{\alpha_1}(s)-[V_1]_{\alpha_2}(s)\big| + \sup_s \big|[V_1]_{\alpha_2}(s)-[V_2]_{\alpha_2}(s)\big| \\
     &\leq |\alpha_1 - \alpha_2| + \text{dist}(V_1, V_2).
\end{align*}
This indicates that the $\epsilon$-covering number $\widetilde{\cN}_\epsilon$ for function class $\widetilde{\cV}_{h+1}$ can be bounded by
\begin{align*}
    \widetilde{\cN}_\epsilon \leq \cN_{1,\frac{\epsilon}{2}} \cdot \cN_{2,\frac{\epsilon}{2}},
\end{align*}
where $\cN_{1,\frac{\epsilon}{2}}$ is the $\frac{\epsilon}{2}$-covering number for optimistic value function class $\cV_{h+1}$ and $\cN_{2,\frac{\epsilon}{2}}$ is the $\frac{\epsilon}{2}$-covering number for closed interval $[0,H]$. Then based on \Cref{lem:function_class_covering_number} and \Cref{lemma:Covering number of an interval}, we have
\begin{align*}
    \log \widetilde{\cN}_\epsilon \leq d \ell \log (1 + 8L/\epsilon) + d^2 \ell \log \big(1+32 \sqrt{d} \beta^2 / \lambda \epsilon^2\big) + \log(6H/\epsilon),
\end{align*}
where  $\ell = dH\log(1+K/\lambda)$ and $L = 2H\sqrt{dK/\lambda}$. Here we set $\epsilon = \sqrt{\lambda}/ 4 H^2 d^3 K$, then the covering entropy can be bounded by 
\begin{align*}
    \log \widetilde{\cN}_\epsilon \leq \widetilde{\cO}(d^3 H).
\end{align*}
For simplicity, we denote $\Delta_{\alpha^\prime} \hat{V}_{k, h+1}^{\rho}$ here as $\Delta V$, then for $\Delta V$, there exsit a function $\widetilde{V}$ in the $\epsilon$-net satisfies that
\begin{align*}
    \text{dist}\big(\Delta V, \widetilde{V}\big) \leq \epsilon.
\end{align*}
Then the difference of the variance of $\Delta V$ and $\widetilde{V}$ can be bounded by
\begin{align*}
    &\big[\mathbb{V}_h \widetilde{V} \big] \big(s^k_h,a^k_h\big) - \big[\mathbb{V}_h \Delta V \big] \big(s^k_h,a^k_h\big) \\
    &= \big[\mathbb{P}_h \widetilde{V}^2 \big] \big(s^k_h,a^k_h\big) - \big[\mathbb{P}_h {(\Delta V)}^2 \big] \big(s^k_h,a^k_h\big) - \big(\big[\mathbb{P}_h \widetilde{V} \big] \big(s^k_h,a^k_h\big)\big)^2 + \big(\big[\mathbb{P}_h (\Delta V) \big] \big(s^k_h,a^k_h\big)\big)^2 \\
    &\leq 2 \sup_s\big|\Delta V(s)- \widetilde{V}(s)\big| \cdot \sup_s\big|\Delta V(s)+ \widetilde{V}(s)\big| \\
    &\leq 4H \text{dist}\big(\Delta V, \widetilde{V}\big) \\
    &\leq \frac{1}{2d^3H}.
\end{align*}
This indicates that
\begin{align}
\label{equ:variance_difference_deltav_widetildev}
\big[\mathbb{V}_h \widetilde{V} \big] \big(s^k_h,a^k_h\big) &\leq \big[\mathbb{V}_h \Delta V \big] \big(s^k_h,a^k_h\big) + \frac{1}{2d^3H} \notag\\
&\leq \Big[\PP_h\big(\big[\hat{V}_{k,h+1}^{\rho}\big]_{\alpha^\prime} - \big[V_{h+1}^{*,\rho} \big]_{\alpha^\prime}\big)^2\Big]\big(s^k_h,a^k_h\big) + \frac{1}{2d^3H} \notag\\
&\leq \Big[\PP_h\big(\hat{V}_{k,h+1}^{\rho} -V_{h+1}^{*,\rho}\big)^2\Big]\big(s^k_h,a^k_h\big) + \frac{1}{2d^3H} \notag\\
&\leq 2H \Big[\PP_h\big(\hat{V}_{k,h+1}^{\rho} - \check{V}_{k,h+1}^{\rho}\big)\Big]\big(s^k_h,a^k_h\big) + \frac{1}{2d^3H} \notag\\
&\leq 2H \Big(\PP_h\hat{V}_{k,h+1}^{\rho}\big(s^k_h,a^k_h\big) - \PP_h\check{V}_{k,h+1}^{\rho} \big(s^k_h,a^k_h\big)\Big) + \frac{1}{2d^3H} \notag\\
&\leq D_{k,h} + \frac{1}{2d^3H} \notag\\
&\leq \bar{\sigma}^{2}_{k,h} / d^3 H,
\end{align}
where the fourth inequality holds due to \Cref{lem:optimism_pessimism} with induction assumption $\cE_{h+1}$, the sixth inequality holds due to the definition of $D_{k,h}$ and the last inequality holds because of the definition of $\sigma_{k,h}$. Then we apply we apply \Cref{lem:bernstein_concentration} with $\xb_i = \bar{\sigma}_{i,h}^{-1}\bphi\big(s^i_h,a^i_h\big)$ and $\eta_i = \bar{\sigma}_{i,h}^{-1}\big( \widetilde{V}(s_{h+1}^i)-\PP_h^0 \widetilde{V}\big(s_h^i, a_h^i\big) \big)$. For $\xb_i$ and $\eta_i$, we have
\begin{align*}
    &\|\xb_i\|_2 \leq \big\|\bphi\big(s^i_h,a^i_h\big)\big\|_2 / \bar{\sigma}_{i,h} \leq 1, \\
    &\EE[\eta_i|\cF_i] = 0, |\eta_i| \leq \big|\bar{\sigma}_{i,h}^{-1}\big( \widetilde{V}(s_{h+1}^i)-\PP_h^0 \widetilde{V}\big(s_h^i, a_h^i\big) \big)\big| \leq 2H,\\
    &\EE[\eta_i^2|\cF_i] = \bar{\sigma}_{i,h}^{-2} \big[\VV_h \widetilde{V}\big] \big(s_h^i, a_h^i\big) \leq 1/d^3H, \\
    &\max_{1 \leq i \leq k} \Big\{|\eta_i|\cdot \min \big\{1, \|\xb_i\|_{\bSigma^{-1}_{i,h}}\big\} \Big\} \leq \max_{1 \leq i \leq k} \Big\{ 2H \bar{\sigma}_{i,h}^{-1} \|\xb_i\|_{\bSigma^{-1}_{i,h}} \Big\} \leq 1/d^3H,
\end{align*}
where we use the construction of $\bar{\sigma}_{i,h}$ in \eqref{eq:weight} and \eqref{equ:variance_difference_deltav_widetildev}. After taking union probability bound over $\epsilon$-covering for function class $\widetilde{\cV}_{h+1} - \big[V_{h+1}^{*,\rho}\big]_{\alpha^\prime}$, we have 
\begin{align*}
    \bigg\Vert \sum_{\tau=1}^{k-1}\bar{\sigma}^{-2}_{\tau,h}\bphi_h^{\tau}\Big[ \widetilde{V}(s_{h+1}^{\tau}) - \big[\PP_h^0 \widetilde{V}\big](s_h^{\tau}, a_h^{\tau}) \Big]\bigg\Vert_{\bSigma_{k, h}^{-1}} \leq \widetilde{\cO}\big(\sqrt{d}\big).
\end{align*}
For simplicity, we denote that $\bar{V}= \Delta V -\widetilde{V} = \Delta_{\alpha^\prime} \hat{V}_{k, h+1}^{\rho} - \widetilde{V}$ and have $\sup_s|\bar{V}(s)| \leq \epsilon$. Then we obtain
\begin{align*}
    J_2 &= \bigg\Vert \sum_{\tau=1}^{k-1}\bar{\sigma}^{-2}_{\tau,h}\bphi_h^{\tau}\Big[\Delta_{\alpha^\prime} \hat{V}_{k, h+1}^{\rho} (s_{h+1}^{\tau}) - \big[\PP_h^0 \big(\Delta_{\alpha^\prime} \hat{V}_{k, h+1}^{\rho}\big)\big](s_h^{\tau}, a_h^{\tau}) \Big]\bigg\Vert_{\bSigma_{k, h}^{-1}} \\
    &\leq 2 \bigg\Vert \sum_{\tau=1}^{k-1}\bar{\sigma}^{-2}_{\tau,h}\bphi_h^{\tau}\Big[ \widetilde{V}(s_{h+1}^{\tau}) - \big[\PP_h^0 \widetilde{V}\big](s_h^{\tau}, a_h^{\tau}) \Big]\bigg\Vert_{\bSigma_{k, h}^{-1}} \\
    &\qquad+ 2\bigg\Vert \sum_{\tau=1}^{k-1}\bar{\sigma}^{-2}_{\tau,h}\bphi_h^{\tau}\Big[ \bar{V}(s_{h+1}^{\tau}) - \big[\PP_h^0 \bar{V}\big](s_h^{\tau}, a_h^{\tau}) \Big]\bigg\Vert_{\bSigma_{k, h}^{-1}} \\
    &\leq \widetilde{\cO}\big(\sqrt{d}\big) + 4\epsilon k /\sqrt{\lambda} \\
    &\leq \widetilde{\cO}\big(\sqrt{d}\big),
\end{align*}
where we use that $\epsilon = \sqrt{\lambda}/ 4 H^2 d^3 K$. Finally, we have
\begin{align*}
    \bigg\Vert \sum_{\tau=1}^{k-1}\bar{\sigma}^{-2}_{\tau,h}\bphi_h^{\tau}\Big[\big[\hat{V}_{k,h+1}^{\rho}(s_{h+1}^{\tau})\big]_{\alpha^\prime}-\big[\PP_h^0 \big[\hat{V}_{k, h+1}^{\rho}\big]_{\alpha^\prime}\big](s_h^{\tau}, a_h^{\tau}) \Big]\bigg\Vert_{\bSigma_{k, h}^{-1}} = J_1+J_2 \leq \gamma,
\end{align*}
where $\gamma = \widetilde{\cO}\big(\sqrt{d}\big)$. Thus, by induction we complete the proof.
\end{proof}

\subsection{Proof of \Cref{lem:est_var_upper_bound}}

\begin{proof}[Proof of \Cref{lem:est_var_upper_bound}]
Conditioned on the event $\mathcal{E}$ and $\bar{\mathcal{E}}$, to bound the weight $\bar{\sigma}_{k,h}^2$, recall from the definition \eqref{eq:weight}, we have
\begin{align*}
    \bar{\sigma}_{k, h} = \max \Big\{\sigma_{k, h}, 1, \sqrt{2d^3H^2} \big\|\bphi\big(s_h^k, a_h^k\big)\big\|_{\bSigma_{k, h}^{-1}}^{\frac{1}{2}} \Big\},
\end{align*}
According to \eqref{eq:estimated_variance_optimal_v_function}, we have
\begin{align*} 
    \sigma_{k,h}^2 = \big[\bar{\mathbb{V}}_h \hat{V}_{k, h+1}^\rho\big]\big(s_h^k, a_h^k\big) + E_{k,h} + d^3 H \cdot D_{k,h} + \frac{1}{2},
\end{align*}
where $E_{k,h}, D_{k,h}$ are defined as follows
\begin{align*}
    E_{k,h} &= \min \Big\{\widetilde{\beta} \big\|\bphi\big(s_h^k, a_h^k\big)\big\|_{\bLambda_{k, h}^{-1}}, H^2\Big\} + \min \Big\{2 H \bar{\beta}\big\|\bphi\big(s_h^k, a_h^k\big)\big\|_{\bLambda_{k, h}^{-1}}, H^2\Big\}, \\
    D_{k,h} &= \min \Big\{4 H \Big(\bphi\big(s_h^k, a_h^k\big)^\top \hat{\bz}^k_{h,1} - \bphi\big(s_h^k, a_h^k\big)^\top \check{\bz}^k_{h,1}
    + 2 \bar{\beta}\big\|\bphi\big(s_h^k, a_h^k\big)\big\|_{\bLambda_{k, h}^{-1}}\Big), H^2\Big\},
\end{align*}
where $\bar{\beta} = \widetilde{\cO}\big(d^{\frac{3}{2}} H^{\frac{3}{2}}\big)$, $\widetilde{\beta} = \widetilde{\cO}\big(d^{\frac{3}{2}} H^3\big)$ when we set $\lambda = 1/H^2$. 
Note that 
\begin{align*}
    \sqrt{2d^3H^2} \big\|\bphi\big(s_h^k, a_h^k\big)\big\|_{\bLambda_{k, h}^{-1}}^{\frac{1}{2}} \leq \sqrt{2d^3H^2} \big\|\bphi\big(s_h^k, a_h^k\big)\big\|_2^{\frac{1}{2}} / {\lambda}^{\frac{1}{4}} \leq \sqrt{2d^3H^3}.
\end{align*}
Also note that 
\begin{align*}
    \sigma_{k,h}^2 &= \big[\bar{\mathbb{V}}_h \hat{V}_{k, h+1}^\rho\big]\big(s_h^k, a_h^k\big) + E_{k,h} + d^3 H \cdot D_{k,h} + \frac{1}{2} \\
    &\leq H^2 + 2H^2 + d^3H \cdot H^2 + \frac{1}{2} \\
    &\leq 2 d^3H^3.
\end{align*}
Then we obtain the trivial upper bound $\bar{\alpha}$ for $\bar{\sigma}_{k, h}$ 
\begin{align}
\label{eq:upper bound for bar sigma}
    \bar{\sigma}_{k, h} \leq 2\sqrt{d^3H^3} = \bar{\alpha},
\end{align}
Based on \Cref{lem:variance_error}, we have
\begin{align*}
    \big[\bar{\mathbb{V}}_h \hat{V}_{k, h+1}^\rho\big]\big(s_h^k, a_h^k\big) \leq E_{k, h}+D_{k, h} + [\mathbb{V}_h V_{h+1}^{*,\rho}]\big(s_h^k, a_h^k\big).
\end{align*}
Then we have 
\begin{align*} 
    \sigma_{k,h}^2 \leq [\mathbb{V}_h V_{h+1}^{*,\rho}]\big(s_h^k, a_h^k\big) + 2E_{k, h} + 2d^3 H \cdot D_{k, h} + \frac{1}{2}.
\end{align*}
Next, we carefully bound $\sigma_{k,h}^2$ and $\bar{\sigma}_{k,h}^2$. 
To this end, we bound term $E_{k,h}, D_{k,h}$ when $k$ is large enough.
The intuition is that, when the episode $k$ is large enough, all the error terms should be small under the assumption \eqref{eq:lambda_min_lower_bound}. 

\paragraph{Bound term $E_{k,h}$:} 
Note that based on \eqref{assumption:lambda_lower_bound}, with the same analysis as the proof of Corollary 5.3 in \citet{liu2024distributionally}, with probability at least $1-\delta$, we have
\begin{align*}
    \lambda_{\min}(\bLambda_{k,h}) \geq \max\big\{c(k-1)/d+\lambda-\sqrt{32k\log(dKH/\delta)}, \lambda \big\}.
\end{align*}
Then when we choose $k > 512d^2 \log(dKH/\delta)/c^2$ and note that $\lambda=1/H^2$, we have
\begin{align*}
    c(k-1)/d+\lambda-\sqrt{32k\log(dKH/\delta)} \geq \frac{c}{2d}  k,
\end{align*}
which indicates that
\begin{align}
\label{eq:lambda_min_lower_bound}
    \lambda_{\min}(\bLambda_{k,h}) \geq \frac{c}{2d} k.
\end{align}
Then when $k > 512d^2 \log(dKH/\delta)/c^2$, we can calculate that
\begin{align}
\label{eq:phi_norm_bound}
    \big\|\bphi\big(s_h^k, a_h^k\big)\big\|_{\bLambda_{k, h}^{-1}} &=  \Big\|\bLambda_{k, h}^{-\frac{1}{2}}\bphi\big(s_h^k, a_h^k\big)\Big\|_2 \leq \sqrt{\lambda_{\text{max}}\big(\bLambda_{k, h}^{-1}\big)}\leq \sqrt{\frac{2d}{kc}},
\end{align}
where in the first inequality we use the fact that $\|\bphi(s,a)\|_2\leq 1$ for all $(s,a)\in\cS\times\cA$.
Then when $k$ is large enough and also at least $k > 512d^2 \log(dKH/\delta)/c^2$, we can have that
\begin{align*}
    E_{k,h} \leq \widetilde{\cO} \Big(\frac{1}{\sqrt{kc}} d^{2} H^3\Big).
\end{align*}
This indicates that there exists an absolute constant $c_E>0$ such that
\begin{align*}
    E_{k,h} \leq c_E\frac{d^{2} H^3}{\sqrt{k}} .
\end{align*}
\textbf{Bound term $D_{k,h}$:} note that $\hat{\bz}^k_{h,1}$ and $\check{\bz}^k_{h,1}$ have the closed-form expression as follows
\begin{align*}
    \hat{\bz}_{h,1}^k &= \bLambda_{k,h}^{-1} \sum_{\tau=1}^{k-1} \bphi\big(s_h^\tau, a_h^\tau\big) \hat{V}_{k, h+1}^\rho\big(s_{h+1}^\tau\big), \\
    \check{\bz}_{h,1}^k &= \bLambda_{k,h}^{-1} \sum_{\tau=1}^{k-1} \bphi\big(s_h^\tau, a_h^\tau\big) \check{V}_{k, h+1}^\rho\big(s_{h+1}^\tau\big).
\end{align*}
Thus we can calculate that
\begin{align}
\label{eq:est_para_diff}
    &\bphi\big(s_h^k, a_h^k\big)^\top \hat{\bz}^k_{h,1} - \bphi\big(s_h^k, a_h^k\big)^\top \check{\bz}^k_{h,1} \notag\\
    &= \bphi\big(s_h^k, a_h^k\big)^\top \bLambda_{k,h}^{-1} \sum_{\tau=1}^{k-1}  \bphi\big(s_h^\tau, a_h^\tau\big) \big(\hat{V}_{k, h+1}^\rho\big(s_{h+1}^\tau\big) - \check{V}_{k, h+1}^\rho\big(s_{h+1}^\tau\big)\big), \notag\\
    & \leq \big\|\bphi\big(s_h^k, a_h^k\big)\big\|_{\bLambda_{k, h}^{-1}} \cdot \bigg\|\sum_{\tau=1}^{k-1}  \bphi\big(s_h^\tau, a_h^\tau\big)\big(\hat{V}_{k, h+1}^\rho\big(s_{h+1}^\tau\big) - \check{V}_{k, h+1}^\rho\big(s_{h+1}^\tau\big)\big)\bigg\|_{\bLambda_{k, h}^{-1}} \notag\\
    & \leq \big\|\bphi\big(s_h^k, a_h^k\big)\big\|_{\bLambda_{k, h}^{-1}} \cdot \sum_{\tau=1}^{k-1} \big\| \bphi\big(s_h^\tau, a_h^\tau\big)\big\|_{\bLambda_{k, h}^{-1}} \cdot \big(\hat{V}_{k, h+1}^\rho(\tilde{s}_{h+1}^k) - \check{V}_{k, h+1}^\rho(\tilde{s}_{h+1}^k)\big) \notag\\
    &\leq \sqrt{\frac{2d}{kc}} \sum_{\tau=1}^{k-1} \sqrt{\frac{2d}{kc}} \cdot \big(\hat{V}_{k, h+1}^\rho(\tilde{s}_{h+1}^k) - \check{V}_{k, h+1}^\rho(\tilde{s}_{h+1}^k)\big), \notag\\
    &\leq 2d/c \cdot \big(\hat{V}_{k, h+1}^\rho(\tilde{s}_{h+1}^k) - \check{V}_{k, h+1}^\rho(\tilde{s}_{h+1}^k)\big),
\end{align}
where $\tilde{s}_{h+1}^k = \argmax_ {s \in \cS} \Big\{\hat{V}_{k, h+1}^\rho(s) - \check{V}_{k, h+1}^\rho(s)\Big\}$, the first inequality holds because of Cauchy-Schwarz inequality, the third inequality holds due to \eqref{eq:phi_norm_bound}.
Next, we bound $\hat{V}_{k, h+1}^\rho(\tilde{s}_{h+1}^k) - \check{V}_{k, h+1}^\rho(\tilde{s}_{h+1}^k)$. The intuition is that when $k$ is large, both $\hat{V}_{k, h+1}^\rho$ and $\check{V}_{k, h+1}^\rho$ should be close to the robust optimal value function. Thus, $\hat{V}_{k, h+1}^\rho$ should be close to $\check{V}_{k, h+1}^\rho$, and the closeness could be quantified by the bonus terms, which is of order $\widetilde{\cO}(d/\sqrt{k})$ under the assumption \eqref{eq:lambda_min_lower_bound}. In particular, we have
\begin{align}
\label{eq:value_diff}
    &\hat{V}_{k, h+1}^\rho(\tilde{s}_{h+1}^k) - \check{V}_{k, h+1}^\rho(\tilde{s}_{h+1}^k) \notag\\
    & = \underbrace{\hat{V}_{k, h+1}^\rho(\tilde{s}_{h+1}^k) - V_{h+1}^{*,\rho}(\tilde{s}_{h+1}^k)}_{\text{I}} + \underbrace{V_{h+1}^{*,\rho}(\tilde{s}_{h+1}^k) - \check{V}_{k, h+1}^\rho(\tilde{s}_{h+1}^k)}_{\text{II}}.
\end{align}
\paragraph{Bound term I in \eqref{eq:value_diff}:} note that
\begin{align*}
    &\hat{V}_{k, h}^\rho(s) - V_{h}^{*,\rho}(s) \\
    &= \hat{Q}_{k, h}^\rho\big(s, \pi_h^k(s)\big) - Q_{h}^{*,\rho}\big(s, \pi_h^*(s)\big) \\
    &\leq \hat{Q}_{k, h}^\rho\big(s, \pi_h^k(s)\big) - Q_{h}^{*,\rho}\big(s, \pi_h^k(s)\big) \\
    &\leq \inf_{P_h(\cdot|s,a) \in \cU_{h}^{\rho}(s,a;\bmu_h^0)}\big[\PP_h \hat{V}_{k, h+1}^\rho\big]\big(s, \pi_h^k(s)\big) - \inf_{P_h(\cdot|s,a) \in \cU_{h}^{\rho}(s,a;\bmu_h^0)} \big[\PP_h V_{h+1}^{*, \rho}\big]\big(s, \pi_h^k(s)\big) \\
    &\qquad+ \Delta_h^k\big(s, \pi_h^k(s)\big) + \hat{\Gamma}_{k,h}\big(s, \pi_h^k(s)\big) \\
    &\leq \big[\hat{\PP}_h\big(\hat{V}_{k,h+1}^{\rho}- V_{h+1}^{*, \rho}\big)\big]\big(s, \pi_h^k(s)\big) + 2\hat{\Gamma}_{k,h}\big(s, \pi_h^k(s)\big),
\end{align*}
where the second inequality holds due to the definition of $\hat{Q}_{k, h}^\rho$, robust Bellman equation and \Cref{lem:error_bound}, $\widehat{P}_h(\cdot|s,a)=\text{arginf}_{P_h(\cdot|s,a)\in\cU_h^{\rho}(s,a;\bmu_{h,i}^0)}\big[\PP_h V_{h+1}^{*,\rho}\big](s,a), \forall (s,a)\in\cS\times\cA$. By recursively applying it, then we have
\begin{align*}
    \hat{V}_{k, h}^\rho(s) - V_{h}^{*,\rho}(s) &\leq \big[\hat{\PP}_h\big(\hat{V}_{k,h+1}^{\rho}- V_{h+1}^{*, \rho}\big)\big]\big(s, \pi_h^k(s)\big) + 2\hat{\Gamma}_{k,h}\big(s, \pi_h^k(s)\big) \\
    &\leq 2\sum_{h^\prime=h}^H \EE^{\pi_{h^\prime}^k, \hat{P}}\big[\hat{\Gamma}_{k,h^\prime}(s,a)|s_{h^\prime}=s\big].
\end{align*}
Note that by \eqref{eq:upper bound for bar sigma} $\bar{\sigma}_{k,h}^2 \leq \bar{\alpha}^2$, we have $\bSigma_{k,h}  \succcurlyeq \bar{\alpha}^{-2} \bLambda_{k,h}$. Similar to the analysis of \eqref{eq:phi_norm_bound}, when $k > 512/c^2 \log(dKH/\delta)$, we have
\begin{align*}
    \hat{\Gamma}_{k,h}(s,a) &= \beta \sum_{i=1}^d \phi_i(s,a) \sqrt{\mathbf{1}_i^{\top}\bSigma_{k,h}^{-1}\mathbf{1}_i} \\
    &\leq \beta \bar{\alpha} \sum_{i=1}^d \phi_i\big(s_h^k, a_h^k\big) \sqrt{\mathbf{1}_i^\top \bLambda_{k, h}^{-1} \mathbf{1}_i} \notag \\
    &\leq \beta \bar{\alpha} \sqrt{\lambda_{\text{max}}\big(\bLambda_{k, h}^{-1}\big)} \notag \\
    &\leq \sqrt{\frac{2d}{kc}}\beta \bar{\alpha}.
\end{align*}
Then we have
\begin{align*}
    \hat{V}_{k, h}^\rho(s) - V_{h}^{*,\rho}(s) \leq 2H\sqrt{\frac{2d}{kc}}\beta \bar{\alpha} \leq \frac{4\beta\sqrt{d} \bar{\alpha} H}{\sqrt{kc}}.
\end{align*}
Therefore, we can bound $I$ as follows
\begin{align*}
    \text{I} = \hat{V}_{k, h+1}^\rho(\tilde{s}_{h+1}^k) - V_{h+1}^{*,\rho}(\tilde{s}_{h+1}^k) \leq \frac{4\beta\sqrt{d} \bar{\alpha} H}{\sqrt{kc}}.
\end{align*}
\paragraph{Bound term II in \eqref{eq:value_diff}:} Similar to the analysis above, we can derive the similar result as follows
\begin{align*}
\text{II} = V_{h+1}^{*,\rho}(\tilde{s}_{h+1}^k) - \check{V}_{k, h+1}^\rho(\tilde{s}_{h+1}^k) \leq \frac{4\bar{\beta}\sqrt{d} \bar{\alpha} H}{\sqrt{kc}}.
\end{align*}
Now we can bound that
\begin{align*}
    \bphi\big(s_h^k, a_h^k\big)^\top \hat{\bz}^k_{h,1} - \bphi\big(s_h^k, a_h^k\big)^\top \check{\bz}^k_{h,1} \leq \frac{2d}{c}\cdot\frac{4(\bar{\beta}+\beta)\sqrt{d} \bar{\alpha} H}{\sqrt{kc}} \leq \frac{16\bar{\beta} d^{3/2}\bar{\alpha} H}{\sqrt{kc^3}}.
\end{align*}
Then when $k$ is large enough, we can have that
\begin{align*}
    D_{k,h} \leq \widetilde{\cO} \Big(\frac{\bar{\alpha}}{\sqrt{k}} d^{3} H^{\frac{7}{2}}\Big).
\end{align*}
This indicates that there exists an absolute constant $c_D >0$ such that
\begin{align*}
    D_{k,h} \leq c_D \frac{\bar{\alpha}}{\sqrt{k}} d^{3} H^{\frac{7}{2}}.
\end{align*}
When $k$ is large enough, we have
\begin{align*}
    \sigma_{k,h}^2 &\leq [\mathbb{V}_h V_{h+1}^{*,\rho}]\big(s_h^k, a_h^k\big) + (2E_{k, h} + 2d^3 H \cdot D_{k, h}) + \frac{1}{2} \\
    &\leq [\mathbb{V}_h V_{h+1}^{*,\rho}]\big(s_h^k, a_h^k\big) + 2 c_E\frac{\bar{\alpha}}{\sqrt{k}} d^{2} H^3 + 2 c_D \frac{\bar{\alpha}}{\sqrt{k}} d^{6} H^{\frac{9}{2}} + \frac{1}{2}.
\end{align*}
When we choose $\widetilde{K} = \widetilde{c}\cdot \bar{\alpha}^2 d^{12} H^{9}$ where $\widetilde{c}=\widetilde{\cO}(1)$. When $k > \widetilde{K}$, then we have
\begin{align*}
    \bar{\sigma}_{k,h}^2 &= \max \Big\{\sigma_{k, h}^2, 1, 2d^3H^2 \big\|\bphi\big(s_h^k, a_h^k\big)\big\|_{\bSigma_{k, h}^{-1}} \Big\} \\
    &\leq \max \big\{\big[\mathbb{V}_h V_{h+1}^{*,\rho}\big]\big(s_h^k, a_h^k\big) + 1, 1 \big\} \\
    & \leq 2 \big[\mathbb{V}_h V_{h+1}^{*,\rho}\big(s_h^k, a_h^k\big)\big]_{[1,H^2]}.
\end{align*}
Based on \Cref{lem:range_shrinkage}, we have
\begin{align*}
    \big[\VV_h V_{h+1}^{*,\rho}\big](s,a)\leq \Big(\frac{1-(1-\rho)^{H-h+1}}{\rho}\Big)^2\leq \Big(\frac{1-(1-\rho)^{H}}{\rho}\Big)^2 = \bTheta\Big(\min \Big\{\frac{1}{\rho^2}, H^2\Big\}\Big).
\end{align*}
Then when $k > \widetilde{K}$, we have
\begin{align*}
    \bar{\sigma}_{k,h}^2 \leq \cO\Big(\min \Big\{\frac{1}{\rho^2}, H^2\Big\}\Big).
\end{align*}
Additionally, note that $\bar{\alpha}^2 =  \cO\big(d^3 H^3\big)$, we have 
\begin{align*}
    \widetilde{K} = \widetilde{\cO} \big( d^{15} H^{12} \big).
\end{align*}
This completes the proof.
\end{proof}

\section{Supporting Lemmas}

\begin{lemma}[Number of value function updates]
\label{lem:number_of_value_function_updates}
The number of episodes where the algorithm updates the value function in \Cref{alg:DR-LSVI-UCB+} is upper bounded by $d H \log (1+K / \lambda)$.
\end{lemma}

\begin{proof}[Proof of \Cref{lem:number_of_value_function_updates}]
    This proof is the same as \citet[Lemma F.1]{he2023nearly} because of the same rare-switching condition (Line \ref{line:update_rule} in \Cref{alg:DR-LSVI-UCB+}).
\end{proof}

\begin{lemma}
\label{lem:dual_parameter_bound}
For any $(k,h) \in [K]\times[H]$, the weight $\hat{\bnu}_h^{\rho,k}$ satisfies
\begin{align*}
    \big\Vert \hat{\bnu}_h^{\rho,k} \big\Vert_2 \leq 2H\sqrt{dk/\lambda}.
\end{align*}
\end{lemma}

\begin{proof}[Proof of \Cref{lem:dual_parameter_bound}]
Denote $\alpha_i = \argmax_{\alpha\in[0,H]} \big\{\hat{z}^{k}_{h,i}(\alpha)-\rho\alpha\big\}, i\in[d]$. Then we have
\begin{align*}
    \big\Vert \hat{\bnu}_h^{\rho,k} \big\Vert_2 &= \Bigg\|\bigg[\max_{\alpha\in[0,H]}\{\hat{z}^{k}_{h,i}(\alpha)-\rho\alpha\} \bigg]_{i\in [d]} \Bigg\|_2 \\
    &\leq  \rho \sqrt{d} \alpha + \Bigg\|\bigg[\bigg(\bSigma_{k,h}^{-1}\sum_{\tau=1}^{k-1} \bar{\sigma}^{-2}_{\tau,h} \bphi\big(s_h^\tau, a_h^\tau\big) \big[\hat{V}_{k, h+1}^\rho\big(s_{h+1}^\tau\big)\big]_{\alpha_i}\bigg)_{i}\bigg]_{i \in [d]}\Bigg\|_2 \\
    &\leq H\sqrt{d} + H \cdot  \Bigg\|\bSigma_{k,h}^{-1}\sum_{\tau=1}^{k-1} \bar{\sigma}^{-2}_{\tau,h} \bphi\big(s_h^\tau, a_h^\tau\big)\Bigg\|_2 \\
    &\leq H\sqrt{d} + H \sqrt{k/\lambda} \cdot \Bigg(\sum_{\tau=1}^{k-1}\big(\bar{\sigma}^{-1}_{\tau,h}\bphi\big(s_h^\tau, a_h^\tau\big)\big)^{\top}\bSigma_{k,h}^{-1}\big(\bar{\sigma}^{-1}_{\tau,h}\bphi\big(s_h^\tau, a_h^\tau\big)\big)\Bigg)^{\frac{1}{2}} \\
    &\leq H\sqrt{d} + H\sqrt{dk/\lambda} \\
    &\leq 2H\sqrt{dk/\lambda},
\end{align*}
where the first inequality holds due to the triangle inequality, the second inequality holds from the fact that $\rho \leq 1$, $0\leq \alpha \leq H$ and $0 \leq \big[\hat{V}_{k, h+1}^\rho\big(s_{h+1}^\tau\big)\big]_{\alpha_i} \leq H$, the third inequality holds because of \Cref{lem:inequal_eigen_matrix} and the fourth inequality holds because $\bSigma_{k,h} \succcurlyeq \lambda \Ib$ and \Cref{lemma:self-normalize}. This completes the proof.
\end{proof}

\begin{lemma}
\label{lem:linear_form_and_bound}
Under a linear MDP, for any stage $h \in [H]$  and any bounded function $V: \mathcal{S} \rightarrow [0, H]$, there always exists a vector $\bz \in \mathbb{R}^d$ such that for all $(s, a) \in \mathcal{S} \times \mathcal{A}$, we have
\begin{align*}
    \big[\mathbb{P}_h^0 V\big](s, a)=\bz^{\top} \bphi(s, a),
\end{align*}
where $\bz$ satisfies that $\|\bz\|_2 \leq H \sqrt{d}$.
\end{lemma}

\begin{proof}[Proof of \Cref{lem:linear_form_and_bound}]
Based on \Cref{assumption:linear_mdp}, we have
\begin{align*}
    \big[\mathbb{P}_h^0 V\big](s,a) &= \int \mathbb{P}_h^0 (s^{\prime}|s,a) V(s^{\prime}) d s^{\prime} \\
    &= \int \bphi(s,a)^{\top} V(s^{\prime}) d \bmu_h^0(s^{\prime}) \\
    &= \bphi(s,a)^{\top} \int V(s^{\prime}) d \bmu_h^0(s^{\prime}) \\
    &= \bphi(s,a)^{\top} \bz,
\end{align*}
where $\bz=\int V(s^{\prime}) d \bmu_h^0(s^{\prime})$. Thus we have
\begin{align*}
    \|\bz\|_2 = \bigg\|\int V(s^{\prime}) d \bmu_h^0(s^{\prime})\bigg\|_2 \leq \max_{s^{\prime}} V(s^{\prime}) \cdot \big\|\bmu_h^0(\mathcal{S})\big\|_2 \leq H\sqrt{d}.
\end{align*}
This completes the proof.
\end{proof}

\section{Proof of the Minimax Lower Bound}
\label{sec:Proof of the Minimax Lower Bound }
In this section, we prove the minimax lower bound. To this end, we first introduce the construction of hard instances in \Cref{sec:Construction of Hard Instances},  and then we prove \Cref{thm:lower_bound} in \Cref{sec:proof of lower bound th}.

\subsection{Construction of Hard Instances}
\label{sec:Construction of Hard Instances}
We construct a family of $d$-rectangular linear DRMDPs based on the hard-to-learn linear MDP introduced in \citet{zhou2021nearly}. Let $\delta = 1/H$, $\Delta = \sqrt{\delta/K}/(4\sqrt{2})$. 
Each $d$-rectangular linear DRMDP in this family is parameterized by a Boolean vector $\bxi = \{\bxi_h\}_{h\in[H-1]}$, where $\bxi_h \in \{-\Delta,\Delta\}^d$. For a given $\bxi$ and uncertainty level $\rho\in(0,3/4]$, the corresponding $d$-rectangular linear DRMDP $M_{\bxi}^{\rho}$ has the following structure. The state space $\cS=\{x_1, x_2, \cdots, x_H, x_{H+1}\}$ and the action space $\cA=\{-1,1\}^d$. The first state is always $x_1$. The feature mapping $\bphi:\cS\times\cA \rightarrow \RR^{2d+2}$ is defined to depend on the state $x_h$ through $\bxi_h$ as follows: 
\begin{align*}
    &\phi(x_1,a) = \left(\begin{aligned}
    \frac{1}{2d} - &\frac{\delta}{d} - \xi_{11}a_1 \\ \frac{1}{2d} - &\frac{\delta}{d} - \xi_{12}a_2\\  & \vdots \\ \frac{1}{2d} - &\frac{\delta}{d} - \xi_{1d}a_d \\ &  \frac{1}{2} \\ & \frac{\delta}{d} + \xi_{11}a_1\\ & \frac{\delta}{d} + \xi_{12}a_2 \\ &\vdots \\ & \frac{\delta}{d} + \xi_{1d}a_d \\ &0
    \end{aligned}  \right), \phi(x_2,a) = \left(\begin{aligned}
    \frac{1}{2d} - &\frac{\delta}{d} - \xi_{21}a_1 \\ \frac{1}{2d} - &\frac{\delta}{d} - \xi_{22}a_2\\  & \vdots \\ \frac{1}{2d} - &\frac{\delta}{d} - \xi_{2d}a_d \\ &  \frac{1}{2} \\ & \frac{\delta}{d} + \xi_{21}a_1\\ & \frac{\delta}{d} + \xi_{22}a_2 \\ &\vdots \\ & \frac{\delta}{d} + \xi_{2d}a_d \\ &0 
    \end{aligned}  \right), \cdots,\\ & \bphi(x_{H-1},a) = \left(\begin{aligned}
    \frac{1}{2d} - &\frac{\delta}{d} - \xi_{H-1,1}a_1 \\ \frac{1}{2d} - &\frac{\delta}{d} - \xi_{H-1,2}a_2\\  & \vdots \\ \frac{1}{2d} - &\frac{\delta}{d} - \xi_{H-1,d}a_d \\ &  \frac{1}{2} \\ & \frac{\delta}{d} + \xi_{H-1,1}a_1\\ & \frac{\delta}{d} + \xi_{H-1,2}a_2 \\ &\vdots \\ & \frac{\delta}{d} + \xi_{H-1,d}a_d \\ &0 
    \end{aligned} \right), 
    \bphi(x_H,a) = \left(\begin{aligned}
    0 \\ 0\\ \vdots \\ 0 \\  0 \\ 0\\ 0 \\ \vdots \\ 0\\ 1 
    \end{aligned} \right), \bphi(x_{H+1},a) = \left(\begin{aligned}
    0 \\ 0\\ \vdots \\ 0 \\  0 \\ \frac{1}{d}\\ \frac{1}{d} \\ \vdots \\ \frac{1}{d}\\ 0 
    \end{aligned} \right).
\end{align*}
We assume that
\begin{align}
\label{eq: lower bound choice of K and H}
    K\geq 9d^2H/32~\text{and}~H\geq 6, 
\end{align}
such that  $ \frac{1}{2d} - \frac{1}{dH} - \delta \geq 0$.
Then it can be easily checked that for any $s\in\cS$, we have $\phi_i(s,a)\geq 0$ %
and $\sum_{i=1}^{2d+2}\phi_i(s,a)=1$. 
The factor distribution $\bmu_1: \cS\rightarrow \RR^{2d+2}$ is defined as follows.
\begin{align*}
    \bmu_1(\cdot) &= (\underbrace{\delta_{x_2}(\cdot),  \cdots, \delta_{x_2}(\cdot)}_{d \text{ terms}},\delta_{x_2}(\cdot), \underbrace{\delta_{x_{H+1}}(\cdot),  \cdots,\delta_{x_{H+1}}(\cdot)}_{d \text{ terms}}, \delta_{x_H}(\cdot))^\top.
\end{align*}
Similarly, for $h=2,\ldots,H$, we have
\begin{align*}
    \bmu_2(\cdot) &= (\delta_{x_3}(\cdot), \cdots, \delta_{x_3}(\cdot),\delta_{x_3}(\cdot), \delta_{x_{H+1}}(\cdot), \cdots,\delta_{x_{H+1}}(\cdot), \delta_{x_H}(\cdot))^\top, \\
    &\cdots \\
     \bmu_{H-1}(\cdot) = \bmu_H(\cdot) &= (\delta_{x_H}(\cdot), \cdots, \delta_{x_H}(\cdot),\delta_{x_H}(\cdot), \delta_{x_{H+1}}(\cdot), \cdots,\delta_{x_{H+1}}(\cdot), \delta_{x_H}(\cdot))^\top,
\end{align*}
Note that for each episode $k$, the initial state $s_1^k$ is always $x_1$. In the nominal environment, at step $h$, the state $s_h^k$ is either $x_h$ or $x_{H+1}$. State $x_{H}$ and $x_{H+1}$ are absorbing states. \Cref{fig:nominal env} illustrates the nominal MDP.

Now we construct the reward parameters $\{\btheta_h\}_{h\in[H]}$ as follows.
\begin{align*}
    \btheta_h = (1,1,\cdots, 1, -1,1, 1,\cdots, 1,0)^\top, ~\forall h\in[H].
\end{align*}
We have $\forall h\in[H]$,
\begin{align*}
    r_h(x_H,a) &= \bphi(x_H,\ba)^\top\btheta_h = 0, \\
    r_h(x_h,a) & = \bphi(x_h,\ba)^\top\btheta_h = 0, \\
    r_h(x_{H+1},a) & = \bphi(x_{H+1},\ba)^\top\btheta_h = 1.
\end{align*}
Thus, only the transition starting from $x_{H+1}$ generates a reward of 1, and transitions starting from any other state generate 0 reward. Next, we consider the model perturbation. An observation is that $x_H$ is the worst state since it is an absorbing state with zero reward. By the definition of the $d$-rectangular uncertainty set, the worst case kernel is the linear combination of worst case factor distributions. Further, by the definition of the factor uncertainty set, the worst case factor distribution is the one that leads to the highest probability $\rho$ to the worst state $x_H$.
Thus, the worst factor distributions are %
{\small\begin{align*}
    \check{\bmu}_1 &= ((1-\rho)\delta_{x_2} + \rho\delta_{x_H}, (1-\rho)\delta_{x_2} + \rho\delta_{x_H}, \cdots, (1-\rho)\delta_{x_2} + \rho\delta_{x_H},(1-\rho)\delta_{x_2} + \rho\delta_{x_H}, \\
    &\qquad (1-\rho)\delta_{x_{H+1}} + \rho\delta_{x_H}, (1-\rho)\delta_{x_{H+1}} + \rho\delta_{x_H}, \cdots,(1-\rho)\delta_{x_{H+1}} + \rho\delta_{x_H}, \delta_{x_H})^\top, \\ 
    \check{\bmu}_2 &= ((1-\rho)\delta_{x_3} + \rho\delta_{x_H}, (1-\rho)\delta_{x_3} + \rho\delta_{x_H}, \cdots, (1-\rho)\delta_{x_3} + \rho\delta_{x_H},(1-\rho)\delta_{x_3} + \rho\delta_{x_H}, \\
    &\qquad (1-\rho)\delta_{x_{H+1}} + \rho\delta_{x_H}, (1-\rho)\delta_{x_{H+1}} + \rho\delta_{x_H}, \cdots,(1-\rho)\delta_{x_{H+1}} + \rho\delta_{x_H}, \delta_{x_H})^\top,\\
    &\cdots \\
     \check{\bmu}_{H-1} &= ((1-\rho)\delta_{x_{H-1}} + \rho\delta_{x_H}, (1-\rho)\delta_{x_{H-1}} + \rho\delta_{x_H}, \cdots, (1-\rho)\delta_{x_{H-1}} + \rho\delta_{x_H},(1-\rho)\delta_{x_{H-1}} + \rho\delta_{x_H}, \\
    &\qquad (1-\rho)\delta_{x_{H+1}} + \rho\delta_{x_H}, (1-\rho)\delta_{x_{H+1}} + \rho\delta_{x_H}, \cdots,(1-\rho)\delta_{x_{H+1}} + \rho\delta_{x_H}, \delta_{x_H})^\top,\\
    \check{\bmu}_{H-1} &= \bmu_H. 
\end{align*}}%
\Cref{fig:worst env} illustrates the worst case MDP.

\begin{figure*}[t]
\centering
\tiny
\subfigure[The nominal MDP environment.]{
    \begin{tikzpicture}[->,>=stealth',shorten >=1pt,auto,node distance=3.4cm,thick,every node/.style={color=black}]
            \node[state] (S1) {$x_1$};
            \node[state] (S2) [right of=S1] {$x_2$};
            \node (ellipsis) [right of=S2] {$\cdots$};
            \node[state] (S3) [right=2cm of ellipsis] {$x_{H-1}$};
            \node[state] (S4) [right of=S3] {$x_H$};
            \node[state] (S5) [below=2cm of ellipsis] {$x_{H+1}$};
            
            \path   
            (S1) edge node[below] {\tiny$1-\delta-\la\bxi_1,\ba\ra$} (S2) edge node[left] {\tiny$\delta+\la\xi_1,a\ra$} (S5) 
            (S2) edge node[below] {\tiny$1-\delta-\la\bxi_2,\ba\ra$} (ellipsis) edge node[right] {\tiny$\delta+\la\bxi_2,\ba\ra$} (S5)
            (ellipsis) edge node[below] {} (S3)
            
            (S3) edge node[below] {\tiny$1-\delta-\la\bxi_{H-1},\ba\ra$} (S4)edge node[right] {\tiny$\delta+\la\bxi_{H-1},\ba\ra$} (S5)
            (S4) edge [loop right] node {\tiny 1} (S4)
            (S5) edge [loop below] node {\tiny 1} (S5)
            ;
\end{tikzpicture}
\label{fig:nominal env}
}

\subfigure[The worst-case MDP environment. %
]{
\begin{tikzpicture}[->,>=stealth',shorten >=1pt,auto,node distance=3.4cm,thick,every node/.style={color=black}]
            \node[state] (S1) {$x_1$};
            \node[state] (S2) [right=3cm of S1] {$x_2$};
            \node (ellipsis) [right of=S2] {$\cdots$};
            \node[state] (S3) [right=1cm of ellipsis] {$x_{H-1}$};
            \node[state] (S4) [right=4cm of S3] {$x_H$};
            \node[state] (S5) [below=2cm of ellipsis] {$x_{H+1}$};
            
            \path   
            (S1) edge node[above] {\tiny$(1-\rho)(1-\delta-\la\bxi_1,\ba\ra)$} (S2) edge node[left] {\tiny$(1-\rho)(\delta+\la\bxi_1,\ba\ra)$} (S5) 
            (S2) edge node[below] {\tiny$(1-\rho)(1-\delta-\la\bxi_2,\ba\ra)$} (ellipsis) edge node[right] {\tiny$(1-\rho)(\delta+\la\bxi_2,\ba\ra)$} (S5)
            (ellipsis) edge node[below] {} (S3)
            
            (S3) edge node[above] {\tiny$(1-\rho)(1-\delta-\la\bxi_{H-1},\ba\ra)+\rho$} (S4) edge node[pos=0.5, right] {\tiny$(1-\rho)(\delta+\la\bxi_{H-1},\ba\ra)$} (S5)
            (S4) edge [loop right] node {\tiny 1} (S4)
            (S5) edge [loop below] node {\tiny $1-\rho$} (S5)
            ;
            \draw[->, bend left] (S1) to node[above] {\tiny$\rho$} (S4);
            \draw[->, bend left] (S2) to node[above] {\tiny$\rho$} (S4);
            \draw[->, bend right] (S5) to node[above] {\tiny$\rho$} (S4);
\end{tikzpicture}
\label{fig:worst env}
} 
\caption{Constructions of the nominal MDP and the worst-case MDP environments.}
\end{figure*}
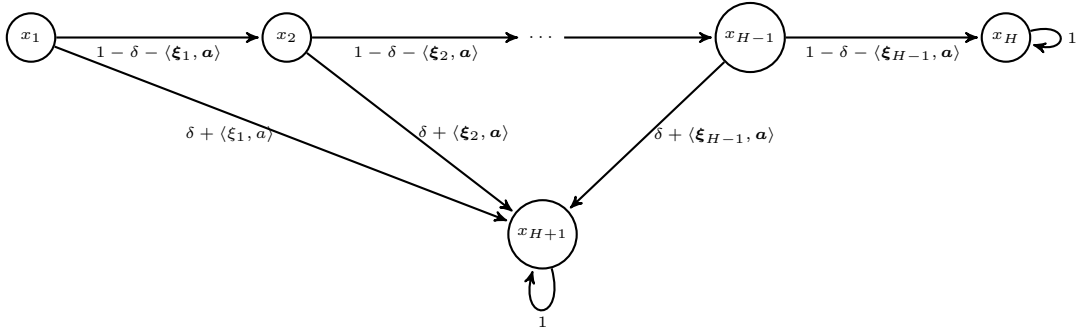
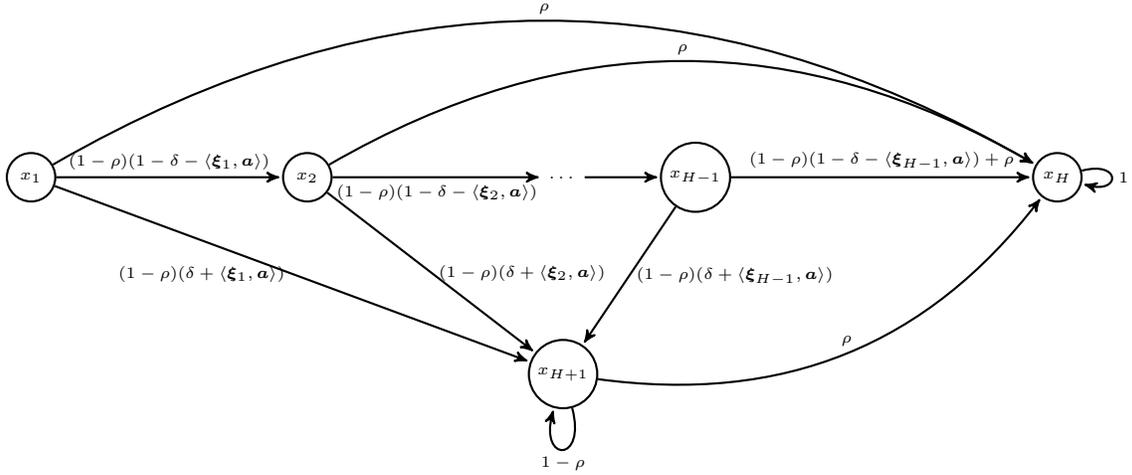

\subsection{Reduction from $d$-Rectangular DRMDP to Linear Bandits}
Note that by construction, at steps $h = 1,\cdots, H-2$, the probability of  transitioning to the worst case state $x_{H}$ is independent of the action $a$. Moreover, since $x_{H+1}$ is the only rewarding state, so the optimal action at step $h$ is the one the leads to the largest probability to $x_{H+1}$, i.e., $\ba_h^\star = \argmax_{\ba\in\cA}\la\bxi_h,\ba\ra$.   Further, in the nominal environment, state $x_h$ can only be reached through states $x_1, x_2, \cdots, x_{h-1}$. As discussed by \citet{zhou2021nearly}, knowing the state $x_h$ is equivalent to knowing the entire history starting from the initial state at current episode. Consequently, policies dictating what actions to take upon reaching a state at the beginning of an episode are equivalent to policies relying on the ``within episode'' history (we refer to the discussion in E.1 of \citet{zhou2021nearly} for more details).  In the following lemma, we shows that the average suboptimality of the $d$-rectangular DRMDP can be lower bounded by the regret of $H/2$ bandit instances. 

\begin{lemma}
\label{lemma:lower bound - lemma 1}
    With the choice of $d, K, H$ in \eqref{eq: lower bound choice of K and H}, we have $3d\Delta\leq \delta$. Fix $\bxi = \{\bxi_h\}_{h\in[H-1]}$. Fix a possibly history dependent policy $\pi$ and define $\bar{\ba}_h^\pi = \EE_{\bxi}[\ba_h|s_h=x_h]$ as the expected action taken by the policy when it visits state $x_h$ in stage $h$. Then, there exist a constant $c>0$ such that
    \begin{align*}
        V_1^{\star, \rho}(x_1) - V_1^{\pi, \rho}(x_1) \geq c\min\Big\{\frac{1}{\rho},H\Big\}\sum_{h=1}^{H/2}\Big(\max_{a\in\cA}\la \bmu_h, \ba\ra - \la\bmu_h,\bar{\ba}_h^\pi \ra\Big).
    \end{align*}
\end{lemma}

\begin{proof}[Proof of \Cref{lemma:lower bound - lemma 1}]
For the fixed policy $\pi$, we first get the ground truth robust value $V_1^{\pi, \rho}(x_1)$ by induction. Starting from the last step $H$, we have
\begin{align*}
    V_H^{\pi,\rho}(x_H) = 0,\quad V_H^{\pi,\rho}(x_{H+1}) = 1.
\end{align*}
For step $H-1$, we have
{\small\begin{align*}
    V_{H-1}^{\pi,\rho}(x_H) = 0, \quad V_{H-1}^{\pi,\rho}(x_{H-1}) = (1-\rho)(\delta+\la \bxi_{H-1},\ba_{H-1}^\pi\ra)\cdot 1, \quad V_{H-1}^{\pi,\rho}(x_{H+1})=1+(1-\rho)\cdot 1.
\end{align*}}%
For step $H-2$, we have $V_{H-2}^{\pi,\rho}(x_H) = 0$ and
\begin{align*}
     V_{H-2}^{\pi,\rho}(x_{H+1}) &= 1+(1-\rho)\cdot V_{H-1}(x_{H+1}) = 1 + (1-\rho) + (1-\rho)^2,\\
    V_{H-2}^{\pi, \rho}(x_{H-2}) & = (1-\rho)(\delta+\la\bxi_{H-2},\bar{\ba}_{H-2}^\pi \ra)\cdot V_{H-1}^{\pi,\rho}(x_{H+1}) \\
    &\qquad+ (1-\rho)(1-\delta-\la\bxi_{H-2},\bar{\ba}_{H-2}^\pi \ra)\cdot V_{H-1}^{\pi, \rho}(x_{H-1}) \\
    &= \big[(1-\rho)+(1-\rho)^2\big](\delta +
     \la \bxi_{H-2},\bar{\ba}_{H-2}^\pi \ra) \\
    &\qquad +(1-\rho)^2(1-\delta-\la\bxi_{H-2},\bar{\ba}_{H-2}^\pi \ra)(\delta+\la\bxi_{H-1},\bar{\ba}_{H-1}^\pi \ra).
\end{align*}
For step $H-3$, we have $V_{H-3}^{\pi,\rho}(x_H) = 0$ and
\begin{align*}
    V_{H-3}^{\pi,\rho}(x_{H+1}) &= 1 + (1-\rho)\cdot V_{H-2}(x_{H+1}) = 1+(1-\rho) + (1-\rho)^2+(1-\rho)^3,\\
    V_{H-3}^{\pi,\rho}(x_{H-3}) & = (1-\rho)(\delta+\la\bxi_{H-3},\bar{\ba}_{H-3}^\pi \ra)\cdot V_{H-2}^{\pi,\rho}(x_{H+1}) \\
    &\qquad+ (1-\rho)(1-\delta-\la\bxi_{H-3},\bar{\ba}_{H-3}^\pi \ra)\cdot V_{H-2}^{\pi, \rho}(x_{H-2})\\
    & = \big[(1-\rho) + (1-\rho)^2 + (1-\rho)^3 \big](\delta+\la\bxi_{H-3},\bar{\ba}_{H-3}^\pi \ra)  \\
    &\qquad +\big[(1-\rho)^2 + (1-\rho)^3 \big](1-\delta-\la\bxi_{H-3},\bar{\ba}_{H-3}^\pi \ra)(\delta+\la\bxi_{H-2},\bar{\ba}_{H-2}^\pi \ra)  \\
    &\qquad +(1-\rho)^3(1-\delta-\la\bxi_{H-3},\bar{\ba}_{H-3}^\pi \ra)(1-\delta-\la\bxi_{H-2},\bar{\ba}_{H-2}^\pi \ra)(\delta+\la\bxi_{H-1},\bar{\ba}_{H-1}^\pi \ra).
\end{align*}
Keep performing the backward induction until step $h=1$, we have%
{\small\begin{align}
    &V_1^{\pi, \rho}(x_1) \notag\\
    &= V_{H-(H-1)}^{\pi,\rho}(x_1)\notag\\
    &= \big[(1-\rho) + \cdots + (1-\rho)^{H-1} \big](\delta + \la \bxi_1, \bar{\ba}_1^\pi\ra) +\notag \\
    &\qquad \big[(1-\rho)^2 + \cdots + (1-\rho)^{H-1} \big](1-\delta-\la\bxi_1, \bar{\ba}_1^\pi \ra)(\delta+\la \bxi_2, \bar{\ba}_2^\pi \ra) + \notag\\
    & \qquad [(1-\rho)^3 + \cdots + (1-\rho)^{H-1}](1-\delta-\la\bxi_1, \bar{\ba}_1^\pi \ra)(1-\delta-\la\bxi_2, \bar{\ba}_2^\pi \ra)(\delta+\la \bxi_3,\bar{\ba}_3^\pi \ra) +\notag\\
    &\qquad +\cdots+\notag\\
    &\qquad (1-\rho)^{H-1}(1-\delta-\la\bxi_1, \bar{\ba}_1^\pi \ra)(1-\delta-\la\bxi_2, \bar{\ba}_2^\pi \ra)\cdots (1-\delta-\la \bxi_{H-2},\bar{\ba}_{H-2}^\pi \ra)(\delta+\la\bxi_{H-1},\bar{\ba}_{H-1}^\pi \ra) \notag\\
    &= \sum_{h=1}^{H-1}\Big(\sum_{i=h}^{H-1}(1-\rho)^i \Big)(o_h+\delta)\prod_{j=1}^{h-1}(1-o_j-\delta),\label{eq:V1pi}
\end{align}}%
where $o_h = \la \bxi_h, \bar{\ba}_h^\pi \ra, \forall h\in[H]$. Recall that the optimal robust action at step $h$ is $\ba_h^{\star} = \argmax_{\ba\in\cA} \la\bxi_h, \ba \ra$, and hence $\max_{\ba\in\cA} \la\bxi_h, \ba \ra = \Delta d$. Thus, we have
\begin{align}
\label{eq:V1star}
    V_1^{\star,\rho}(x_1) = \sum_{h=1}^{H-1}\Big(\sum_{i=h}^{H-1}(1-\rho)^i \Big)(d\Delta+\delta)\prod_{j=1}^{h-1}(1-d\Delta-\delta).
\end{align}
For $k\in[H-1]$, we define %
\begin{align}
    S_k &= \sum_{h=k}^{H-1}\Big(\sum_{i=h-k+1}^{H-k}(1-\rho)^i \Big)\prod_{j=k}^{h-1}(1-o_j-\delta)(o_h+\delta),\label{eq:S_k}\\
    T_k &= \sum_{h=k}^{H-1}\Big(\sum_{i=h-k+1}^{H-k}(1-\rho)^i \Big)\prod_{j=k}^{h-1}(1-d\Delta-\delta)(d\Delta+\delta).\label{eq:T_k}
\end{align}
Then by \eqref{eq:V1pi}, \eqref{eq:V1star}, \eqref{eq:S_k} and \eqref{eq:T_k}, we know $V_1^{\star,\rho}(x_1) - V_1^{\pi, \rho}(x_1) = T_1 - S_1$.
Next, we aim to lower bound $T_1 - S_1$.
Inspired by the backward induction process, we have
\begin{align*}
    S_k &= \Big(\sum_{i=1}^{H-k}(1-\rho)^i \Big)(o_k+\delta) + S_{k+1}(1-o_k-\delta),\\
    T_k &= \Big(\sum_{i=1}^{H-k}(1-\rho)^i \Big)(d\Delta+\delta) + T_{k+1}(1-d\Delta-\delta).
\end{align*}
Then, we have
\begin{align}
    T_k - S_k &= \Big(\sum_{i=1}^{H-k}(1-\rho)^i \Big)(d\Delta - o_k) - S_{k+1}(1-o_k-\delta) + T_{k+1}(1-d\Delta-\delta)\notag\\
    &= \Big(\sum_{i=1}^{H-k}(1-\rho)^i -T_{k+1} \Big)(d\Delta - o_k) + (1-o_k-\delta)(T_{k+1}-S_{k+1}).\label{eq: recursive formular of Tk-Sk}
\end{align}
Define $T_H =S_H=0$, then by the recursive formula \eqref{eq: recursive formular of Tk-Sk},  we have 
\begin{align}
\label{eq:T_1-S_1}
    T_1 - S_1 = \sum_{h=1}^{H-1}(d\Delta-o_h)\Big(\underbrace{\sum_{i=1}^{H-h}(1-\rho)^i - T_{h+1} }_{\text{I}}\Big)\prod_{j=1}^{h-1}(1-o_j-\delta).
\end{align}
To further bound \eqref{eq:T_1-S_1}, 
we first study the term I. 
Next we derive a close form expression of $T_k$. 
In specific, we have
\begin{align}
\label{eq:T_k expression}
    T_k &= \sum_{h=k}^{H-1}\Big(\sum_{i=h-k+1}^{H-k}(1-\rho)^i \Big)\prod_{j=k}^{h-1}(1-d\Delta-\delta)(d\Delta+\delta)\notag\\
    & = \Big(\sum_{i=1}^{H-k}(1-\rho)^i \Big)(d\Delta+\delta) + \Big(\sum_{i=2}^{H-k}(1-\rho)^i \Big)(1-d\Delta-\delta)(d\Delta+\delta) \notag\\
    &\qquad+ \Big(\sum_{i=3}^{H-k}(1-d\Delta-\delta)^2(d\Delta+\delta) \Big) +\cdots+(1-\rho)^{H-k}(1-d\Delta-\delta)^{H-k-1}(d\Delta+\delta).
\end{align}
Multiply $T_k$ by $(1-d\Delta-\delta)$, we have
\begin{align}
\label{eq:T_k multiply}
   &(1-d\Delta-\delta)T_k \notag\\
   &=\Big(\sum_{i=1}^{H-k}(1-\rho)^i \Big)(d\Delta+\delta)(1-d\Delta-\delta) + \Big(\sum_{i=2}^{H-k}(1-\rho)^i \Big)(1-d\Delta-\delta)^2(d\Delta+\delta) \notag \\ 
    &\qquad+ \Big(\sum_{i=3}^{H-k}(1-d\Delta-\delta)^2(d\Delta+\delta) \Big)  +\cdots+(1-\rho)^{H-k}(1-d\Delta-\delta)^{H-k}(d\Delta+\delta).
\end{align}
Then we have
\begin{align}
\label{eq:take difference}
    &\eqref{eq:T_k expression} - \eqref{eq:T_k multiply} \notag\\
    &= (d\Delta+\delta)T_k\notag \\
    & = \Big(\sum_{i=1}^{H-k}(1-\rho)^i \Big)(d\Delta+\delta) - (1-\rho)(1-d\Delta-\delta)(d\Delta+\delta)-(1-\rho)^2(1-d\Delta-\delta)^2(d\Delta+\delta)\notag\\
    & \qquad - \cdots - (1-\rho)^{H-k}(1-d\Delta-\delta)^{H-k}(d\Delta+\delta).
\end{align}
Divide both side of equation \eqref{eq:take difference} by $(d\Delta+\delta)$ and then apply the formula for the sum of a geometric series, we know  $T_k$ has the following closed form expression
\begin{align*}
    T_k = \Big(\sum_{i=1}^{H-k}(1-\rho)^i\Big) - \frac{(1-\rho)(1-d\Delta-\delta)(1-(1-\rho)^{H-k}(1-d\Delta-\delta)^{H-k})}{1-(1-\rho)(1-d\Delta-\delta)}.
\end{align*}
Then, for any $h\leq H/2$, we have the following bound on the term I of \eqref{eq:T_1-S_1},
\begin{align}
    &\sum_{i=1}^{H-h}(1-\rho)^i - T_{h+1} \notag\\
    &= \sum_{i=1}^{H-h}(1-\rho)^i - \sum_{i=1}^{H-h-1}(1-\rho)^i + \frac{(1-\rho)(1-d\Delta-\delta)(1-(1-\rho)^{H-h-1}(1-d\Delta-\delta)^{H-h-1})}{1-(1-\rho)(1-d\Delta-\delta)}\notag \\
    &= (1-\rho)^{H-h} + \frac{(1-\rho)(1-d\Delta-\delta)(1-(1-\rho)^{H-h-1}(1-d\Delta-\delta)^{H-h-1})}{1-(1-\rho)(1-d\Delta-\delta)}\notag \\
    & = (1-\rho)^{H-h} + (1-\rho)(1-d\Delta-\delta) + \cdots + (1-\rho)^{H-h-1}(1-d\Delta-\delta)^{H-h-1} \notag \\
    & \geq (1-d\Delta-\delta)^H\big( (1-\rho) + \cdots + (1-\rho)^{H-h-1}+ (1-\rho)^{H-h} \big)\label{eq:(1-2/H)^H} \\
    &\geq \Big(1-\frac{2}{H} \Big)^H\big( (1-\rho) + \cdots + (1-\rho)^{H-h-1}+ (1-\rho)^{H-h} \big)\notag \\
    &\geq \frac{1}{12}\sum_{i=1}^{H-h}(1-\rho)^i\label{eq:1/12}
    ,
\end{align}
where \eqref{eq:(1-2/H)^H} holds due to $3d\Delta\leq \delta=1/H$ and \eqref{eq:1/12} holds due to $H\geq 6$. Next, we carefully bound the LHS of \eqref{eq:1/12} with respect to $\rho$. For any $h\leq H/2$ and $\rho\in(0,3/4]$, we have
\begin{align*}
     \frac{1}{12}\sum_{i=1}^{H-h}(1-\rho)^i&\geq \frac{1}{12}\frac{(1-\rho)(1-(1-\rho)^{H/2})}{\rho}\geq \frac{1}{50}\frac{1-(1-\rho)^{H/2}}{\rho}.
\end{align*}
Given the fact that 
\begin{align*}
    \frac{1-(1-\rho)^{H/2}}{\rho} = \Theta\Big(\min\Big(H, \frac{1}{\rho}\Big)\Big),
\end{align*}
there exist a constant $c>0$, such that 
\begin{align*}
    \frac{1-(1-\rho)^{H/2}}{\rho} \geq c\cdot \min\Big(H, \frac{1}{\rho}\Big).
\end{align*}
Then we have
\begin{align}
\label{eq:bound I}
    \sum_{i=1}^{H-h}(1-\rho)^i - T_{h+1} \geq c'\cdot \min\Big(H, \frac{1}{\rho}\Big),
\end{align}
where $c'=c/50$. Moreover, with the choice of parameter $3d\Delta\leq \delta, \delta=1/H$, and $H\geq 6$, we have
\begin{align}
\label{eq:bound II}
    \prod_{j=1}^{h-1}(1-o_j-\delta)\geq (1-4\delta/3)^H\geq 1/3.
\end{align}
Therefore, by \eqref{eq:T_1-S_1}, \eqref{eq:bound I} and \eqref{eq:bound II},
we have 
\begin{align*}
    V_1^{\star,\rho}(x_1) - V_1^{\pi,\rho}(x_1) &=T_1-S_1 \\ 
    & \geq c''\cdot \min\{H, 1/\rho\}\cdot \sum_{h=1}^{H/2}(d\Delta-o_h)\\
    & = c''\cdot \min\{H, 1/\rho\}\cdot \sum_{h=1}^{H/2} \Big(\max_{a\in\cA}\la \bmu_h, \ba\ra - \la\bmu_h,\bar{\ba}_h^\pi \ra\Big),
\end{align*}
where $c''=c'/3$.
This completes the proof.
\end{proof}

\subsection{Proof of \Cref{thm:lower_bound}}
\label{sec:proof of lower bound th}

Next, we present an existing result on lower bounding the regret of linear bandits induced by \Cref{lemma:lower bound - lemma 1}. This result is useful in deriving the lower bound in \Cref{thm:lower_bound}. 
\begin{lemma}\citep[Lemma 25]{zhou2021nearly}
\label{lemma:lower bound - lemma 2}
    Fix a positive real $0\leq \delta \leq 1/3$, and positive integers $K,d$ and assume that $K\geq d^2/(2\delta)$. Let $\Delta=\sqrt{\delta /K}/(4\sqrt{2})$ and consider the linear bandit problems $\cL_{\bmu}$ parameterized with a parameter vector $\bmu\in\{-\Delta, \Delta\}^d$ and action set $\cA = \{-1,1\}^d$ so that the reward distribution for taking action $\ba\in\cA$ is a Bernoulli distribution $\text{Bernoulli}(\delta+\la \bmu,\ba\ra)$. Then for any bandit algorithm $\cB$, there exists a $\bmu^\star\in\{-\Delta, \Delta\}^d$ such that the expected pseudo-regret of $\cB$ over first $K$ steps on bandit $\cL_{\bmu^\star}$ is lower bounded as follows: 
    \begin{align*}
        \EE_{\bmu^\star}Regret(K) \geq \frac{d\sqrt{K\delta}}{8\sqrt{2}}.
    \end{align*}
    Note that the expectation is with respect to a distribution that depends both on $\cB$ and $\mu^\star$, but since $\cB$ is fixed, this dependence is hidden.
\end{lemma}

Now we are ready to prove the lower bound in \Cref{thm:lower_bound}.
\begin{proof}[Proof of \Cref{thm:lower_bound}]
    By \Cref{lemma:lower bound - lemma 1}, we have
    \begin{align*}
        \EE_{\bxi}\text{AveSubopt}(M_{\bxi},K) &= \frac{1}{K}\EE_{\bxi}\Big[\sum_{k=1}^K[V_1^{\star, \rho}(x_1)-V_1^{\pi, \rho}(x_1)] \Big]\\
        &\geq c\cdot \frac{\min\{H,1/\rho\}}{K} \sum_{h=1}^{H/2}\EE_{\bxi}\bigg[\sum_{k=1}^K \Big(\max_{a\in\cA}\la \bxi_h,\ba \ra - \la \bxi_h, \bar{\ba}_h^{\pi_k}\ra \Big) \bigg].
    \end{align*}
     Note that the learning process is conducted on the nominal environment, which is exactly the MDP in \citet{zhou2021nearly}, thus the rest proof of \Cref{thm:DRLSVIUCB} follows the argument in the proof of Theorem 8 in \citet{zhou2021nearly}. In particular, define $\bxi^{-h} = (\bxi_1, \cdots, \bxi_{h-1}, \bxi_{h+1}, \cdots, \bxi_H)$, then every MDP policy $\pi$ induces a bandit algorithm $\cB_{\pi, h, \bxi^{-h}}$ for the linear bandit of \Cref{lemma:lower bound - lemma 2}. Moreover, our choice of parameters in \eqref{eq: lower bound choice of K and H} satisfy the requirement of \Cref{lemma:lower bound - lemma 2}.
     Denote the regret of this bandit problem on $\cL_{\bxi}$ as $\text{BanditRegret}(\cB_{\pi, h, \bxi^{-h}}, \bxi_h)$, then we have
     \begin{align*}
         \sup_{\bxi}\EE_{\bxi}\text{AveSubopt}(M_{\bxi},K)&\geq \sup_{\bxi}c\cdot\frac{\min\{H,1/\rho\}}{K}\sum_{h=1}^{H/2}\text{BanditRegret}(\cB_{\pi, h, \bxi^{-h}}, \bxi_h)\\
         &\geq \sup_{\bxi}c\cdot\frac{\min\{H,1/\rho\}}{K}\sum_{h=1}^{H/2}\inf_{\tilde{\bxi}^{-h}}\text{BanditRegret}(\cB_{\pi, h, \tilde{\bxi}^{-h}}, \bxi_h)\\
         &=c\cdot\frac{\min\{H,1/\rho\}}{K}\sum_{h=1}^{H/2}\sup_{\bxi}\inf_{\tilde{\bxi}^{-h}}\text{BanditRegret}(\cB_{\pi, h, \tilde{\bxi}^{-h}}, \bxi_h)\\
         &\geq c\cdot\frac{\min\{H,1/\rho\}dH\sqrt{K\delta}}{16\sqrt{2}\cdot K}\\
         &=\frac{c}{16\sqrt{2}}\cdot\frac{d\sqrt{H}\cdot \min\{H,1/\rho\}}{\sqrt{K}}.
     \end{align*}
    This completes the proof.
\end{proof}

\section{Auxiliary Lemmas}
In this section, we present some standard technical results in the literature that our proofs are built on.

\begin{proposition}
\label{prop:strong duality for TV}
(Strong duality for TV \citep[Lemma 4]{shi2023curious}). Given any probability measure $\mu^0$ over $\cS$, a fixed uncertainty level $\rho$, the uncertainty set $ \cU^{\rho}(\mu^0) =\{\mu: \mu\in \Delta(\cS), D_{TV}(\mu||\mu^0)\leq \rho\}$, and any function $V:\cS \rightarrow [0,H]$, we obtain 
\begin{align}
\label{eq:duality}
    {\textstyle\inf_{\mu\in\cU^{\rho}(\mu^0)}\EE_{s\sim\mu}V(s) = \max_{\alpha \in [V_{\min}, V_{\max}]}\big\{\EE_{s\sim \mu^0}[V(s)]_{\alpha} -\rho\big(\alpha - \min_{s'}[V(s')]_{\alpha}\big) \big\}},
\end{align}
where $[V(s)]_{\alpha}=\min\{V(s), \alpha\}$, $V_{\min}=\min_{s}V(s)$ and $V_{\max}=\max_{s}V(s)$. Notably, the range of $\alpha$ can be relaxed to $[0,H]$ without impacting the optimization. 
\end{proposition}

\begin{lemma}\citep[Lemma 12]{abbasi2011improved} 
\label{lem12_abbasi}
Let $\Ab$, $\Bb$ and $\Cb$ be positive semi-definite matrices such that $\Ab=\Bb+\Cb$. Then we have that
\begin{align*}
    \sup _{\xb \neq 0} \frac{\xb^{\top} \Ab \xb}{\xb^{\top} \Bb \xb} \leq \frac{\text{det}(\Ab)}{\text{det}(\Bb)} .   
\end{align*}
\end{lemma}

\begin{lemma} \citep[Confidence Ellipsoid, Theorem 2]{abbasi2011improved}
\label{lem:original_confidence_ellipsoid}
Let $\{\mathcal{G}_k\}_{k=1}^{\infty}$ be a filtration, and $\{\xb_k, \eta_k\}_{k \geq 1}$ be a stochastic process such that $\xb_k \in \mathbb{R}^d$ is $\mathcal{G}_k$-measurable and $\eta_k \in \mathbb{R}$ is $\mathcal{G}_{k+1}$-measurable. Let $L$, $\sigma, \bSigma, \epsilon>0, \bmu^* \in \mathbb{R}^d$. For $k \geq 1$, let $y_k=\langle\bmu^*, \xb_k\rangle+\eta_k$ and suppose that $\eta_k, \xb_k$ also satisfy
\begin{align*}
    \mathbb{E}[\eta_k \mid \mathcal{G}_k]=0,|\eta_k| \leq R,\|\xb_k\|_2 \leq L .
\end{align*}
For $k \geq 1$, let $\Zb_k=\lambda \Ib+\sum_{i=1}^k \xb_i \xb_i^{\top}, \bb_k=\sum_{i=1}^k y_i \xb_i, \bmu_k=\Zb_k^{-1} \bb_k$, and
\begin{align*}
    \beta_k=R \sqrt{d \log \bigg(1 + \frac{k L^2}{d\lambda}\bigg) + 2 \log \frac{1}{\delta}} .
\end{align*}
Then, for any $0<\delta<1$, we have with probability at least $1-\delta$ that,
\begin{align*}
    \forall k \geq 1, \bigg\|\sum_{i=1}^k \xb_i \eta_i\bigg\|_{\Zb_k^{-1}} \leq \beta_k,\|\bmu_k-\bmu^*\|_{\Zb_k} \leq \beta_k+\sqrt{\lambda}\|\bmu^*\|_2.
\end{align*}
\end{lemma}

\begin{lemma}\citep[Lemma D.1]{jin2020provably}
\label{lemma:self-normalize}
Let $\bLambda_t=\lambda \Ib + \sum_{i=1}^t\bphi_i\bphi_i^{\top}$, where $\bphi_i\in\RR^d$ and $\lambda > 0$. Then we have
\begin{align*}
    \sum_{i=1}^t\bphi_i^{\top}(\bLambda_t)^{-1}\bphi_i \leq d.
\end{align*}
\end{lemma}

\begin{lemma} \citep[Lemma D.5]{ishfaq2021randomized} \label{lem:inequal_eigen_matrix}
Let $\Ab\in \mathbb{R}^{d \times d}$ be a positive definite matrix where its largest eigenvalue $\lambda_{\max}(\Ab) \leq \lambda$. Let $\xb_1, ..., \xb_k$ be $k$ vectors in $\mathbb{R}^d$. Then it holds that
\begin{align*}
    \bigg\|\Ab \sum_{i=1}^k \xb_i\bigg\| \leq \sqrt{\lambda k} \bigg(\sum_{i=1}^k \|\xb_i\|_\Ab^2 \bigg)^{1/2}.
\end{align*}
\end{lemma}

\begin{lemma}\citep[Covering number of Euclidean ball]{vershynin2018high} 
\label{lem:vershynin2018high}
For any $\varepsilon>0$, $\mathcal{N}_{\varepsilon}$, the $\varepsilon$-covering number of the Euclidean ball of radius $B>0$ in $\mathbb{R}^d$ satisfies
\begin{align*}
    \mathcal{N}_{\varepsilon} \leq\bigg(1+\frac{2 B}{\varepsilon}\bigg)^d \leq\bigg(\frac{3 B}{\varepsilon}\bigg)^d.
\end{align*}
\end{lemma}

\begin{lemma}\citep[Covering number of an interval]{vershynin2018high}
\label{lemma:Covering number of an interval}
Denote the $\epsilon$-covering number of the closed interval $[a,b]$ for some real number $b>a$ with respect to the distance metric $d(\alpha_1, \alpha_2)=|\alpha_1-\alpha_2|$ as $\cN_{\epsilon}([a,b])$. Then we have $\cN_{\epsilon}([a,b])\leq 3(b-a)/\epsilon$.
\end{lemma}

\begin{lemma}\citep[Theorem 4.3]{zhou2022computationally}
\label{lem:bernstein_concentration}
Let $\{\mathcal{G}_k\}_{k=1}^{\infty}$ be a filtration, and $\{\mathrm{x}_k, \eta_k\}_{k \geq 1}$ be a stochastic process such that $\mathbf{x}_k \in \mathbb{R}^d$ is $\mathcal{G}_k$-measurable and $\eta_k \in \mathbb{R}$ is $\mathcal{G}_{k+1}$-measurable. Let $L, \sigma>0, \mu^* \in \mathbb{R}^d$. For $k \geq 1$, let $y_k=\langle\boldsymbol{\mu}^*, \mathbf{x}_k\rangle+\eta_k$ and suppose that $\eta_k, \mathbf{x}_k$ also satisfy
\begin{align*}
    \mathbb{E}[\eta_k \mid \mathcal{G}_k]=0, \mathbb{E}[\eta_k^2 \mid \mathcal{G}_k] \leq \sigma^2,|\eta_k| \leq R,\|\xb_k\|_2 \leq L .
\end{align*}
For $k \geq 1$, let $\beta_k=\widetilde{O}\big(\sigma \sqrt{d}+\max _{1 \leq i \leq k}|\eta_i| \min \big\{1,\|\mathbf{x}_i\|_{\mathbf{Z}_{i-1}^{-1}}\big\}\big)$ and $\mathbf{Z}_k=\lambda \mathbf{I}+\sum_{i=1}^k \mathbf{x}_i \mathbf{x}_i^{\top}$, $\mathbf{b}_k=$ $\sum_{i=1}^k y_i \mathbf{x}_i$, $\boldsymbol{\mu}_k=\mathbf{Z}_k^{-1} \mathbf{b}_k$. Then, for any $0<\delta<$ 1 , with probability at least $1-\delta$, for all $k \in[K]$, we have
\begin{align*}
    \bigg\|\sum_{i=1}^k \mathbf{x}_i \eta_i\bigg\|_{\mathbf{Z}_k^{-1}} \leq \beta_k,\|\boldsymbol{\mu}_k-\boldsymbol{\mu}^*\|_{\mathbf{Z}_k} \leq \beta_k+\sqrt{\lambda}\|\boldsymbol{\mu}^*\|_2.
\end{align*}
\end{lemma}

\begin{lemma} \citep[Lemma D.4]{jin2020provably} \label{lem:jin_D.4}
Let $\{s_i\}_{i=1}^{\infty}$ be a stochastic process on state space $\mathcal{S}$ with corresponding filtration $\{\mathcal{F}_i\}_{i=1}^{\infty}$. Let $\{\bphi_i\}_{i=1}^{\infty}$ be an $\mathbb{R}^d$-valued stochastic process where $\bphi_i \in \mathcal{F}_{i-1}$, and $\|\bphi_i\| \leq 1$. Let $\bLambda_k=\lambda \Ib+\sum_{i=1}^k \bphi_i \bphi_i^{\top}$. Then for any $\delta>0$, with probability at least $1-\delta$, for all $k \geq 0$, and any $V \in \mathcal{V}$ with $\sup _{s \in \mathcal{S}}|V(s)| \leq H$, we have
\begin{align*}
    \bigg\|\sum_{i=1}^k \bphi_i\{V(s_i)-\mathbb{E}[V(s_i) \mid \mathcal{F}_{i-1}]\}\bigg\|_{\bLambda_k^{-1}}^2 \leq 4 H^2\bigg[\frac{d}{2} \log \bigg(\frac{k+\lambda}{\lambda}\bigg)+\log \frac{\mathcal{N}_{\varepsilon}}{\delta}\bigg]+\frac{8 k^2 \varepsilon^2}{\lambda},
\end{align*}
where $\mathcal{N}_{\varepsilon}$ is the $\varepsilon$-covering number of $\mathcal{V}$ with respect to the distance $\text{dist}(V, V^{\prime})=\sup _{s \in \mathcal{S}} |V(s)-$ $V^{\prime}(s)|$.
\end{lemma}

\begin{lemma} \citep[Lemma 5.1 (Range Shrinkage)]{liu2024minimax}
\label{lem:range_shrinkage}
    For any $(\rho,\pi, h) \in (0,1] \times \Pi \times [H]$, we have $\max_{s\in\cS}V_h^{\pi, \rho}(s) - \min_{s\in\cS}V_h^{\pi, \rho}(s) \leq  (1-(1-\rho)^{H-h+1})/\rho$.
\end{lemma}

\bibliography{reference}

\begin{thebibliography}{61}
\expandafter\ifx\csname natexlab\endcsname\relax\def\natexlab#1{#1}\fi
\expandafter\ifx\csname url\endcsname\relax
  \def\url#1{\texttt{#1}}\fi
\expandafter\ifx\csname urlprefix\endcsname\relax\def\urlprefix{URL }\fi

\bibitem[{Abbasi-Yadkori et~al.(2011)Abbasi-Yadkori, P{\'a}l and Szepesv{\'a}ri}]{abbasi2011improved}
\textsc{Abbasi-Yadkori, Y.}, \textsc{P{\'a}l, D.} and \textsc{Szepesv{\'a}ri, C.} (2011).
\newblock Improved algorithms for linear stochastic bandits.
\newblock \textit{Advances in neural information processing systems} \textbf{24}.

\bibitem[{Azar et~al.(2017)Azar, Osband and Munos}]{azar2017minimax}
\textsc{Azar, M.~G.}, \textsc{Osband, I.} and \textsc{Munos, R.} (2017).
\newblock Minimax regret bounds for reinforcement learning.
\newblock In \textit{International Conference on Machine Learning}. PMLR.

\bibitem[{Bai et~al.(2019)Bai, Xie, Jiang and Wang}]{bai2019provably}
\textsc{Bai, Y.}, \textsc{Xie, T.}, \textsc{Jiang, N.} and \textsc{Wang, Y.-X.} (2019).
\newblock Provably efficient q-learning with low switching cost.
\newblock \textit{Advances in Neural Information Processing Systems} \textbf{32}.

\bibitem[{Blanchet et~al.(2023)Blanchet, Lu, Zhang and Zhong}]{blanchet2023double}
\textsc{Blanchet, J.}, \textsc{Lu, M.}, \textsc{Zhang, T.} and \textsc{Zhong, H.} (2023).
\newblock Double pessimism is provably efficient for distributionally robust offline reinforcement learning: Generic algorithm and robust partial coverage.
\newblock \textit{arXiv preprint arXiv:2305.09659} .

\bibitem[{Chen and Jiang(2019)}]{chen2019information}
\textsc{Chen, J.} and \textsc{Jiang, N.} (2019).
\newblock Information-theoretic considerations in batch reinforcement learning.
\newblock In \textit{International Conference on Machine Learning}. PMLR.

\bibitem[{Dong et~al.(2022)Dong, Li, Wang and Zhang}]{dong2022online}
\textsc{Dong, J.}, \textsc{Li, J.}, \textsc{Wang, B.} and \textsc{Zhang, J.} (2022).
\newblock Online policy optimization for robust mdp.
\newblock \textit{arXiv preprint arXiv:2209.13841} .

\bibitem[{Eysenbach et~al.(2020)Eysenbach, Asawa, Chaudhari, Levine and Salakhutdinov}]{eysenbach2020off}
\textsc{Eysenbach, B.}, \textsc{Asawa, S.}, \textsc{Chaudhari, S.}, \textsc{Levine, S.} and \textsc{Salakhutdinov, R.} (2020).
\newblock Off-dynamics reinforcement learning: Training for transfer with domain classifiers.
\newblock \textit{arXiv preprint arXiv:2006.13916} .

\bibitem[{Farebrother et~al.(2018)Farebrother, Machado and Bowling}]{farebrother2018generalization}
\textsc{Farebrother, J.}, \textsc{Machado, M.~C.} and \textsc{Bowling, M.} (2018).
\newblock Generalization and regularization in dqn.
\newblock \textit{arXiv preprint arXiv:1810.00123} .

\bibitem[{Goyal and Grand-Clement(2023)}]{goyal2023robust}
\textsc{Goyal, V.} and \textsc{Grand-Clement, J.} (2023).
\newblock Robust markov decision processes: Beyond rectangularity.
\newblock \textit{Mathematics of Operations Research} \textbf{48} 203--226.

\bibitem[{He et~al.(2023)He, Zhao, Zhou and Gu}]{he2023nearly}
\textsc{He, J.}, \textsc{Zhao, H.}, \textsc{Zhou, D.} and \textsc{Gu, Q.} (2023).
\newblock Nearly minimax optimal reinforcement learning for linear markov decision processes.
\newblock In \textit{International Conference on Machine Learning}. PMLR.

\bibitem[{He et~al.(2021)He, Zhou and Gu}]{he2021logarithmic}
\textsc{He, J.}, \textsc{Zhou, D.} and \textsc{Gu, Q.} (2021).
\newblock Logarithmic regret for reinforcement learning with linear function approximation.
\newblock In \textit{International Conference on Machine Learning}. PMLR.

\bibitem[{Hsu et~al.(2024)Hsu, Wang, Pajic and Xu}]{hsu2024randomized}
\textsc{Hsu, H.-L.}, \textsc{Wang, W.}, \textsc{Pajic, M.} and \textsc{Xu, P.} (2024).
\newblock Randomized exploration in cooperative multi-agent reinforcement learning.
\newblock \textit{arXiv preprint arXiv:2404.10728} .

\bibitem[{Hu et~al.(2023)Hu, Chen and Huang}]{hu2023nearly}
\textsc{Hu, P.}, \textsc{Chen, Y.} and \textsc{Huang, L.} (2023).
\newblock Nearly minimax optimal reinforcement learning with linear function approximation.

\bibitem[{Ishfaq et~al.(2021)Ishfaq, Cui, Nguyen, Ayoub, Yang, Wang, Precup and Yang}]{ishfaq2021randomized}
\textsc{Ishfaq, H.}, \textsc{Cui, Q.}, \textsc{Nguyen, V.}, \textsc{Ayoub, A.}, \textsc{Yang, Z.}, \textsc{Wang, Z.}, \textsc{Precup, D.} and \textsc{Yang, L.} (2021).
\newblock Randomized exploration in reinforcement learning with general value function approximation.
\newblock In \textit{International Conference on Machine Learning}. PMLR.

\bibitem[{Ishfaq et~al.(2023)Ishfaq, Lan, Xu, Mahmood, Precup, Anandkumar and Azizzadenesheli}]{ishfaq2023provable}
\textsc{Ishfaq, H.}, \textsc{Lan, Q.}, \textsc{Xu, P.}, \textsc{Mahmood, A.~R.}, \textsc{Precup, D.}, \textsc{Anandkumar, A.} and \textsc{Azizzadenesheli, K.} (2023).
\newblock Provable and practical: Efficient exploration in reinforcement learning via langevin monte carlo.
\newblock \textit{arXiv preprint arXiv:2305.18246} .

\bibitem[{Iyengar(2005)}]{iyengar2005robust}
\textsc{Iyengar, G.~N.} (2005).
\newblock Robust dynamic programming.
\newblock \textit{Mathematics of Operations Research} \textbf{30} 257--280.

\bibitem[{Jiang et~al.(2021)Jiang, Zhang, Ho, Bai, Liu, Levine and Tan}]{jiang2021simgan}
\textsc{Jiang, Y.}, \textsc{Zhang, T.}, \textsc{Ho, D.}, \textsc{Bai, Y.}, \textsc{Liu, C.~K.}, \textsc{Levine, S.} and \textsc{Tan, J.} (2021).
\newblock Simgan: Hybrid simulator identification for domain adaptation via adversarial reinforcement learning.
\newblock In \textit{2021 IEEE International Conference on Robotics and Automation (ICRA)}. IEEE.

\bibitem[{Jin et~al.(2018)Jin, Allen-Zhu, Bubeck and Jordan}]{jin2018q}
\textsc{Jin, C.}, \textsc{Allen-Zhu, Z.}, \textsc{Bubeck, S.} and \textsc{Jordan, M.~I.} (2018).
\newblock Is q-learning provably efficient?
\newblock \textit{Advances in neural information processing systems} \textbf{31}.

\bibitem[{Jin et~al.(2020)Jin, Yang, Wang and Jordan}]{jin2020provably}
\textsc{Jin, C.}, \textsc{Yang, Z.}, \textsc{Wang, Z.} and \textsc{Jordan, M.~I.} (2020).
\newblock Provably efficient reinforcement learning with linear function approximation.
\newblock In \textit{Conference on Learning Theory}. PMLR.

\bibitem[{Kim et~al.(2022)Kim, Yang and Jun}]{kim2022improved}
\textsc{Kim, Y.}, \textsc{Yang, I.} and \textsc{Jun, K.-S.} (2022).
\newblock Improved regret analysis for variance-adaptive linear bandits and horizon-free linear mixture mdps.
\newblock \textit{Advances in Neural Information Processing Systems} \textbf{35} 1060--1072.

\bibitem[{Laber et~al.(2018)Laber, Meyer, Reich, Pacifici, Collazo and Drake}]{laber2018optimal}
\textsc{Laber, E.~B.}, \textsc{Meyer, N.~J.}, \textsc{Reich, B.~J.}, \textsc{Pacifici, K.}, \textsc{Collazo, J.~A.} and \textsc{Drake, J.~M.} (2018).
\newblock Optimal treatment allocations in space and time for on-line control of an emerging infectious disease.
\newblock \textit{Journal of the Royal Statistical Society Series C: Applied Statistics} \textbf{67} 743--789.

\bibitem[{Liu et~al.(2023)Liu, Clifton, Laber, Drake and Fang}]{liu2023deep}
\textsc{Liu, Z.}, \textsc{Clifton, J.}, \textsc{Laber, E.~B.}, \textsc{Drake, J.} and \textsc{Fang, E.~X.} (2023).
\newblock Deep spatial q-learning for infectious disease control.
\newblock \textit{Journal of Agricultural, Biological and Environmental Statistics}  1--25.

\bibitem[{Liu and Xu(2024{\natexlab{a}})}]{liu2024distributionally}
\textsc{Liu, Z.} and \textsc{Xu, P.} (2024{\natexlab{a}}).
\newblock Distributionally robust off-dynamics reinforcement learning: Provable efficiency with linear function approximation.
\newblock \textit{arXiv preprint arXiv:2402.15399} .

\bibitem[{Liu and Xu(2024{\natexlab{b}})}]{liu2024minimax}
\textsc{Liu, Z.} and \textsc{Xu, P.} (2024{\natexlab{b}}).
\newblock Minimax optimal and computationally efficient algorithms for distributionally robust offline reinforcement learning.
\newblock \textit{arXiv preprint arXiv:2403.09621} .

\bibitem[{Lu et~al.(2024)Lu, Zhong, Zhang and Blanchet}]{lu2024distributionally}
\textsc{Lu, M.}, \textsc{Zhong, H.}, \textsc{Zhang, T.} and \textsc{Blanchet, J.} (2024).
\newblock Distributionally robust reinforcement learning with interactive data collection: Fundamental hardness and near-optimal algorithm.
\newblock \textit{arXiv preprint arXiv:2404.03578} .

\bibitem[{Ma et~al.(2022)Ma, Liang, Xia, Zhang, Blanchet, Liu, Zhao and Zhou}]{ma2022distributionally}
\textsc{Ma, X.}, \textsc{Liang, Z.}, \textsc{Xia, L.}, \textsc{Zhang, J.}, \textsc{Blanchet, J.}, \textsc{Liu, M.}, \textsc{Zhao, Q.} and \textsc{Zhou, Z.} (2022).
\newblock Distributionally robust offline reinforcement learning with linear function approximation.
\newblock \textit{arXiv preprint arXiv:2209.06620} .

\bibitem[{Mannor et~al.(2016)Mannor, Mebel and Xu}]{mannor2016robust}
\textsc{Mannor, S.}, \textsc{Mebel, O.} and \textsc{Xu, H.} (2016).
\newblock Robust mdps with k-rectangular uncertainty.
\newblock \textit{Mathematics of Operations Research} \textbf{41} 1484--1509.

\bibitem[{Modi et~al.(2020)Modi, Jiang, Tewari and Singh}]{modi2020sample}
\textsc{Modi, A.}, \textsc{Jiang, N.}, \textsc{Tewari, A.} and \textsc{Singh, S.} (2020).
\newblock Sample complexity of reinforcement learning using linearly combined model ensembles.
\newblock In \textit{International Conference on Artificial Intelligence and Statistics}. PMLR.

\bibitem[{Nilim and El~Ghaoui(2005)}]{nilim2005robust}
\textsc{Nilim, A.} and \textsc{El~Ghaoui, L.} (2005).
\newblock Robust control of markov decision processes with uncertain transition matrices.
\newblock \textit{Operations Research} \textbf{53} 780--798.

\bibitem[{Panaganti and Kalathil(2022)}]{panaganti2022sample}
\textsc{Panaganti, K.} and \textsc{Kalathil, D.} (2022).
\newblock Sample complexity of robust reinforcement learning with a generative model.
\newblock In \textit{International Conference on Artificial Intelligence and Statistics}. PMLR.

\bibitem[{Panaganti et~al.(2024)Panaganti, Wierman and Mazumdar}]{panaganti2024model}
\textsc{Panaganti, K.}, \textsc{Wierman, A.} and \textsc{Mazumdar, E.} (2024).
\newblock Model-free robust $phi$-divergence reinforcement learning using both offline and online data.
\newblock \textit{arXiv preprint arXiv:2405.05468} .

\bibitem[{Panaganti et~al.(2022)Panaganti, Xu, Kalathil and Ghavamzadeh}]{panaganti2022robust}
\textsc{Panaganti, K.}, \textsc{Xu, Z.}, \textsc{Kalathil, D.} and \textsc{Ghavamzadeh, M.} (2022).
\newblock Robust reinforcement learning using offline data.
\newblock \textit{Advances in neural information processing systems} \textbf{35} 32211--32224.

\bibitem[{Peng et~al.(2018)Peng, Andrychowicz, Zaremba and Abbeel}]{peng2018sim}
\textsc{Peng, X.~B.}, \textsc{Andrychowicz, M.}, \textsc{Zaremba, W.} and \textsc{Abbeel, P.} (2018).
\newblock Sim-to-real transfer of robotic control with dynamics randomization.
\newblock In \textit{2018 IEEE international conference on robotics and automation (ICRA)}. IEEE.

\bibitem[{Shen et~al.(2024)Shen, Xu and Zavlanos}]{shen2024wasserstein}
\textsc{Shen, Y.}, \textsc{Xu, P.} and \textsc{Zavlanos, M.} (2024).
\newblock Wasserstein distributionally robust policy evaluation and learning for contextual bandits.
\newblock \textit{Transactions on Machine Learning Research} Featured Certification.
\newline\urlprefix\url{https://openreview.net/forum?id=NmpjDHWIvg}

\bibitem[{Shi and Chi(2022)}]{shi2022distributionally}
\textsc{Shi, L.} and \textsc{Chi, Y.} (2022).
\newblock Distributionally robust model-based offline reinforcement learning with near-optimal sample complexity.
\newblock \textit{arXiv preprint arXiv:2208.05767} .

\bibitem[{Shi et~al.(2023)Shi, Li, Wei, Chen, Geist and Chi}]{shi2023curious}
\textsc{Shi, L.}, \textsc{Li, G.}, \textsc{Wei, Y.}, \textsc{Chen, Y.}, \textsc{Geist, M.} and \textsc{Chi, Y.} (2023).
\newblock The curious price of distributional robustness in reinforcement learning with a generative model.
\newblock \textit{arXiv preprint arXiv:2305.16589} .

\bibitem[{Sutton and Barto(2018)}]{sutton2018reinforcement}
\textsc{Sutton, R.~S.} and \textsc{Barto, A.~G.} (2018).
\newblock \textit{Reinforcement learning: An introduction}.
\newblock MIT press.

\bibitem[{Vershynin(2018)}]{vershynin2018high}
\textsc{Vershynin, R.} (2018).
\newblock \textit{High-dimensional probability: An introduction with applications in data science}, vol.~47.
\newblock Cambridge university press.

\bibitem[{Wagenmaker et~al.(2022)Wagenmaker, Chen, Simchowitz, Du and Jamieson}]{wagenmaker2022reward}
\textsc{Wagenmaker, A.~J.}, \textsc{Chen, Y.}, \textsc{Simchowitz, M.}, \textsc{Du, S.} and \textsc{Jamieson, K.} (2022).
\newblock Reward-free rl is no harder than reward-aware rl in linear markov decision processes.
\newblock In \textit{International Conference on Machine Learning}. PMLR.

\bibitem[{Wang et~al.(2024)Wang, Shi and Chi}]{wang2024sample}
\textsc{Wang, H.}, \textsc{Shi, L.} and \textsc{Chi, Y.} (2024).
\newblock Sample complexity of offline distributionally robust linear markov decision processes.
\newblock \textit{arXiv preprint arXiv:2403.12946} .

\bibitem[{Wang et~al.(2020{\natexlab{a}})Wang, Du, Yang and Salakhutdinov}]{wang2020reward}
\textsc{Wang, R.}, \textsc{Du, S.~S.}, \textsc{Yang, L.} and \textsc{Salakhutdinov, R.~R.} (2020{\natexlab{a}}).
\newblock On reward-free reinforcement learning with linear function approximation.
\newblock \textit{Advances in neural information processing systems} \textbf{33} 17816--17826.

\bibitem[{Wang et~al.(2020{\natexlab{b}})Wang, Foster and Kakade}]{wang2020statistical}
\textsc{Wang, R.}, \textsc{Foster, D.~P.} and \textsc{Kakade, S.~M.} (2020{\natexlab{b}}).
\newblock What are the statistical limits of offline rl with linear function approximation?
\newblock \textit{arXiv preprint arXiv:2010.11895} .

\bibitem[{Wang et~al.(2021)Wang, Zhou and Gu}]{wang2021provably}
\textsc{Wang, T.}, \textsc{Zhou, D.} and \textsc{Gu, Q.} (2021).
\newblock Provably efficient reinforcement learning with linear function approximation under adaptivity constraints.
\newblock \textit{Advances in Neural Information Processing Systems} \textbf{34} 13524--13536.

\bibitem[{Wiesemann et~al.(2013)Wiesemann, Kuhn and Rustem}]{wiesemann2013robust}
\textsc{Wiesemann, W.}, \textsc{Kuhn, D.} and \textsc{Rustem, B.} (2013).
\newblock Robust markov decision processes.
\newblock \textit{Mathematics of Operations Research} \textbf{38} 153--183.

\bibitem[{Xie et~al.(2021)Xie, Cheng, Jiang, Mineiro and Agarwal}]{xie2021bellman}
\textsc{Xie, T.}, \textsc{Cheng, C.-A.}, \textsc{Jiang, N.}, \textsc{Mineiro, P.} and \textsc{Agarwal, A.} (2021).
\newblock Bellman-consistent pessimism for offline reinforcement learning.
\newblock \textit{Advances in neural information processing systems} \textbf{34} 6683--6694.

\bibitem[{Xu and Mannor(2006)}]{xu2006robustness}
\textsc{Xu, H.} and \textsc{Mannor, S.} (2006).
\newblock The robustness-performance tradeoff in markov decision processes.
\newblock \textit{Advances in Neural Information Processing Systems} \textbf{19}.

\bibitem[{Xu et~al.(2023)Xu, Panaganti and Kalathil}]{xu2023improved}
\textsc{Xu, Z.}, \textsc{Panaganti, K.} and \textsc{Kalathil, D.} (2023).
\newblock Improved sample complexity bounds for distributionally robust reinforcement learning.
\newblock In \textit{International Conference on Artificial Intelligence and Statistics}. PMLR.

\bibitem[{Yang and Wang(2020)}]{yang2020reinforcement}
\textsc{Yang, L.} and \textsc{Wang, M.} (2020).
\newblock Reinforcement learning in feature space: Matrix bandit, kernels, and regret bound.
\newblock In \textit{International Conference on Machine Learning}. PMLR.

\bibitem[{Yang et~al.(2023{\natexlab{a}})Yang, Wang, Kozuno, Jordan and Zhang}]{yang2023avoiding}
\textsc{Yang, W.}, \textsc{Wang, H.}, \textsc{Kozuno, T.}, \textsc{Jordan, S.~M.} and \textsc{Zhang, Z.} (2023{\natexlab{a}}).
\newblock Avoiding model estimation in robust markov decision processes with a generative model.
\newblock \textit{arXiv preprint arXiv:2302.01248} .

\bibitem[{Yang et~al.(2022)Yang, Zhang and Zhang}]{yang2022toward}
\textsc{Yang, W.}, \textsc{Zhang, L.} and \textsc{Zhang, Z.} (2022).
\newblock Toward theoretical understandings of robust markov decision processes: Sample complexity and asymptotics.
\newblock \textit{The Annals of Statistics} \textbf{50} 3223--3248.

\bibitem[{Yang et~al.(2023{\natexlab{b}})Yang, Guo, Xu, Liu and Anandkumar}]{yang2023distributionally}
\textsc{Yang, Z.}, \textsc{Guo, Y.}, \textsc{Xu, P.}, \textsc{Liu, A.} and \textsc{Anandkumar, A.} (2023{\natexlab{b}}).
\newblock Distributionally robust policy gradient for offline contextual bandits.
\newblock In \textit{International Conference on Artificial Intelligence and Statistics}. PMLR.

\bibitem[{Yu and Xu(2015)}]{yu2015distributionally}
\textsc{Yu, P.} and \textsc{Xu, H.} (2015).
\newblock Distributionally robust counterpart in markov decision processes.
\newblock \textit{IEEE Transactions on Automatic Control} \textbf{61} 2538--2543.

\bibitem[{Zanette et~al.(2020)Zanette, Brandfonbrener, Brunskill, Pirotta and Lazaric}]{zanette2020frequentist}
\textsc{Zanette, A.}, \textsc{Brandfonbrener, D.}, \textsc{Brunskill, E.}, \textsc{Pirotta, M.} and \textsc{Lazaric, A.} (2020).
\newblock Frequentist regret bounds for randomized least-squares value iteration.
\newblock In \textit{International Conference on Artificial Intelligence and Statistics}. PMLR.

\bibitem[{Zhang et~al.(2021{\natexlab{a}})Zhang, Chen, Boning and Hsieh}]{zhang2021robust}
\textsc{Zhang, H.}, \textsc{Chen, H.}, \textsc{Boning, D.} and \textsc{Hsieh, C.-J.} (2021{\natexlab{a}}).
\newblock Robust reinforcement learning on state observations with learned optimal adversary.
\newblock \textit{arXiv preprint arXiv:2101.08452} .

\bibitem[{Zhang et~al.()Zhang, Hu and Li}]{zhangsoft}
\textsc{Zhang, R.}, \textsc{Hu, Y.} and \textsc{Li, N.} (????).
\newblock Soft robust mdps and risk-sensitive mdps: Equivalence, policy gradient, and sample complexity.
\newblock In \textit{The Twelfth International Conference on Learning Representations}.

\bibitem[{Zhang et~al.(2021{\natexlab{b}})Zhang, Yang, Ji and Du}]{zhang2021improved}
\textsc{Zhang, Z.}, \textsc{Yang, J.}, \textsc{Ji, X.} and \textsc{Du, S.~S.} (2021{\natexlab{b}}).
\newblock Improved variance-aware confidence sets for linear bandits and linear mixture mdp.
\newblock \textit{Advances in Neural Information Processing Systems} \textbf{34} 4342--4355.

\bibitem[{Zhao et~al.(2023)Zhao, He, Zhou, Zhang and Gu}]{zhao2023variance}
\textsc{Zhao, H.}, \textsc{He, J.}, \textsc{Zhou, D.}, \textsc{Zhang, T.} and \textsc{Gu, Q.} (2023).
\newblock Variance-dependent regret bounds for linear bandits and reinforcement learning: Adaptivity and computational efficiency.
\newblock \textit{arXiv preprint arXiv:2302.10371} .

\bibitem[{Zhao et~al.(2020)Zhao, Queralta and Westerlund}]{zhao2020sim}
\textsc{Zhao, W.}, \textsc{Queralta, J.~P.} and \textsc{Westerlund, T.} (2020).
\newblock Sim-to-real transfer in deep reinforcement learning for robotics: a survey.
\newblock In \textit{2020 IEEE symposium series on computational intelligence (SSCI)}. IEEE.

\bibitem[{Zhou and Gu(2022)}]{zhou2022computationally}
\textsc{Zhou, D.} and \textsc{Gu, Q.} (2022).
\newblock Computationally efficient horizon-free reinforcement learning for linear mixture mdps.

\bibitem[{Zhou et~al.(2021{\natexlab{a}})Zhou, Gu and Szepesvari}]{zhou2021nearly}
\textsc{Zhou, D.}, \textsc{Gu, Q.} and \textsc{Szepesvari, C.} (2021{\natexlab{a}}).
\newblock Nearly minimax optimal reinforcement learning for linear mixture markov decision processes.
\newblock In \textit{Conference on Learning Theory}. PMLR.

\bibitem[{Zhou et~al.(2021{\natexlab{b}})Zhou, Zhou, Bai, Qiu, Blanchet and Glynn}]{zhou2021finite}
\textsc{Zhou, Z.}, \textsc{Zhou, Z.}, \textsc{Bai, Q.}, \textsc{Qiu, L.}, \textsc{Blanchet, J.} and \textsc{Glynn, P.} (2021{\natexlab{b}}).
\newblock Finite-sample regret bound for distributionally robust offline tabular reinforcement learning.
\newblock In \textit{International Conference on Artificial Intelligence and Statistics}. PMLR.

\end{thebibliography}
\bibliographystyle{ims}
\end{document}